\documentclass[review,onefignum,onetabnum]{siamonline190516}

\usepackage{multirow}
\usepackage{booktabs}
\usepackage{threeparttable}
\usepackage{subfigure}
\usepackage{lipsum}
\usepackage{amsfonts}
\usepackage{graphicx}
\usepackage{epstopdf}
\usepackage{algorithmic}
\ifpdf
  \DeclareGraphicsExtensions{.eps,.pdf,.png,.jpg}
\else
  \DeclareGraphicsExtensions{.eps}
\fi

\usepackage{enumitem}
\setlist[enumerate]{leftmargin=.5in}
\setlist[itemize]{leftmargin=.5in}

\newsiamremark{remark}{Remark}
\newsiamremark{hypothesis}{Hypothesis}
\crefname{hypothesis}{Hypothesis}{Hypotheses}
\newsiamthm{claim}{Claim}

\graphicspath{{fig/}}

\headers{AN ICTM-LVF SOVER FOR IMAGE SEGMENTATION}{CHAOYU LIU, ZHONGHUA QIAO, and QIAN ZHANG}

\title{An Active Contour Model with Local Variance Force Term and Its Efficient Minimization Solver for Multi-phase Image Segmentation\thanks{Submitted to the editors DATE.
\funding{This work is supported by the CAS AMSS-PolyU Joint Laboratory of Applied Mathematics grant 1-ZVA8. Z. Qiao's work is partially supported by the Hong Kong Research Grant Council RFS grant RFS2021-5S03, GRF grant 15302919 and the Hong Kong Polytechnic University internal grant 4-ZZLS. Q. Zhang's research is supported by the 2019 Hong Kong Scholar Program G-YZ2Y.}}}

\author{Chaoyu Liu \thanks{Department of Applied Mathematics, The Hong Kong Polytechnic University, Hung Hom, Hong Kong (\email{polyucy.liu@connect.polyu.hk})}
	\and Zhonghua Qiao \thanks{Department of Applied Mathematics \& Research Institute for Smart Energy, The Hong Kong Polytechnic University, Hung Hom, Hong Kong (\email{zhonghua.qiao@polyu.edu.hk})}
	\and  Qian Zhang\thanks{Corresponding author. Department of Applied Mathematics, The Hong Kong Polytechnic University, Hung Hom, Hong Kong (\email{qian77.zhang@polyu.edu.hk})}}

\usepackage{amsopn}

\begin{document}
	\ifpdf
	\hypersetup{
	  pdftitle={An Active Contour Model with Local Variance Force Term and Its Efficient Minimization Solver for Multi-phase Image Segmentation},
	  pdfauthor={CHAOYU LIU, ZHONGHUA QIAO, and QIAN ZHANG}
	}
	\fi

\maketitle

\begin{abstract}
In this paper, we propose an active contour model with a local variance force (LVF) term that can be applied to multi-phase image segmentation problems. With the LVF, the proposed model is very effective in the segmentation of images with noise. To solve this model efficiently, we represent the regularization term by characteristic functions and then design a minimization algorithm based on a modification of the iterative convolution-thresholding method (ICTM), namely ICTM-LVF. This minimization algorithm enjoys the energy-decaying property under some conditions and has highly efficient performance in the segmentation. To overcome the initialization issue of active contour models, we generalize the inhomogeneous graph Laplacian initialization method (IGLIM) to the multi-phase case and then apply it to give the initial contour of the ICTM-LVF solver. Numerical experiments are conducted on synthetic images and real images to demonstrate the capability of our initialization method, and the effectiveness of the local variance force for noise robustness in the multi-phase image segmentation.
\end{abstract}



\section{Introduction}
\label{sec1}
Image segmentation is a process of partitioning an image into disjoint segments according to the certain characterization of the contents in the image. It plays a fundamental role in the visual information analysis processes, including object detection, image recognition, and visual field monitoring \cite{Abburu,mitiche2010variational,Pham}.\par
Various approaches have been proposed for image segmentation problems over the last few decades, among which active contour models are widely used in many fields and enjoy great success \cite{CasellesVicent1997Geodesic,chan2001active,GUO2021108013,Kass1988Snakes,ma2020fast,niu2017robust}. The basic procedure for active contour models achieving image segmentation is first creating an initial contour or region and then driving it to evolve to the object edges or covering the object according to edge-based or region-based information from the image. In the edge-based active contour models, the evolution of the contour is mainly dependent on the gradient information, which makes it sensitive to the initialization, noise, and intensity strength of the boundary \cite{li2011level,min2021inhomogeneous}.
Region-based active contour models \cite{akram2017active,chan2001active,li2007implicit,li2008minimization,niu2017robust,song2020active,wang2009active,yang2019image,zosso2017image} typically define an energy functional over the image domain and all segments. Minimizing these well-designed energy functionals, one can obtain desired segmentation results. Compared to the edge-based active contour models, region-based ones have better performance on images with noise and weak boundaries.

Besides the models, the minimization methods of the energy functionals are also of great importance for image segmentation. In the region-based active contour models, the alternative direction method can effectively simplify the process of the minimization for multi-variable energy functionals \cite{jung2007multiphase,WU2021108017,2020Non}. Then the gradient descent method is commonly used to locate the contours' positions where the energy functional reach its minimum, see for instance, the celebrated level set method introduced by Osher and Sethian in \cite{osher1988fronts}, where the choice of the numerical method plays an important role for solving the derived partial differential equation. An improper numerical method will result in inefficiency, even instability.
Recently, Wang et al. \cite{2017An,2019The} proposed a new framework for the minimization of the energy functional in the region-based active contour models in which segments are represented by their characteristic functions, and a concave functional is used to approximate the contour length. Then an ICTM is proposed in this framework to minimize the modified energy functionals in an efficient and energy stable way. This framework is applicable to a wide range of active contour models and can be implemented in a simple and efficient way for both binary and multi-phase segmentation. But for images with noise, the segmentation results of ICTM will not be so satisfying when it is applied to some noise-sensitive active contour models.

In addition to the challenge by the noise, region-based active contour models are also facing an initialization issue. The initial contour significantly affects segmentation results since most of the energy functionals in the region-based active contour models are non-convex. A feasible solution for this initialization issue is to relax the energy functional into a convex functional \cite{bae2015efficient,bae2011global,bresson2007fast,brown2010convex,Cai_Zeng,chan2006algorithms,pock2009algorithm,wu2021color}. 
However, convex relaxation deprives all the non-convex parts of the original energy functional, and hence may entail the loss of non-convex information of the segments and reduce models' capability of preserving the sharpness and neatness of edges \cite{chan2018convex,wu2021two}. Very recently, Qiao and Zhang proposed an initialization method \cite{qiao2021twophase} where the non-local Laplacian operator is employed to select the initial contour. Compared to the typical choice, for instance, select artificially, or through some simple threshold, \cite{jung2007multiphase,li2011multiphase,2019The,yang2019image}, the detected initial contour can reduce the sensitivity of initialization. But for images with intensity inhomogeneity, it may not work. To handle this problem, Liu et al. \cite{qiao2021two} developed an IGLIM for binary segmentation, which can give reliable initial contours on grayscale images even with intensity inhomogeneity. For the multi-phase segmentation problem, the selection for the initial contour will become more tricky and sensitive, and an improper initialization will cause an unacceptable segmentation result.

In this paper, we propose a new region-based active contour model with a local variance force term. Similar to the model in \cite{2019The}, segments are represented by their characteristic functions, and a concave functional is used to approximate the contour length. The local variance force term can help to segment images with noise in a highly effective way. Then we develop an ICTM-LVF solver to minimize the proposed energy to achieve highly efficient segmentation. To get reliable initial contours for the minimization solver, we generalize the IGLIM to the multi-phase case and propose the Multi-IGLIM. The rest of this paper is organized as follows. In \cref{sec2}, we introduce the LVF and propose the new active contour model. In \cref{sec3}, we design the ICTM-LVF algorithm for solving the developed model with the initialization by the Multi-IGLIM. The energy dissipation of the ICTM-LVF is also proved in this section. Numerical experiments are given in \cref{sec4} to show the effectiveness of the Multi-IGLIM and the performance of the ICTM-LVF solver. Comparisons are also made between the ICTM-LVF solver and some state-of-art image segmentation methods. Finally, the paper ends with some conclusions in \cref{sec5}.

\section{Proposed model}
\label{sec2}
In this section, we will propose a new active contour model with a local variance force term. First, we will briefly review some region-based active contour models and present the general form of their energy functionals. Then a local variance force is introduced to improve the robustness to noise of active contour models.
\subsection{General energy functionals for region-based active contour models}
Region-based active contour models achieve image segmentation by minimizing energy functionals that are defined on segments. To make it more clear, we will provide two representative models. The first one is the Chan-Vese (CV) model \cite{chan2001active,vese2002multiphase} in which the energy functional for an $n$-phase image can be defined as follows:
\begin{equation}
	\label{CV}
	E_{CV}(\Omega_1,\cdots,\Omega_n,c_1,\cdots,c_n) = \sum_{i=1}^n \int_{\Omega_i} \lambda_i|I(x)-c_i|^2 dx + \mu\sum_{i=1}^n |\partial \Omega_i|,
\end{equation}
where $\Omega_i$ represents the $i$-th segment whose boundary is $\partial \Omega_i$, $|\partial \Omega_i|$ denotes the length of $\partial \Omega_i$, $\lambda_i$, $\mu$ are fixed positive
parameters, and $c_i$ is a constant used to approximate the average of the image intensity $I$ within $\Omega_i$. Once the image $I$ is given, our goal is to find suitable partitions with $\Omega_i$ and $c_i$ such that $E_{CV}$ reaches its minimum.
The CV model is a classic and effective region-based active contour model. However, the CV model assumes that the intensity is constant on all segments, which makes it inappropriate for images with intensity inhomogeneity. To tackle the segmentation difficulty caused by intensity inhomogeneity, another widely used model, the local binary fitting (LBF) model, was given in \cite{li2007implicit,li2008minimization}. This model was originally proposed for two-phase image segmentations. But generally, the local image fitting (LIF) model can achieve an $n$-phase segmentation, whose energy functional is defined as
\begin{equation}
	\label{ILF}
	E_{LIF}(\Omega_1,\cdots,\Omega_n,f_1,\cdots,f_n)= \sum_{i=1}^n \lambda_i\int_\Omega\int_{\Omega_i} G_\sigma(x-y)|I(y)-f_i(x)|^2 dydx + \mu\sum_{i=1}^n |\partial \Omega_i|,
\end{equation}
where
\begin{equation}
	\label{Gaussian_Kernel}
	G_\sigma(x) = \frac{1}{2\pi\sigma}e^{-\frac{|x|^2}{2\sigma}}
\end{equation}
is a two-dimensional Gaussian kernel with standard deviation $\sigma$, and $f_i(x)$ are local intensity fitting functions.

The CV model and the LIF model are two representative region-based active contour models whose energy functionals consist of a fidelity
term and a regularized term $\mu\sum_{i=1}^n |\partial \Omega_i|$. In fact, as shown in \cite{2019The}, numerous active contour models are similar in terms of the form of the energy functional (a fidelity term and a regularized term). From these models, one can abstract a rather general energy functional:
\begin{equation}
	\label{g_energy}
	\begin{split}
		E_{g}(u_1,\cdots,u_n,\Theta_1,\cdots,\Theta_n) &= \sum_{i=1}^n\lambda_i\int_{\Omega}u_iF_i(I,\Theta_1,\cdots,\Theta_n) dx \\ &\quad+\mu\sum_{i=1}^n\sum\limits_{j=1,j\ne i}^n\sqrt{\frac{\pi}{\tau}}\int_{\Omega}  u_iG_\tau*u_j dx,
	\end{split}
\end{equation}
where $\Theta_i = (\Theta_{i,1},\cdots,\Theta_{i,n})$ denotes all possible
variables or functions in the fidelity term. In the model of \eqref{g_energy}, all the segments are represented by their characteristic functions $u = (u_1,u_2,\cdots,u_n)$ and $|\partial \Omega_i|$ is approximated by
$$
\sum\limits_{j=1,j\ne i}^n\sqrt{\frac{\pi}{\tau}}\int_{\Omega}  u_iG_\tau*u_j dx,
$$
where $\tau$ is the standard derivation of the Gaussian kernel $G_\tau$ and $*$ denotes the convolution.

To minimize the energy functional \cref{g_energy} for image segmentation, Wang et al. \cite{2019The} proposed an efficient numerical solver, called the ICTM, which has energy-decaying property and can converge rapidly in numerical experiments. Despite efforts of the ICTM to increase the convergence speed and stability in the segmentation, initialization and noise robustness are still two significant issues for active contour models. In this paper, we will develop a new region-based active contour model with a local variance force term, which has better performance in the segmentation of images with noise and can be solved by a majority minimization method that is similar to the ICTM. Our minimization algorithm also enjoys the energy-decaying property under some conditions and can converge rapidly in experiments on various images.
\subsection{Local variance force}
In this part, we introduce a local variance force which can be combined with the original energy  \eqref{g_energy} to ameliorate segments' robustness to noise and lead to a smooth segmentation result for images with noise.

The local variance force for $x\in \Omega$ in the $i$-th phase is defined as follows:
\begin{equation}
	\label{local variance force}
	F_i^{LV}(I,m_i) = \sum\limits_{y\in K_r(x)}|I(y)-m_i|^2,
\end{equation}
where $I$ is a given image, $r$ is a given positive integer, $K_r(x)$ is a square of $(2r+1)\times (2r+1)$ pixels centered at $x$, and $m_i$ is an approximated mean intensity of the non-degenerated images in the $i$-th phase.

	For degenerated images, such as images with noise or intensity inhomogeneity, it is not easy to find the exact mean value of pixel intensities of the non-degenerated images in a small local region. However, the intensity value can be approximated by some constants or functions, which is a commonly approach in active contour models. For instance, the CV model uses the mean values of the segments to approximate the non-degenerated intensity value of each pixel in different phases:
	$$
		m_i\approx c_i = \frac{\int_\Omega u_i I dx}{\int_\Omega u_i dx}.
	$$
	In the LBF/LIF model, the following functions are used to approximate the local mean image intensities for each pixel in different phases:
	$$
	m_i\approx f_i = \frac{\int_{\Omega}u_i(y)G_\sigma(y-x)I(y)dy}{\int_{\Omega}u_i(y)G_\sigma(y-x)dy}.
	$$
	Typically, $m_i$ is related to $I$, and $u_i$, therefore we can approximate $m_i$ by some prior estimating functionals which we denote by $M(I,u_i)$.

	For the CV model, the prior mean intensity estimating functional can be chosen as:
	$$
		M(I,u_i)=\frac{\int_\Omega u_i I dx}{\int_\Omega u_i dx}.
	$$
	For the LBF/LIF model, we have two choices:
	$$
M(I,u_i) =\frac{\int_\Omega u_i I dx}{\int_\Omega u_i dx} \quad \text{ or }\quad
M(I,u_i) =\frac{\int_{\Omega}u_i(y)G_\sigma(y-x)I(y)dy}{\int_{\Omega}u_i(y)G_\sigma(y-x)dy}.
	$$

\subsection{New active contour model with a local variance force term}
Since the variance of local regions should be relatively small in a proper segmentation, we can use \cref{local variance force} to evaluate the fitness of $x\in \Omega$ in the $i$-th phase. In fact, for most of noise-sensitive models, the local variance is not considered. As a result, pixels with noise are prone to be divided into an improper segment. In view of this, we add the local variance force to the original energy \cref{g_energy} to prevent the misclassification caused by the noise and ensure the smoothness of the segmentation result. The modified energy is given below.
\begin{equation}
	\label{energy-lvf}
	\begin{aligned}
			E(u,\Theta,m) &= \sum_{i=1}^n\lambda_i\int_{\Omega}u_iF_i(I,\Theta_1,\cdots,\Theta_n) dx  +\mu\sum_{i=1}^n\sum\limits_{j=1,j\ne i}^n\sqrt{\frac{\pi}{\tau}}\int_{\Omega}  u_iG_\tau*u_j dx,\\
			& + p\sum_{i=1}^n\lambda_i\int_{\Omega}\int_{\Omega} u_i(x) K_r(y-x)(I(y)-m_i(x))^2dydx,
	\end{aligned}
\end{equation}
where $p$ is a positive constant.
\section{Numerical method}
\label{sec3}
In this part, we develop an efficient numerical solver, the ICTM-LVF minimization algorithm, to solve the proposed energy model \eqref{energy-lvf} for the multi-phase segmentation. To get a proper initial contour for the minimization process, we employ the Multi-IGLIM, which is generalized from the IGLIM proposed in \cite{qiao2021two}.
\subsection{Multi-IGLIM}
The IGLIM is merely designed for binary segmentation. We can generalize it to the multi-phase case through K-means to give initial contours of active contour models for multi-phase segmentation.

The IGLIM generates initial contours according to the information derived from an anisotropic Laplacian operator, namely the inhomogeneous graph Laplacian operator. Moreover, a denoising step is performed in the IGLIM to remove the possible noise effect and obtain more reliable initial contours.
\subsubsection{Inhomogeneous graph Laplacian operator}
Let $\Omega$ be a 2D discrete image domain, and $I$ be an image defined on it with $M_1\times M_2$ pixels. Denote a pixel $x_0 = (i,j)\in \Omega$, its corresponding intensity vector $I(x_0)$ by $I_{i,j}$. The inhomogeneous graph Laplacian operator $L$ is defined as
\begin{equation}\label{Graph_L}
L(x_0)= \sum_{k=1}^{8}\sum_{p=1}^{r}c_kI_{i,j}^k(p) -I_{i,j}(p),
\end{equation}
where $r$ is the number of channels, and $I_{i,j}^k$ is the intensity vector of the $k$-th neighbour point of $I_{i,j}$, more precisely,
$$
\begin{aligned}
	& I_{i,j}^1 = I_{i-1,j-1}, I_{i,j}^2 = I_{i-1,j}, I_{i,j}^3 = I_{i-1,j+1}, I_{i,j}^4 = I_{i,j+1}, \\
	& I_{i,j}^5 = I_{i+1,j+1}, I_{i,j}^6 = I_{i+1,j},
	I_{i,j}^7 = I_{i+1,j-1}, I_{i,j}^8 = I_{i,j-1},
\end{aligned}
$$
and
$$
c_k=\frac{\sum_{p=1}^{r}e^{\lambda(I_{i,j}(p)-I_{i,j}^k(p))^2}}{\sum_{k=1}^{8}\sum_{p=1}^{r}e^{\lambda(I_{i,j}(p)-I_{i,j}^k(p))^2}},\text{ $\lambda$ is a given non-negative parameter}.
$$
The parameter $r$ equals 1 when the target image is grayscale, while $r$ of a color image is usually set to 3 corresponding to the R, G, B channel, respectively. 
We then approximate the zero-cross points of the inhomogeneous Laplacian
operator to obtain rough initial edges. The edge
points are divided into two groups $S_p$ and $S_n$ according to the sign of the inhomogeneous Laplacian value.
The readers may refer to \cite{qiao2021two} for a more detailed description of the inhomogeneous graph Laplacian operator.

\subsubsection{A denoising method based on the connectivity of edge points}
The research in \cite{qiao2021two} gives a denoising method based on the diagonal connectivity of edge points to remove some possible noise pixels from the rough initial contour obtained from the inhomogeneous graph Laplacian operator.
The diagonally connected points are defined as follows:
\begin{definition} \cite{qiao2021two}
When $x = (i,j)\in S_p$ ($S_n$), the neighbor areas are divided into four parts:
\begin{align*}
	&S_1 = \{(i-1,j-1),(i,j-1),(i-1,j)\}, \ S_2 = \{(i-1,j+1),(i,j+1),(i-1,j)\},\\
	&S_3 = \{(i+1,j-1),(i,j-1),(i+1,j)\}, \ S_4 = \{(i+1,j+1),(i,j+1),(i+1,j)\}.
\end{align*}
We call $x\in S_p$ ($S_n$) a diagonally connected point if both $S_1$ and $S_4$ or both $S_2$ and $S_3$ have at least one pixel that also belongs to $S_p$ ($S_n$).
\end{definition}

To eliminate the noise in the rough initial contour, the IGLIM keeps all the diagonally connected points in $S_p$ ($S_n$) and removes the other points from $S_p$ ($S_n$). The denoising process needs to be repeated $M$ times where $M$ is a pre-setting small integer.

For a better understanding of the Multi-IGLIM, we display the algorithm for the IGLIM in \cref{IGLIM}.
\begin{algorithm}
	\caption{IGLIM \cite{qiao2021two} }
	\label{IGLIM}
\begin{algorithmic}[1]
	\REQUIRE{$I, \lambda, \alpha, M$}
	\ENSURE{Initial contours for binary phases $u^0$}
	\STATE{Compute the inhomogeneous Laplacian value of each pixel by \eqref{Graph_L}.}
	\STATE{Set a small positive number $\alpha$. Determine $S_p$ and $S_n$ according to the sign of the inhomogeneous Laplacian value at each pixel.}
	\STATE{Take $S_p$ $(S_n)$ as a rough initial contour.}
	\STATE{Go through every pixel in $S_p$ $(S_n)$ and judge whether it is diagonally connected. If not, remove it from $S_p$ $(S_n)$.}
	\STATE{Set an appropriate integer $M$, and repeat Step 4 for $M$ times.}
	\STATE{Output $S_p$ $(S_n)$ as the initial contour $u^0$.}
\end{algorithmic}
\end{algorithm}

\subsubsection{Algorithm for the Multi-IGLIM}
Since there are only two cases for the sign of inhomogeneous graph Laplacian operator, i.e., positive and negative, the IGLIM can only be applied to binary image segmentation. For the multi-phase image segmentation, we generalize the IGLIM to the Multi-IGLIM for the initialization. First, we determine all the possible edge points through their absolute values of inhomogeneous graph Laplacian rather than from their signs. If the absolute value is larger than a certain threshold, its corresponding point will be regarded as a possible edge point. Then when all the possible edge points are located, we employ K-means to divide all the edge points into different phases according to their pixel values. The algorithm of the Multi-IGLIM is summarized in \cref{multi-IGLIM}.

\begin{algorithm}
	\caption{Multi-IGLIM}
		\label{multi-IGLIM}
\begin{algorithmic}[1]
	\REQUIRE{$I, \lambda, \alpha, M, n$}
	\ENSURE{Initial contours for $n$ phases $u^0 = (u_1^0,u_2^0,\cdots,u_n^0)$}
	\STATE{Compute the inhomogeneous graph Laplacian value of each pixel by \eqref{Graph_L}}.
	\STATE{Set a small positive number $\alpha$. Put all together pixels $x$ whose $|L(x)|\ge \alpha$ into a set denoted as $S$}.
	\STATE{Divide into $n$ separate sets all pixels of $S$ by K-means according to their pixel values. Denote these sets as $S_1,\cdots,S_n$}.
	\STATE{Go through every pixel in $S_1,\cdots,S_n$ and judge whether it is diagonally connected in their set. If not, remove it from their set.}
	\STATE{Set an appropriate integer $M$, and repeat Step 4 for $M$ times.}
	\STATE{Output the characteristic function of $S_i, i = 1,\cdots,n$, which is defined as $u_i^0$, as the initial contours.}
\end{algorithmic}
\end{algorithm}

\subsection{ICTM-LVF solver}
To solve the proposed model \cref{energy-lvf} efficiently, we develop a minimization algorithm, viz. the ICTM-LVF solver.

Similar to the ICTM \cite{2019The}, we split the original problem into two optimization problems instead of solving it directly. Given an initial segments $u^0=(u_1^0, u_2^0,\cdots, u_n^0)$ obtained from the Multi-IGLIM in \cref{multi-IGLIM}, we can update $\Theta,m$ and $u$ as follows: 
\begin{equation}\label{Mini_Alg}
\left\{
  \begin{array}{ll}
    \displaystyle\Theta^{k} = \arg\min_{\Theta}E(u^k,\Theta),\\
    \displaystyle m^{k} = M(I,u^k), \\
	\displaystyle u^{k+1} = \arg\min_{u\in\mathcal{S}}E(u,\Theta^k,m^k),
  \end{array}
\right.
\end{equation}
for $k = 0,1,2,\cdots$, where $\mathcal{S}:=\{(u_1,u_2,\cdots,u_n)|u_i\in[0,1],\sum_{i=1}^n u_i = 1\}.$

Note that
$$
	E(u,\Theta^k,m^k) = E(u^k,\Theta^k,m^k) + L(u,u^k,\Theta^k,m^k) + h.o.t.
$$
Here
\begin{equation}
	\label{linearization}
	L(u,u^k,\Theta^k,m^k) = \sum_{i=1}^n\int_{\Omega}u_i\phi_i^k dx =\int_{\Omega}\sum_{i=1}^n \phi_i^k u_i dx,
\end{equation}
where
$$
\phi_i^k = \lambda_iF_i(I,\Theta^k) +2\mu\sqrt{\frac{\pi}{\tau}}\sum\limits_{j=1,j\ne i}^n G_\tau*u_j^k +p\lambda_iK_r*(I-m_i^k)^2.
$$
Instead of solving $u^{k+1}= \arg\min_{u}E(u,\Theta^k,m^k)$ in \eqref{Mini_Alg}, we approximate $u^{k+1}$ by minimizing the linearized functional \cref{linearization} of $E(u,\Theta^k,m^k)$ as
\begin{equation}
	\label{min_L}
	u^{k+1}= \arg\min_{u}L(u,u^k,\Theta^k,m^k).
\end{equation}
From \cref{linearization}, one can compare the coefficients $\phi_i^k(x)$ of $u_i$ for a fixed point $x$ and find that the minimum for $L$ with respect to $u$ is attained at
$$
	u_i(x) = \begin{cases}
	1 \quad \text{if } i = \arg\min_{l\in [n]}\phi_l^k,\\
	0 \quad \text{otherwise}.
\end{cases}
$$
Therefore, solving \cref{min_L} for $u^{k+1}$ can be carried out at each $x\in \Omega$ independently.
The algorithm of ICTM-LVF is displayed in \cref{ICTM with LVF}. Compared with the original ICTM, a local variance force based on a prior mean estimating function is introduced when updating $u$. In fact, in the original ICTM, the update procedure of $u$ is also independently carried out at each $x\in \Omega$, which makes it outperform the classical level set method in terms of efficiency. However, the ICTM suffers from the noise heavily since the value of $u_i^{k+1}$ merely depends on the intensity of a single pixel. While in the ICTM-LVF algorithm, an index of the fitness of segments, i.e., the local variance force, is considered at each $x\in \Omega$ when $u$ is updated. As we can see from the numerical experiments in \cref{sec4}, the local variance force term can effectively reduce the noise effect. In addition, the energy-decaying property for the ICTM-LVF solver can also be guaranteed under some conditions.
\begin{algorithm}[H]
  \caption{ICTM-LVF}
	\label{ICTM with LVF}
\begin{algorithmic}[1]
  \REQUIRE{$I,u^0$}
  \ENSURE{Segmentation Results $u^s$}
  \STATE{Set $k=0$}
  \WHILE{not converged}
  \STATE{For the fixed $u^k$, find
      $$
				\Theta^k = \arg\min\limits_{\Theta}\sum_{i=1}^n\lambda_i\int_{\Omega_i}u_iF_i(I,\Theta) dx.
			$$}
	\STATE{For the fixed $u^k$, update
			$$
				m^k = M(I,u^k).
			$$}
	\STATE{For $i \in [n]$, evaluate
    	$$\phi_i^k = F_i(I,\Theta^k)+ pF_i^{LV}(I,m^k) + 2\mu\sum\limits_{j=1,j\ne i}^n\sqrt{\frac{\pi}{\tau}}G_\tau*u_j^k.$$}
	\STATE{For $i \in [n]$, set
		$$
			u_i^{k+1} = \begin{cases}
			1 \quad \text{if } i = \min\{\arg\min_{l\in [n]}\phi_l^k\},\\
			0 \quad \text{otherwise}.
		\end{cases}
		$$}
	\ENDWHILE
	\STATE{	Set $k = k+1$}
\end{algorithmic}
\end{algorithm}

\subsection{Energy-decaying property for the ICTM-LVF}
In this part, we will prove the energy-decaying property for the ICTM-LVF under some conditions. Firstly, we will introduce two lemmas related to the discrete convolution.
\begin{lemma}
	\label{lemma1}
	Given $u(x):\mathbb{R}\to\mathbb{R}$, let $g(x) = \frac{1}{\sqrt{2\pi\sigma}}e^{-\frac{x^2}{2\sigma}},g_k = g(k),u_k = u(k)$, and
	$$
	U =
	\left[
	u_0,u_{1},\cdots, u_{n-1}
	\right]^T.
	$$
	Define the discrete form of the convolution $g*u$ on the meshgrid $[0,1,2,\cdots,n-1]$ by an $n\times 1$ vector $D_{g*u}$ as follows:
    $$
	D_{g*u}(k) = \sum_{i = -\lceil\frac{n}{2}-1\rceil}^{\lfloor \frac{n}{2}\rfloor}g(-i)u(k-1+i),\quad k = 1,\cdots,n.
	$$
	Then $D_{g*u}$	can be expressed as the product of a matrix $G$ and $U$ if $u$ is under the periodic boundary condition. Moreover, if $\sigma <\frac{1}{2}$, $G$ is positive definite and all its eigenvalues are larger than a positive constant independent on $n$.
\end{lemma}
\begin{proof}
	Let
	$$
	G =
\left[\begin{array}{cccccccc}
g_0   & g_1   & \ldots & g_{\lfloor \frac{n}{2}\rfloor} & g_{\lceil\frac{n}{2}-1\rceil}   & \ldots&g_2& g_1\\
g_1   & g_0   &g_1 & \ldots   & g_{\lfloor \frac{n}{2}\rfloor}     & \ddots&g_3& g_2\\
\vdots&\ddots & \ddots&\ddots       &\ldots      &\ddots       &\ddots &\vdots\\
g_{\lfloor \frac{n}{2}\rfloor}   &\ddots   & \ddots&  \ddots   & \ddots & \ldots& \ddots&\ddots\\
g_{\lceil\frac{n}{2}-1\rceil} &\ddots   & \ddots&  \ddots   & \ddots & \ddots& \ldots&\ddots\\
\vdots&\ddots   & \ddots&  \ddots   & \ddots & \ddots&\ddots &\ldots\\
g_2&\ldots   & g_{\lceil\frac{n}{2}-1\rceil}&g_{\lfloor \frac{n}{2}\rfloor}   &\ldots &g_1& g_0&g_1 \\
g_1 &g_2    & \ldots & g_{\lceil\frac{n}{2}-1\rceil}     & g_{\lfloor \frac{n}{2}\rfloor} & \ldots&g_1&g_0
\end{array}\right].
$$
Then we have
$$
\begin{aligned}
	\left[GU\right]_k &= \sum_{i = -\lceil\frac{n}{2}-1\rceil}^{\lfloor \frac{n}{2}\rfloor}g(i)u(k-1+i)\\
	&=\sum_{i = -\lceil\frac{n}{2}-1\rceil}^{\lfloor \frac{n}{2}\rfloor}g(-i)u(k-1+i),\\
	& = D_{g*u}(k),
\end{aligned}
$$
where $k = 1,\cdots,n,$ and $\left[GU\right]_{k}$ represents the $k$th entry of $GU$.
Therefore, $$D_{g*u} = GU.$$
Note that $G$ is a circulant matrix, we can obtain its eigenvalues by calculating the product of discrete Fourier transform matrix $F$ and the first column of $G$, and then we can find they are all not less than a certain positive constant.
Indeed, denoting by $v$ the first column of $G$, we have
$$
\begin{aligned}
	\left[Fv\right]_j & =\frac{1}{\sqrt{2\pi\sigma}}\sum_{k=1}^{n}e^{\frac{2(k-1)j\pi i}{n}}G_{k,1}\\
	 &\ge \frac{1}{\sqrt{\pi}}(1-\sum_{k=1}^{\lfloor \frac{n}{2}\rfloor} 2e^{-k^2})\\
	 &\ge \frac{1}{\sqrt{\pi}}(1-\sum_{k=1}^{\infty} 2e^{2-3k})\\
	 & =\frac{1}{\sqrt{\pi}}(1-\frac{2e^2}{e^3-1})\\
	 &>0,
\end{aligned}
$$
where $\left[Fv\right]_j$ represents the $j$th entry of $Fv$.
\end{proof}

\cref{lemma1} can be generalized to the 2D case.
\begin{lemma}
	\label{lemma2}
Given $u(x,y):\mathbb{R}^2\to\mathbb{R}$, let $g(x) = \frac{1}{2\pi\sigma}e^{-\frac{x^2+y^2}{2\sigma}},g_k = \frac{1}{\sqrt{2\pi\sigma}}e^{-\frac{k^2}{2\sigma}}, u_{k_1,k_2} = u(k_1,k_2)$, and $U =(u_{i,j})_{m\times n}.$
Define the discrete form of the convolution $g*u$ on the meshgrid $[0,1,2,\cdots,n-1] \times [0,1,2,\cdots,n-1]$ by an $m\times n$ matrix $D_{g*u}$, where for $p = 1,\cdots,m, \quad q = 1,\cdots,n$,
$$
\begin{aligned}
	D_{g*u}(p,q) = \sum_{i = -\lceil\frac{m}{2}-1\rceil}^{\lfloor \frac{m}{2}\rfloor}\sum_{j = -\lceil\frac{n}{2}-1\rceil}^{\lfloor \frac{n}{2}\rfloor}g(-i)g(-j)u(p-1+i,q-1+j).
\end{aligned}
$$
Then $V(D_{g*u})$ can be expressed as the product of a matrix $G$ and $V(U)$ if $u$ is under the periodic boundary condition. Here $V$ denotes the vectorization operator and
$$
V(A) = [a_1^T,\cdots,a_n^T]^T
$$
for any $m\times n$ matrix $A= \left[a_1,\cdots,a_n\right]$.
Moreover, if $\sigma <\frac{1}{2}$, $G$ is positive definite and all its eigenvalues are larger than a positive constant independent on $m$ or $n$.
\end{lemma}

\begin{proof}
	Let
$$
	G_n =
\left[\begin{array}{cccccccc}
g_0   & g_1   & \ldots & g_{\lfloor \frac{n}{2}\rfloor} & g_{\lceil\frac{n}{2}-1\rceil}   & \ldots&g_2& g_1\\
g_1   & g_0   &g_1 & \ldots   & g_{\lfloor \frac{n}{2}\rfloor}     & \ddots&g_3& g_2\\
\vdots&\ddots & \ddots&\ddots       &\ldots      &\ddots       &\ddots &\vdots\\
g_{\lfloor \frac{n}{2}\rfloor}   &\ddots   & \ddots&  \ddots   & \ddots & \ldots& \ddots&\ddots\\
g_{\lceil\frac{n}{2}-1\rceil} &\ddots   & \ddots&  \ddots   & \ddots & \ddots& \ldots&\ddots\\
\vdots&\ddots   & \ddots&  \ddots   & \ddots & \ddots&\ddots &\ldots\\
g_2&\ldots   & g_{\lceil\frac{n}{2}-1\rceil}&g_{\lfloor \frac{n}{2}\rfloor}   &\ldots &g_1& g_0&g_1 \\
g_1 &g_2    & \ldots & g_{\lceil\frac{n}{2}-1\rceil}     & g_{\lfloor \frac{n}{2}\rfloor} & \ldots&g_1&g_0
\end{array}\right].
$$
Then we have
$$
D_{g*u} = G_mUG_n,
$$
and
$$
V(D_{g*u}) = V(G_mUG_n) = (G_n\otimes G_m)V(U),
$$
where $\otimes$ is the Kronecker product.
Note that
$$
\begin{aligned}
	G_n\otimes G_m & = (F_n^{-1}\Lambda_n F_n)\otimes(F_m^{-1}\Lambda_m F_m)\\
	& = (F_n^{-1}\otimes F_m^{-1})(\Lambda_n \otimes \Lambda_m )(F_n\otimes F_m).
\end{aligned}
$$
According to \cref{lemma1}, when $\sigma <\frac{1}{2}$, all diagonal entries of $\Lambda_n$ and $\Lambda_m$ are larger than a positive constant that is not dependent on $m$ or $n$, which implies $G_n\otimes G_m $ is also positive definite and all its eigenvalues are larger than a positive constant independent on $m$ or $n$. Hence, we can
obtain the desired $G$ by letting $G = G_n\otimes G_m$.
\end{proof}

The following theorem gives the energy-decaying property of the ICTM-LVF solver displayed in \cref{ICTM with LVF}.
\begin{theorem}
	Given $(u^k,\Theta^k)$ at the $k$-th iteration in \cref{ICTM with LVF}. Suppose that the mean intensity estimating functional $M(I,u)$ is Lipschitz continuous w.r.t. $u$. If $\mu$ is sufficiently large and $\tau<\frac{1}{2}$, then

	\begin{equation}
		\label{energy_dacay}
		E(u^{k+1},\Theta^{k+1},m^{k+1})\le E(u^k,\Theta^k,m^k),
	\end{equation}
	 where $E$ is defined in \cref{energy-lvf}.
\end{theorem}

\begin{proof}
	Let
	$$
	\begin{aligned}
			&E_1(u,\Theta) = \sum_{i=1}^n\lambda_i\int_{\Omega}u_iF_i(I,\Theta) dx  +\mu\sum_{i=1}^n\sum\limits_{j=1,j\ne i}^n\sqrt{\frac{\pi}{\tau}}\int_{\Omega}  u_iG_\tau*u_j dx,\\
			&E_2(u,m) = p\sum_{i=1}^n\lambda_i\int_{\Omega} u_i K_r*(I-m_i)^2dx,\\
			&L(u,m,u^k,\Theta^k) = \sum_{i=1}^n\int_{\Omega}u_i\left[\lambda_iF_i(I,\Theta^k) +2\mu\sqrt{\frac{\pi}{\tau}}\sum\limits_{j=1,j\ne i}^n G_\tau*u_j^k +p\lambda_iK_r*(I-m_i)^2 \right] dx.
	\end{aligned}
	$$
	Then we have
	\begin{align}
		E(u^k,\Theta^k,m^k) &= L(u^k,m^k,u^k,\Theta^k) -\mu\sqrt{\frac{\pi}{\tau}}\sum_{i=1}^n\sum\limits_{j=1,j\ne i}^n u_i^kG_\tau*u_j^k,\label{L1}\\
		E(u^{k+1},\Theta^k,m^{k}) &= L(u^{k+1},m^k,u^k,\Theta^k) + \mu\sqrt{\frac{\pi}{\tau}}\sum_{i=1}^n\sum\limits_{j=1,j\ne i}^n u_i^{k+1}G_\tau*u_j^{k+1}\\
		& - 2\mu\sqrt{\frac{\pi}{\tau}}\sum_{i=1}^n\sum\limits_{j=1,j\ne i}^n u_i^{k+1}G_\tau*u_j^k,\notag\\
		E(u^{k+1},\Theta^{k+1},m^k) &\le E(u^{k+1},\Theta^{k},m^k),\label{L2} \\
		E(u^{k+1},\Theta^{k+1},m^{k+1}) &= E(u^{k+1},\Theta^{k+1},m^k) + E_2(u^{k+1},m^{k+1}) - E_2(u^{k+1},m^k)\label{L3}.
	\end{align}
	Note that
\begin{equation}
	L(u^{k+1},m^k,u^k,\Theta^k)\le L(u^{k},m^k,u^k,\Theta^k).
\end{equation}
	Combining \cref{L1}-\cref{L3}, we can get
\begin{equation}
	\label{E}
	\begin{aligned}
			&E(u^{k+1},\Theta^{k+1},m^{k+1})-E(u^{k},\Theta^{k},m^k)\\ \le & E_2(u^{k+1},m^{k+1}) - E_2(u^{k+1},m^{k}) -\mu\sqrt{\frac{\pi}{\tau}}\sum_{i=1}^n\int_{\Omega}(u_i^{k+1}-u_i^{k})G_\tau*(u_i^{k+1}-u_i^{k})dx.
		\end{aligned}
\end{equation}
For convenience, we denote
$$
P_{\tau}(u) = \sqrt{\frac{\pi}{\tau}}\int_{\Omega}uG_\tau*udx.
$$
Then for the right-hand side of \cref{E}, we have
\begin{equation}
	\label{E1E2}
	\begin{aligned}
			& E_2(u^{k+1},m^{k+1}) - E_2(u^{k+1},m^k)-\mu\sum_{i=1}^nP_{\tau}(u_i^{k+1}-u_i^{k})\\
		\le & \sum_{i=1}^n\lambda_i\int_{\Omega} u_i^{k+1}q_1(x)|M^2(I,u_i^{k+1})-M^2(I,u_i^{k})|dx\\
		&+\sum_{i=1}^n\lambda_i\int_{\Omega} u_i^{k+1}q_2(x)|M(I,u_i^{k+1})-M(I,u_i^{k})|dx -\mu \sum_{i=1}^nP_{\tau}(u_i^{k+1}-u_i^{k})\\
		\le & \lambda_i(L_1+L_2)\sum_{i=1}^n\int_{\Omega} |u_i^{k+1}-u_i^{k}|dx-\mu\sum_{i=1}^nP_{\tau}(u_i^{k+1}-u_i^{k})\\
		 = &\lambda_i(L_1+L_2)\sum_{i=1}^n\int_{\Omega} (u_i^{k+1}-u_i^{k})^2dx-\mu\sum_{i=1}^nP_{\tau}(u_i^{k+1}-u_i^{k}),
	\end{aligned}
\end{equation}
where $q_1(x) = p K*1$, $q_2(x) = p K*I$ and $L_1$, $L_2$ are constants that are not dependent on $u$.
Combining \cref{E,E1E2} and using \cref{lemma2}, we complete the proof.
\end{proof}

\section{Numerical Experiments}
\label{sec4}
This section displays experiments to demonstrate the necessity of the Multi-IGLIM, the high robustness of our model to noise, and the efficiency of the ICTM-LVF solver for segmenting various images. All the experiments are implemented on a laptop with a Windows system, 2.60-GHz CPU, 16GB RAM, and MATLAB R2020a.
\subsection{Evaluation for the Multi-IGLIM}

In this part, we test the performance of the Multi-IGLIM. The initial contours obtained by the Multi-IGLIM for an image of flowers are shown in \cref{fig3.1}, and the segmentation result based on this initialization is displayed in \cref{fig3.2}. Gaussian noise is added to another synthetic RGB image for a further test. \cref{fig4.2} and \cref{fig4.1} display the corresponding initial contours and the segmentation results, respectively. It can be observed that the initial contours obtained by the Multi-IGLIM are reasonable and we can get desired segmentation results effectively based on these initializations.

In addition, we compare the results from the ICTM solver with the CV model (ICTM-CV) and different initializations in \cref{fig_comparison_flower}. Among them, the first three columns are the results of three different initial contours given artificially, and the last column gives the result obtained from the Multi-IGLIM. The iteration numbers of different initializations are recorded in \cref{table2}. One can see that a slight change in artificially selected initial contours can cause a longer CPU time and a relatively low-quality result for the segmentation, while the initial contour obtained by the Multi-IGLIM can lead to a high-quality segmentation result and, to some degree, reduce the iteration number and CPU time to reach the steady state.

\begin{figure}[htbp]
  \centering
\subfigure[$u_1$ (stalks)]{\includegraphics[width=3cm]{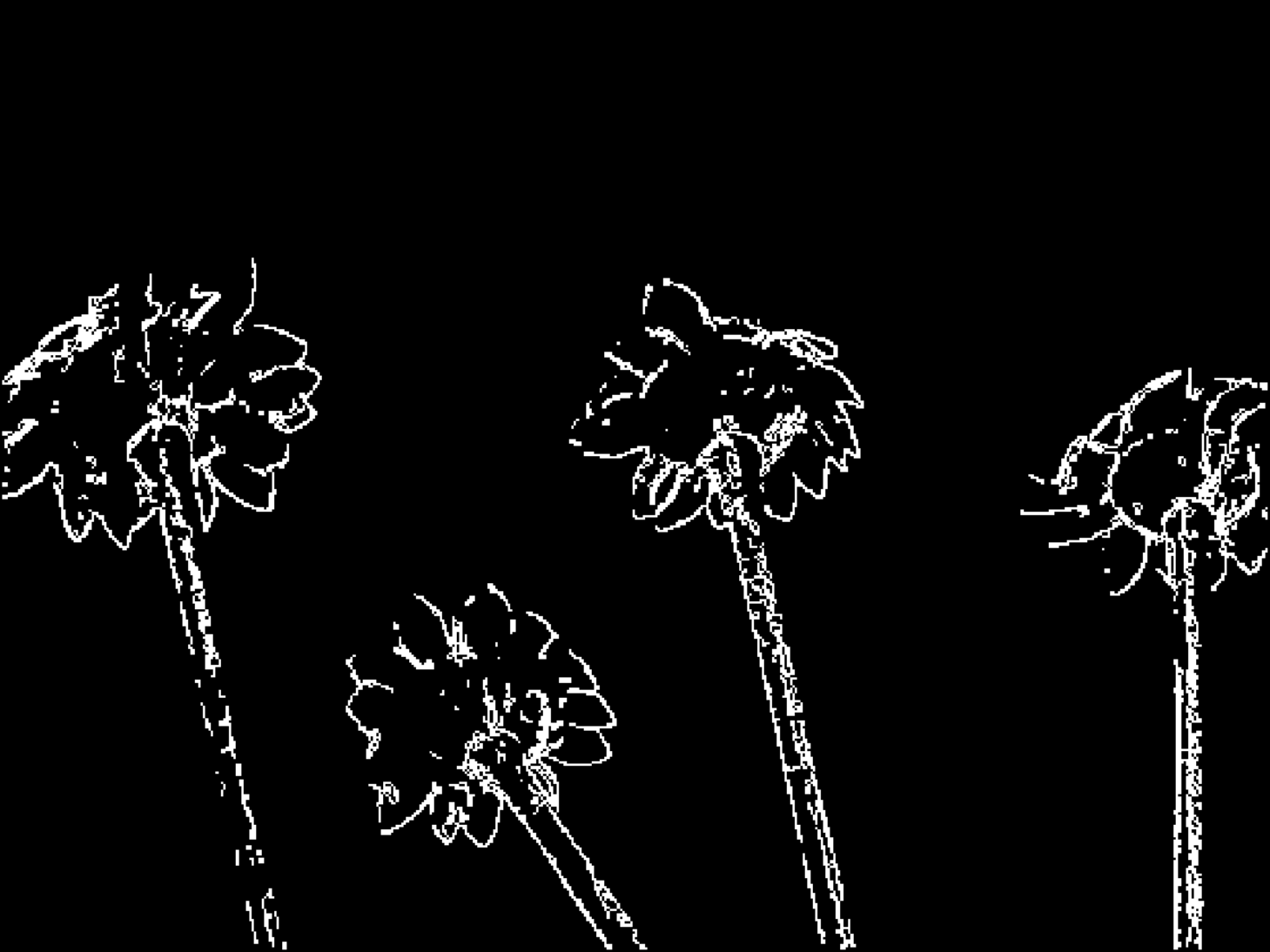}}
\subfigure[$u_2$ (red flowers)]{\includegraphics[width=3cm]{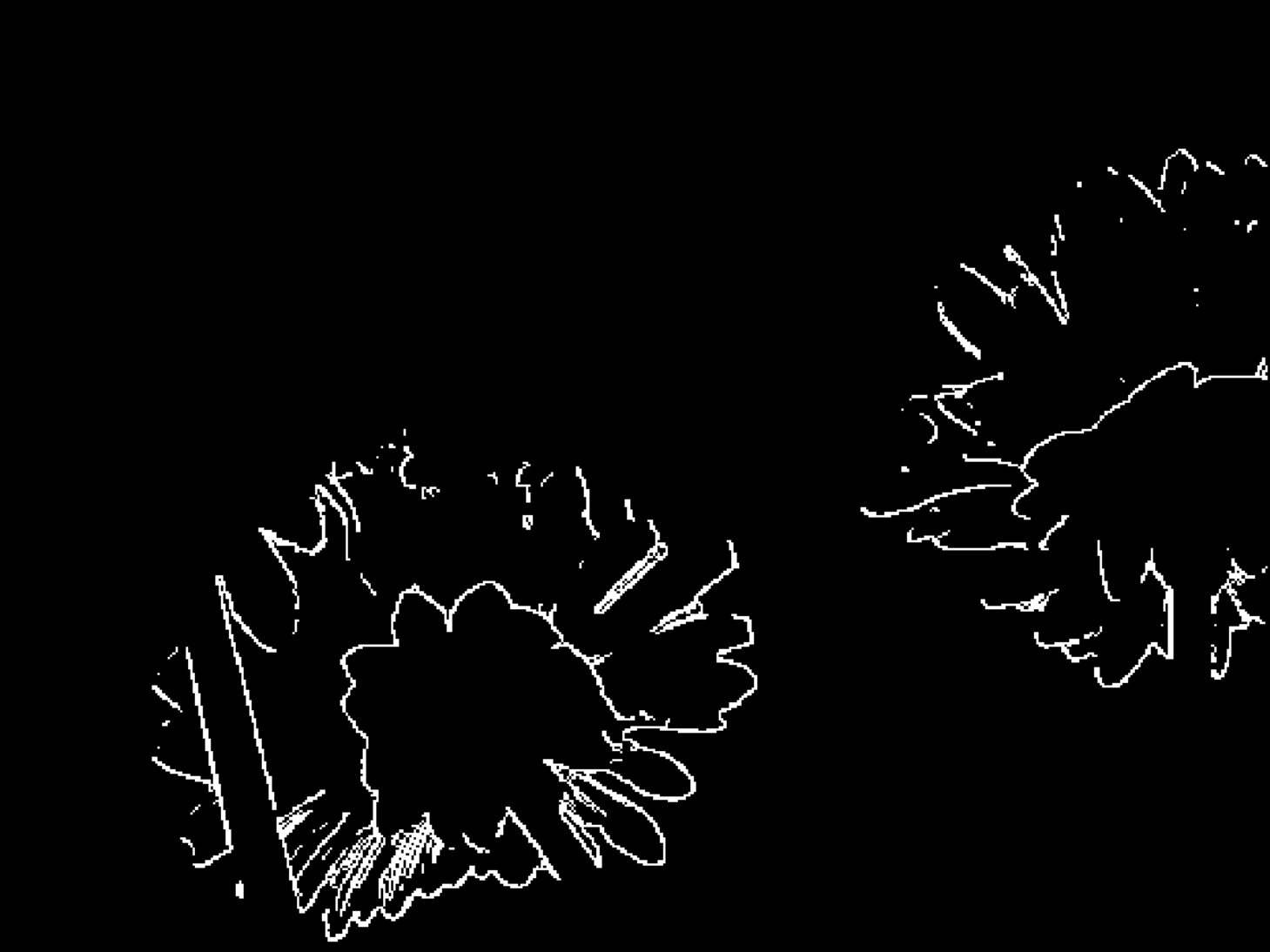}}
\subfigure[$u_3$ (yellow flowers)]{\includegraphics[width=3cm]{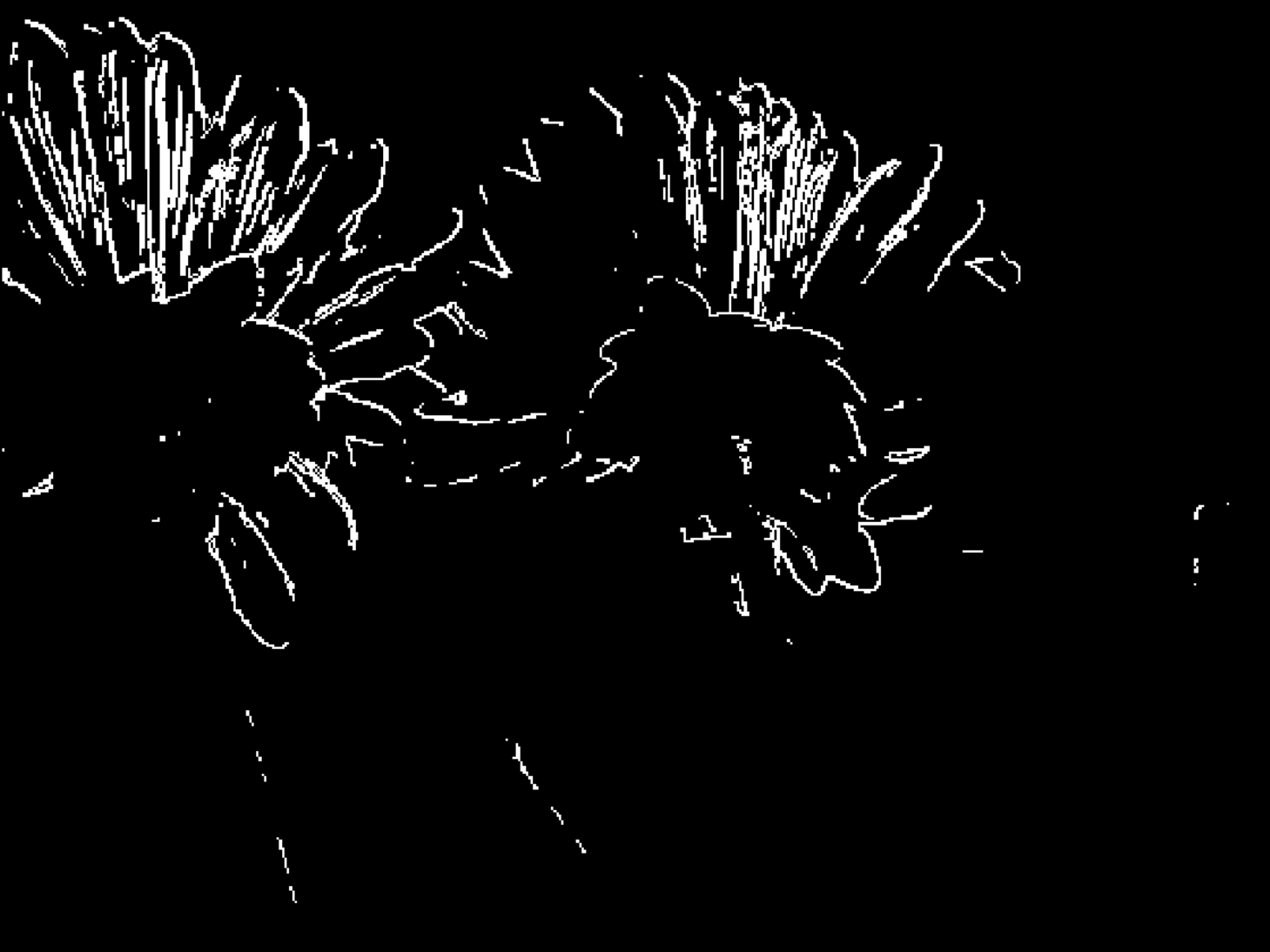}}
\subfigure[$u_4$ (background)]{\includegraphics[width=3cm]{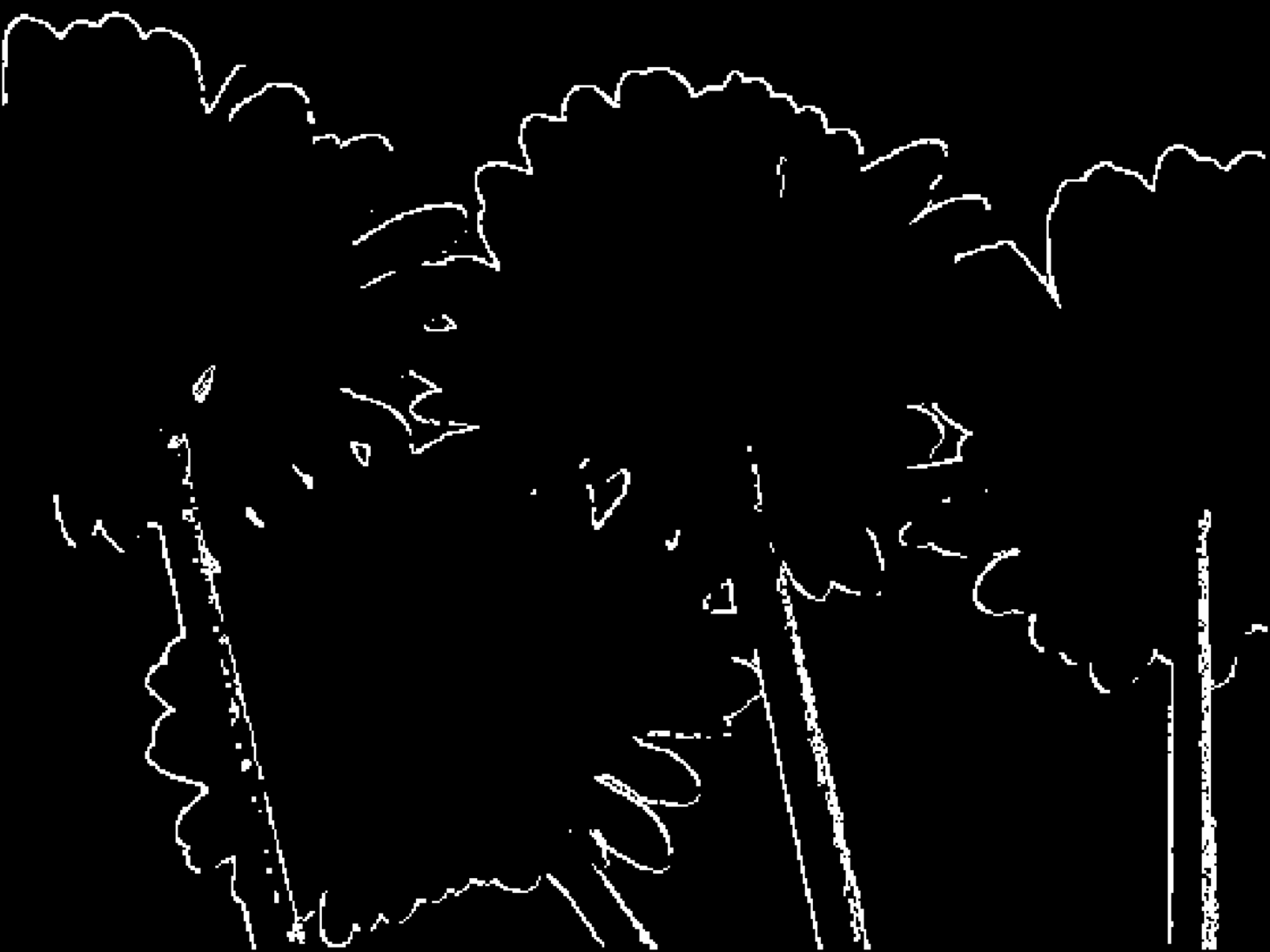}}
\caption{From left to right: the initialization (white lines) of $u_1,u_2,u_3\text{ and }u_4$ obtained by the Multi-IGLIM.}\label{fig3.1}
\end{figure}

\begin{figure}[htbp]
  \centering
\subfigure[original image]{\includegraphics[width=3cm]{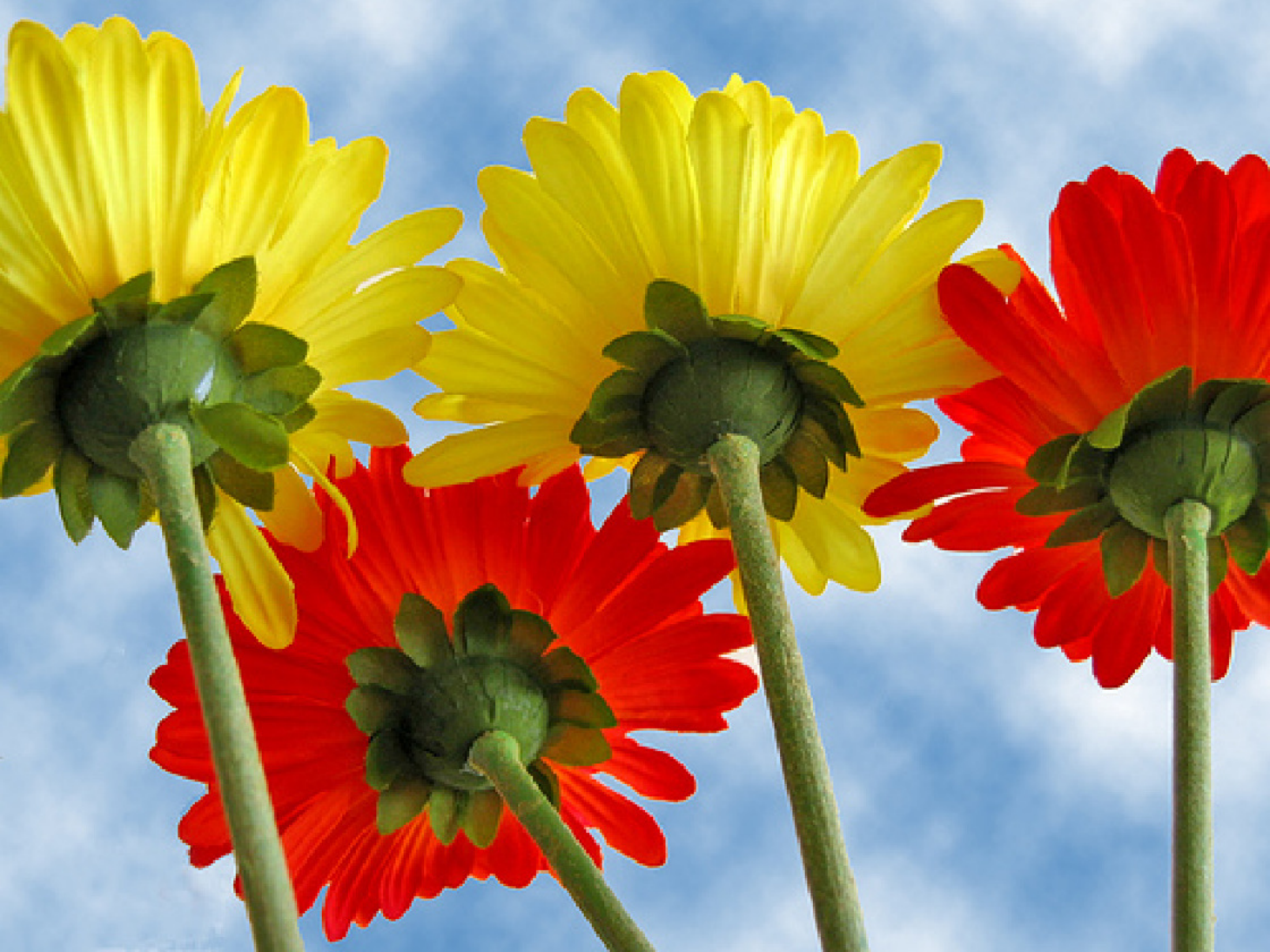}}
\subfigure[final contour]{\includegraphics[width=3cm]{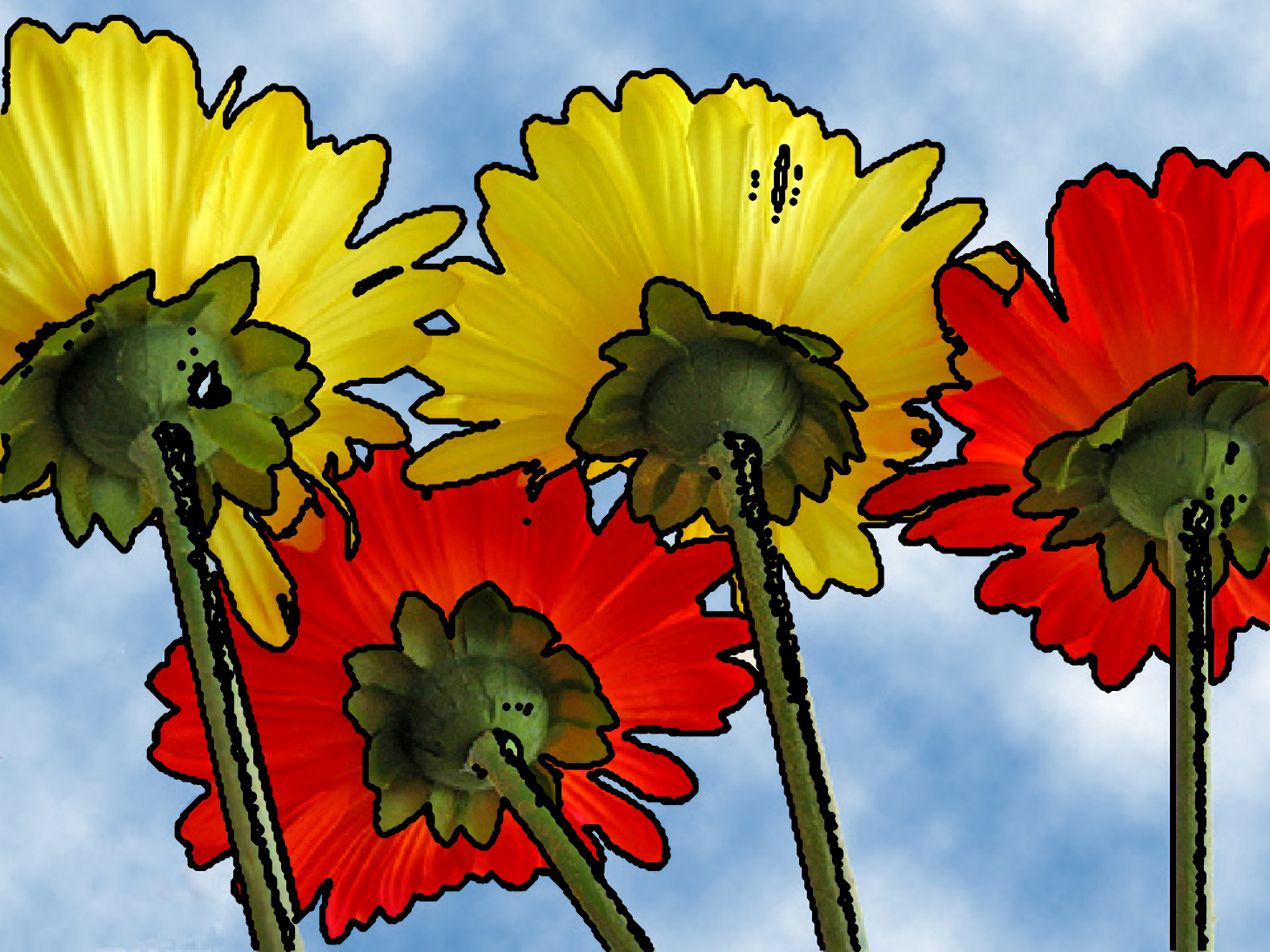}}
\subfigure[final segments]{\includegraphics[width=3cm]{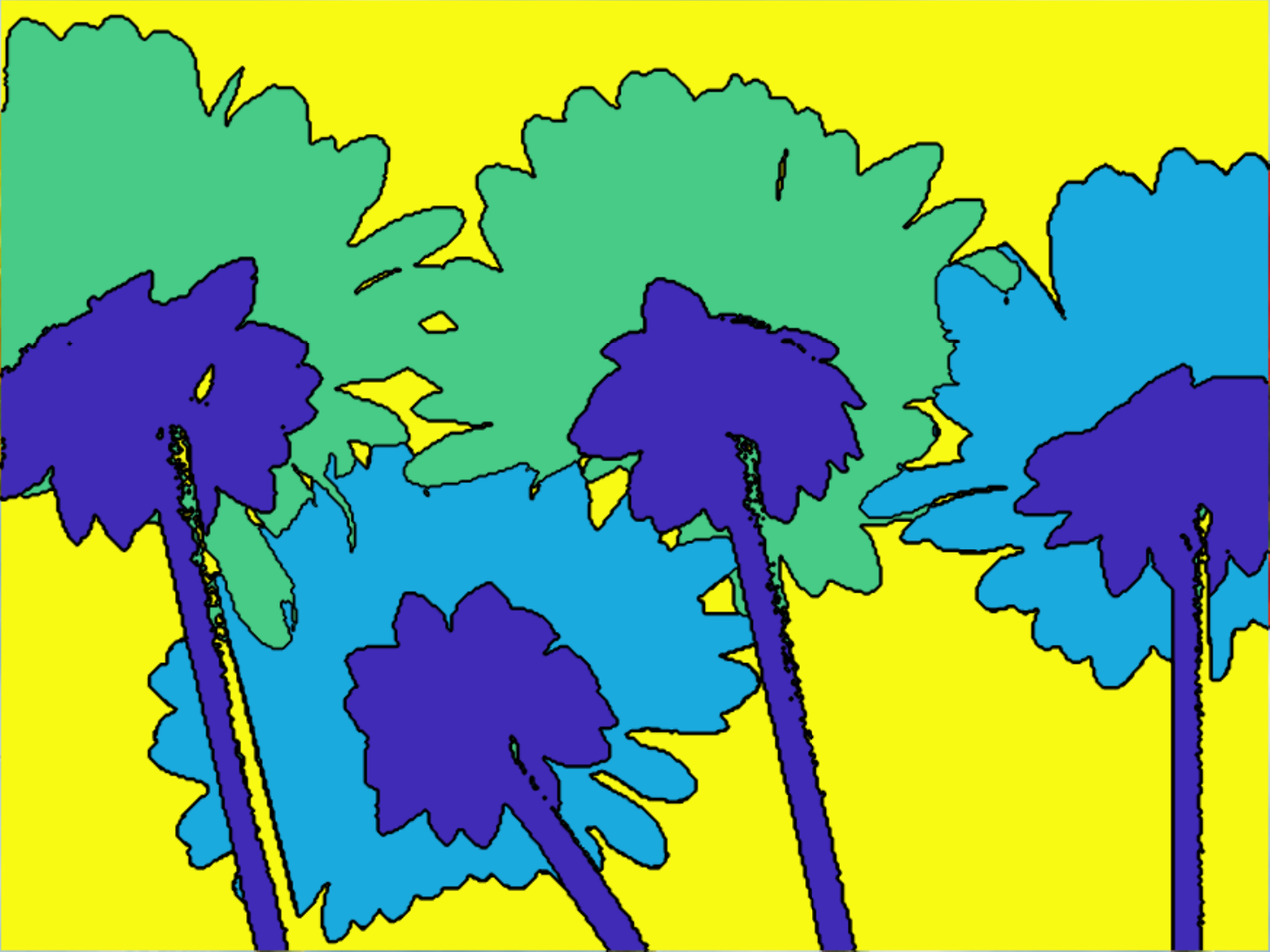}}
\caption{From left to right: original image and the segmentation results.}\label{fig3.2}
\end{figure}

\begin{figure}[htbp]
  \centering
\subfigure[$u_1$ (heart)]{\includegraphics[width=2.8cm,height=2.5cm]{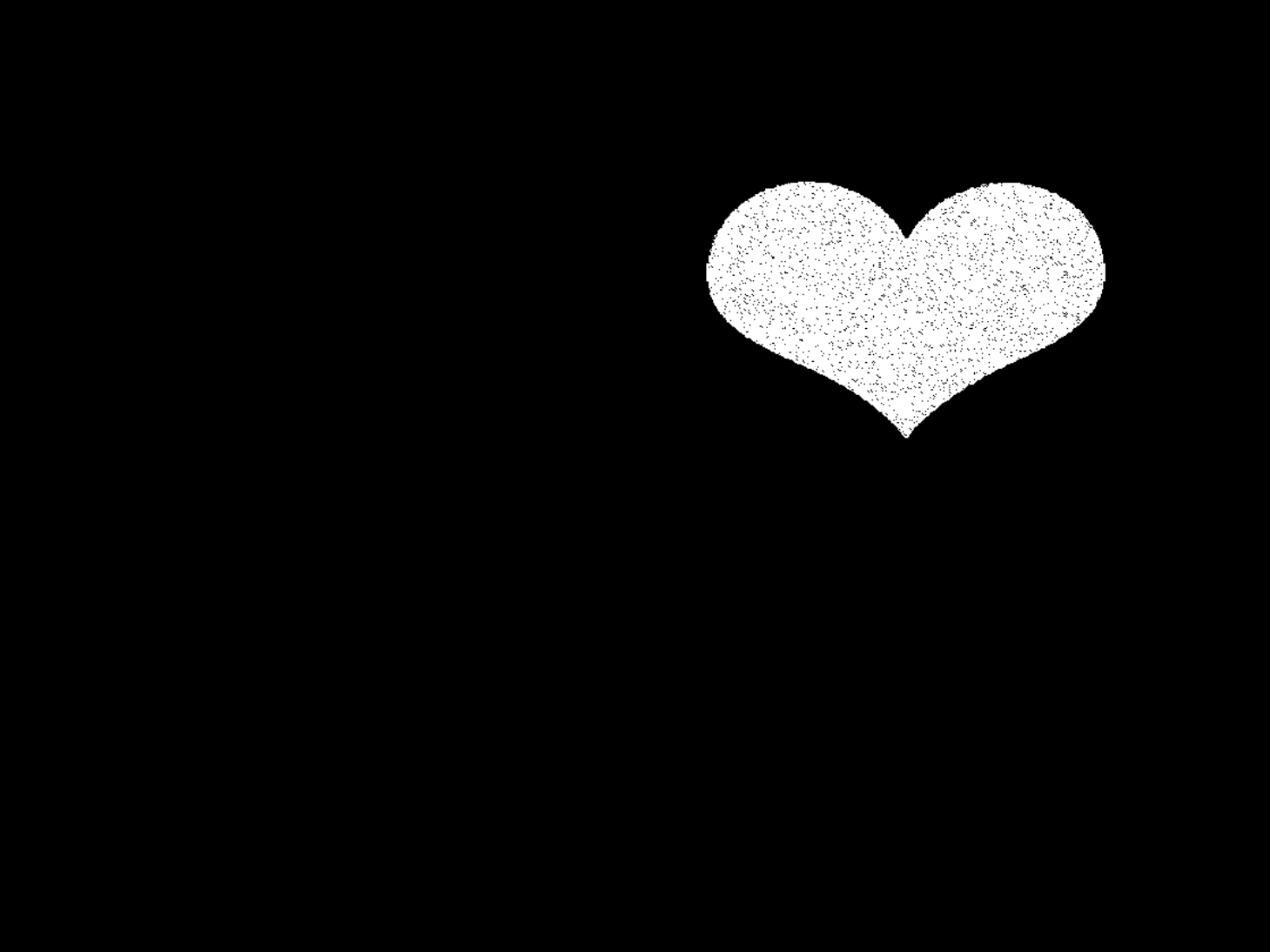}}
\subfigure[$u_2$ (circle)]{\includegraphics[width=2.8cm,height=2.5cm]{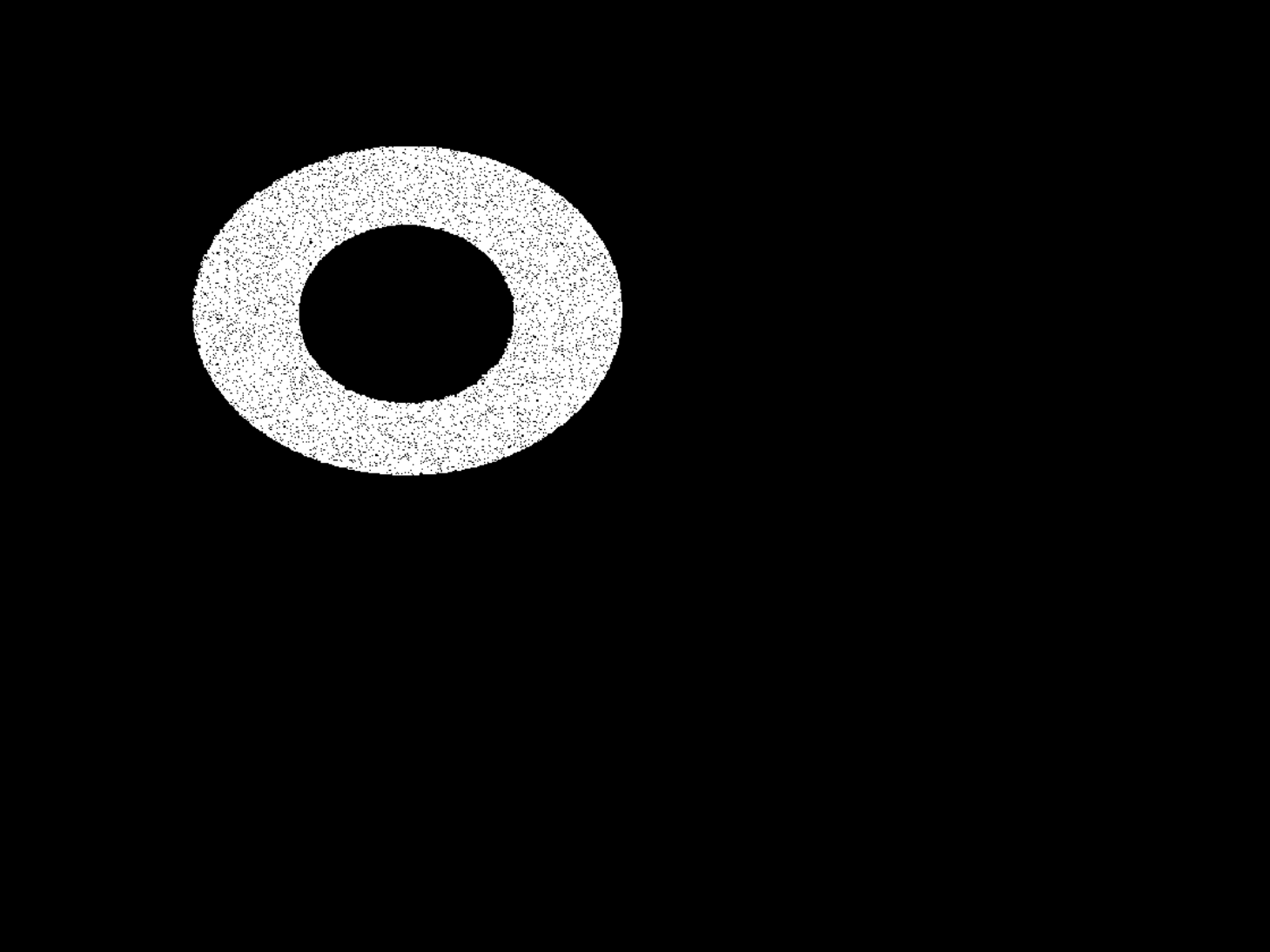}}
\subfigure[$u_3$ (arch)]{\includegraphics[width=2.8cm,height=2.5cm]{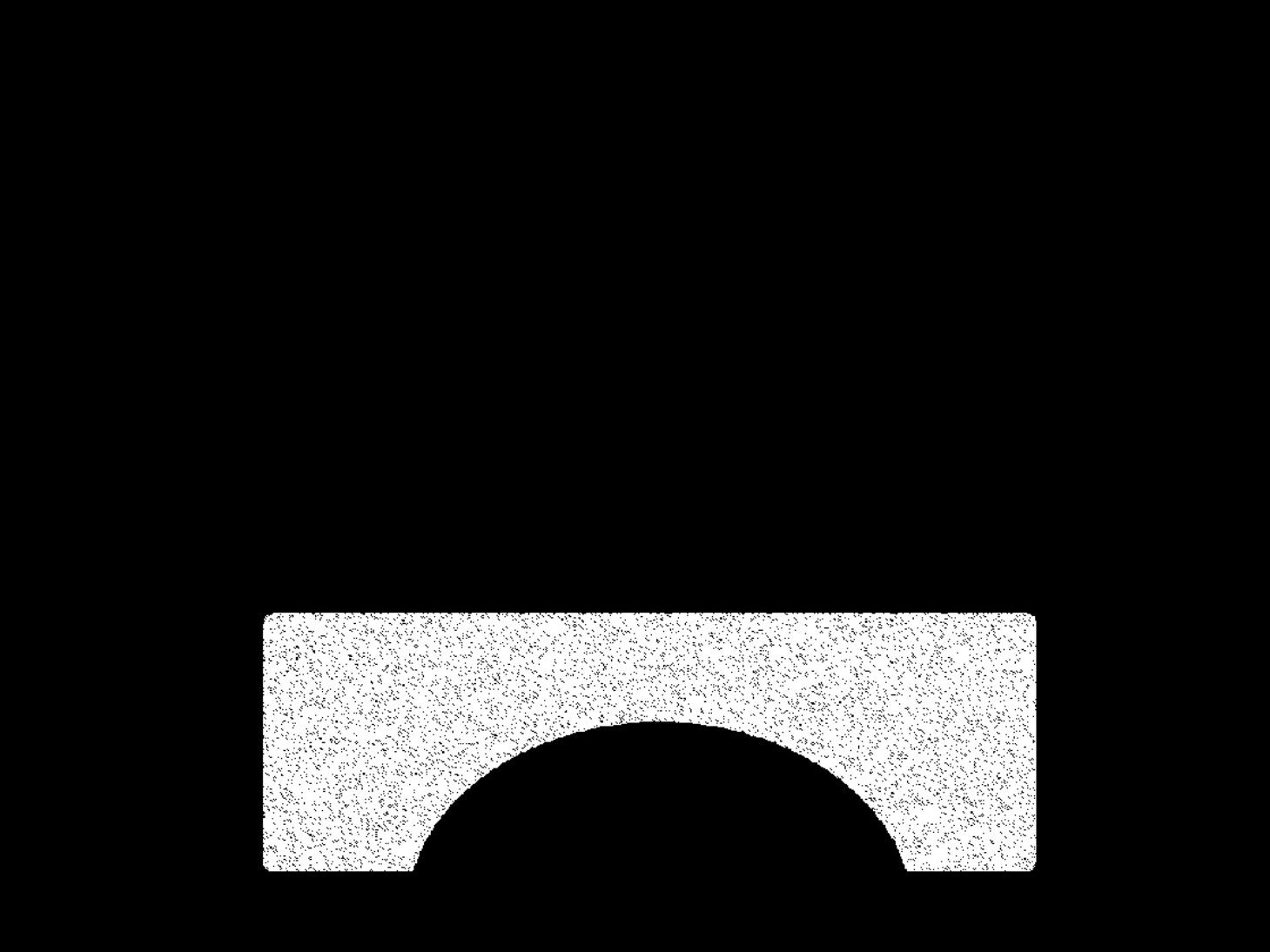}}
\subfigure[$u_4$ (background)]{\includegraphics[width=3cm,height=2.5cm]{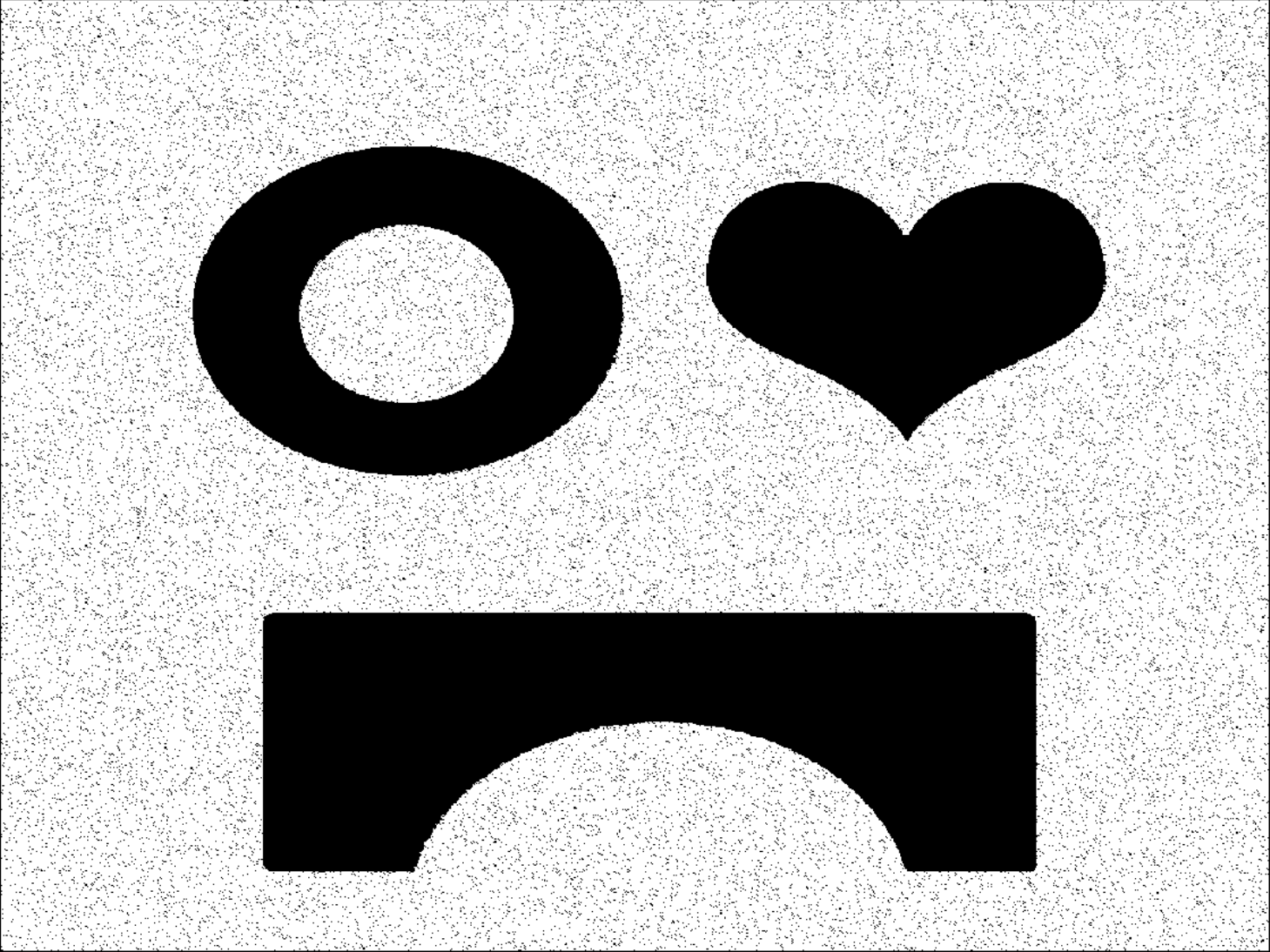}}
\caption{From left to right: the initialization of $u_1,u_2,u_3\text{ and }u_4$ obtained by the Multi-IGLIM.}\label{fig4.2}
\end{figure}

\begin{figure}[htbp]
  \centering
\subfigure[original image]{\includegraphics[width=2.8cm,height=2.5cm]{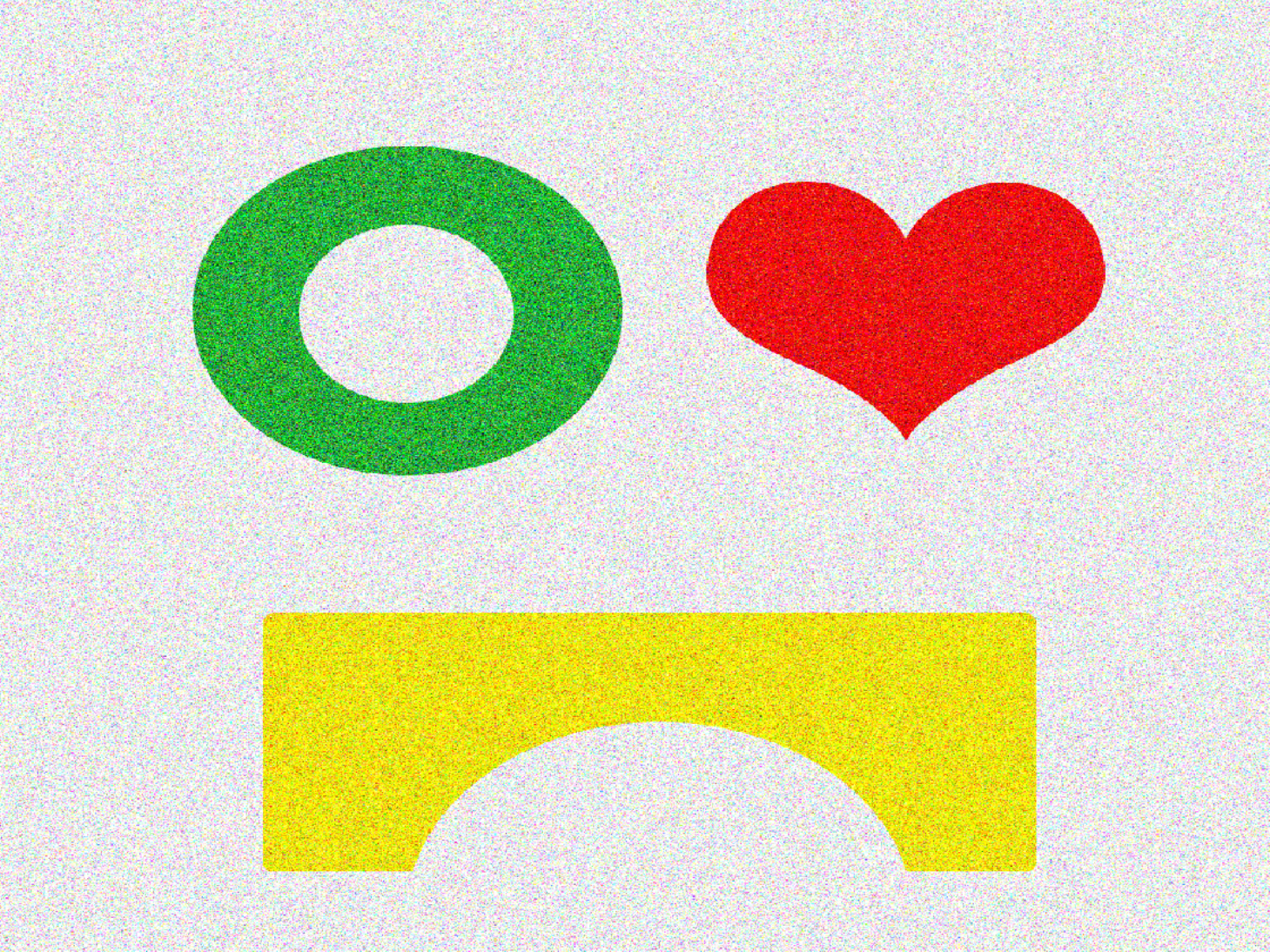}}
\subfigure[final contour]{\includegraphics[width=2.8cm,height=2.5cm]{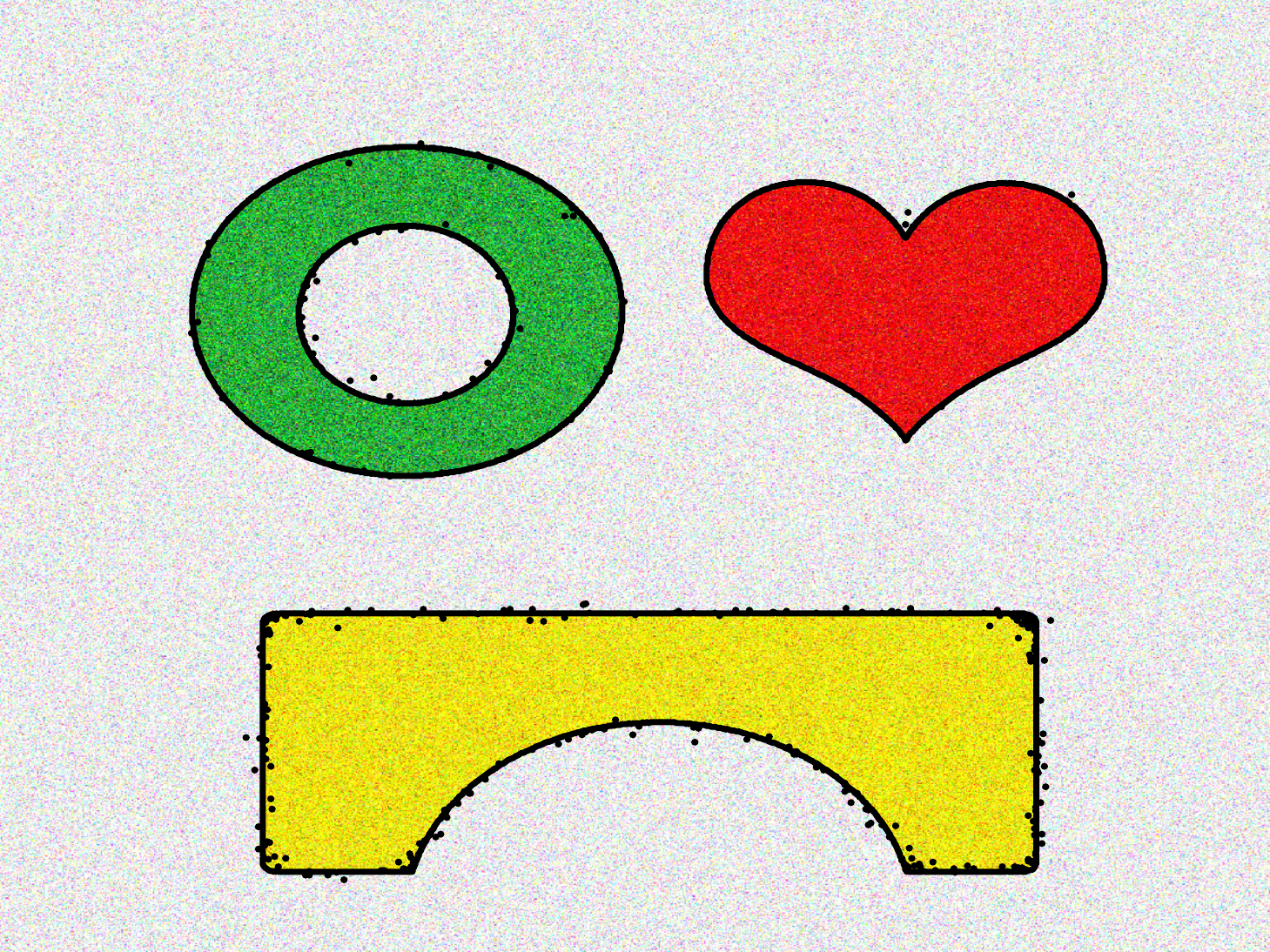}}
\subfigure[final segments]{\includegraphics[width=2.8cm,height=2.5cm]{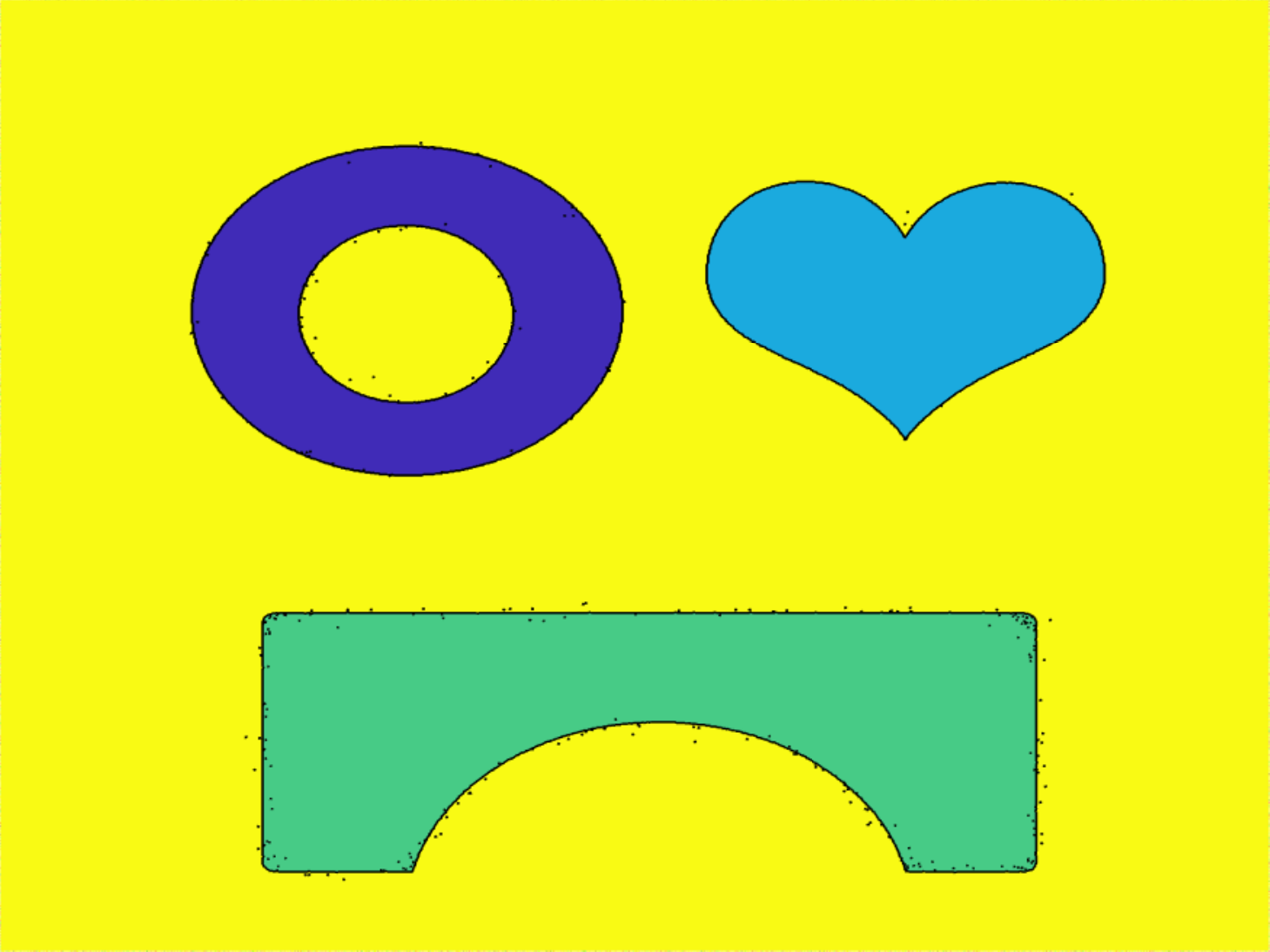}}
\caption{From left to right: original image and the segmentation results.}\label{fig4.1}
\end{figure}

\begin{figure}[htbp]
	\centering
	\subfigure{
	\includegraphics[width=2.8cm,height=2.2cm]{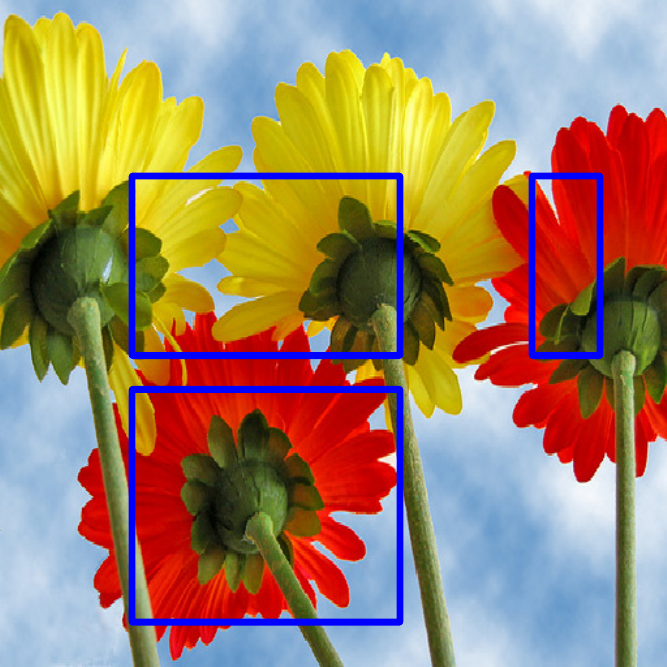}
	\includegraphics[width=2.8cm,height=2.2cm]{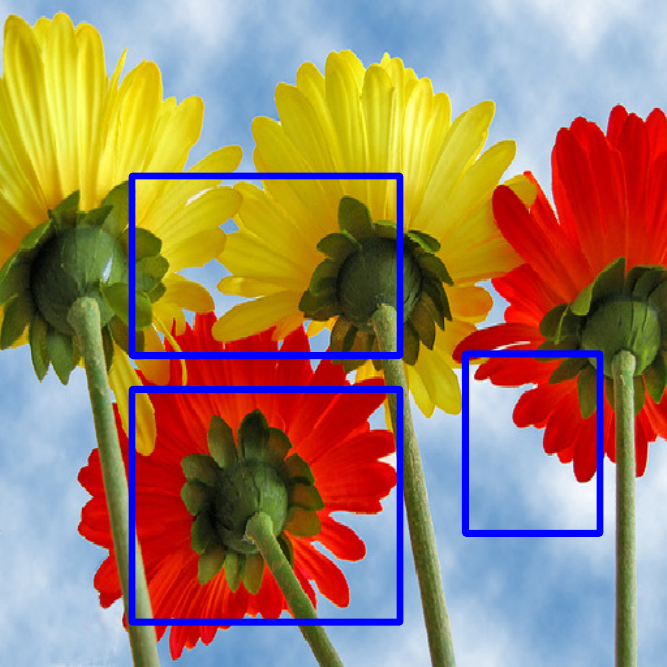}
	\includegraphics[width=2.8cm,height=2.2cm]{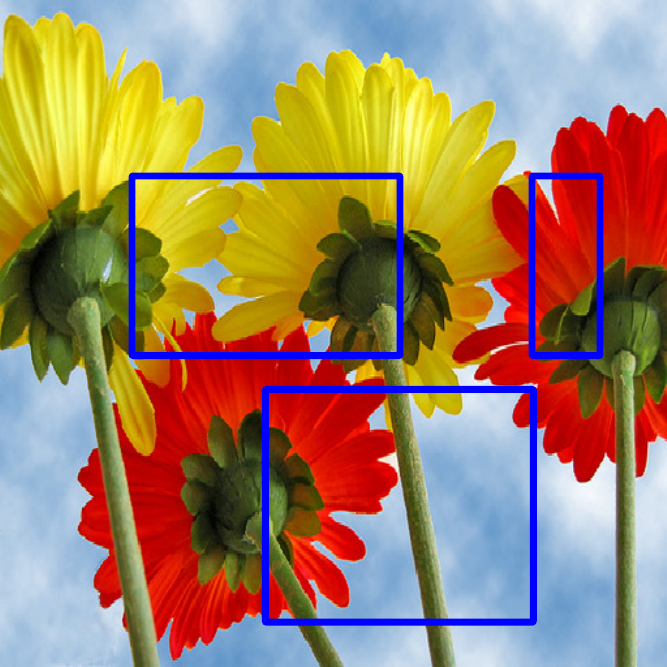}
	\includegraphics[width=2.8cm,height=2.2cm]{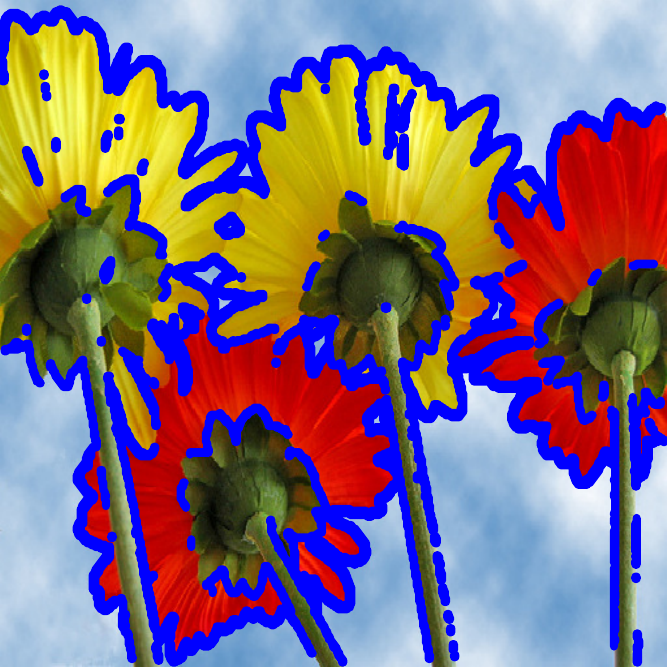}
	}
	\subfigure{
	\includegraphics[width=2.8cm,height=2.2cm]{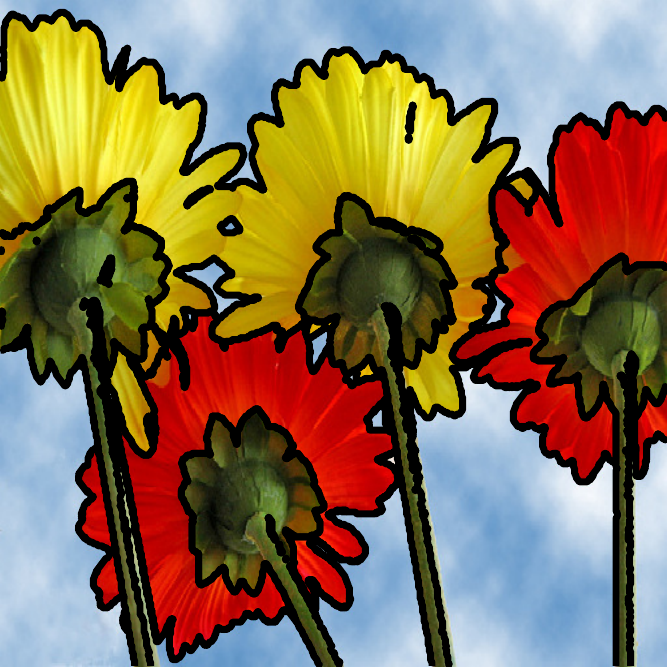}
	\includegraphics[width=2.8cm,height=2.2cm]{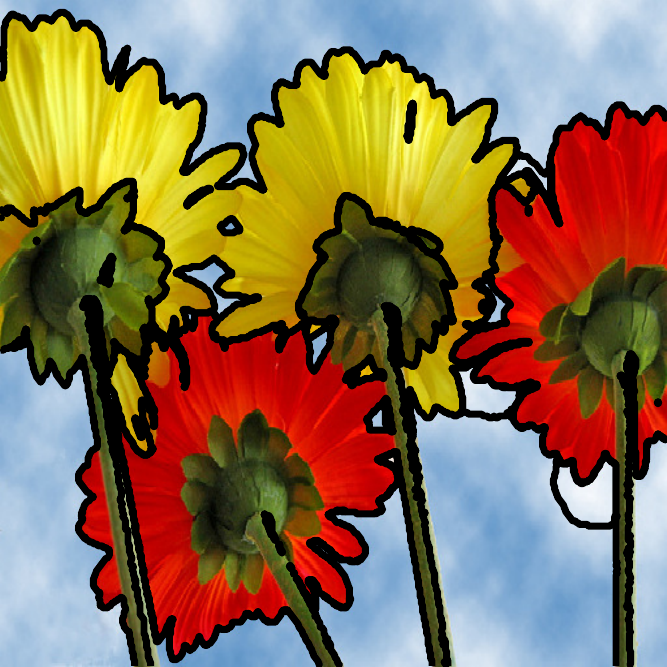}
	\includegraphics[width=2.8cm,height=2.2cm]{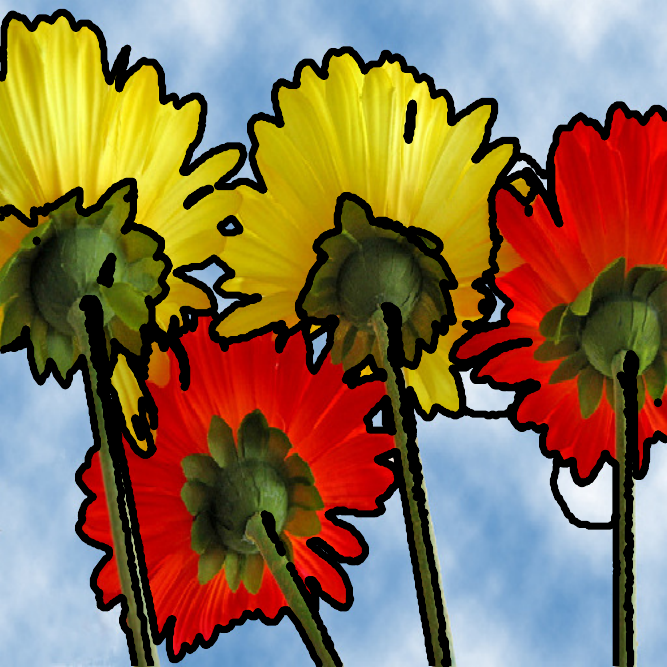}
	\includegraphics[width=2.8cm,height=2.2cm]{flower_seg0-eps-converted-to.pdf}
	}
	\subfigure{
	\includegraphics[width=2.8cm,height=2.2cm]{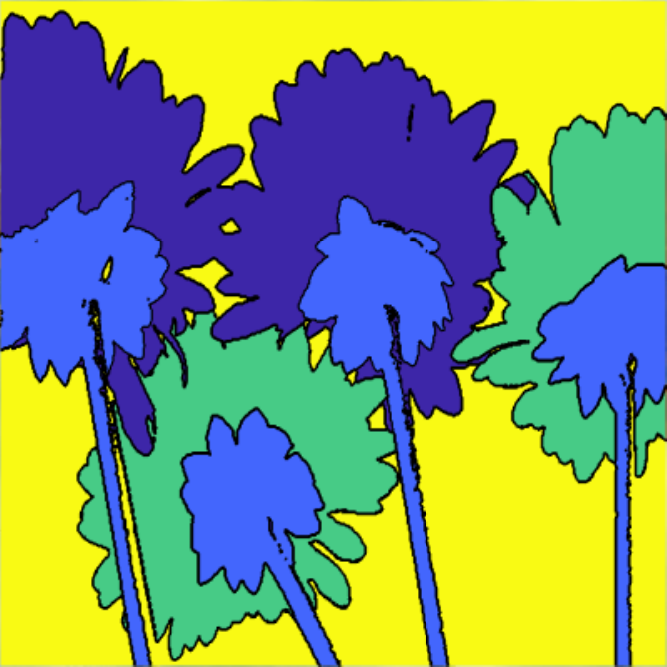}
	\includegraphics[width=2.8cm,height=2.2cm]{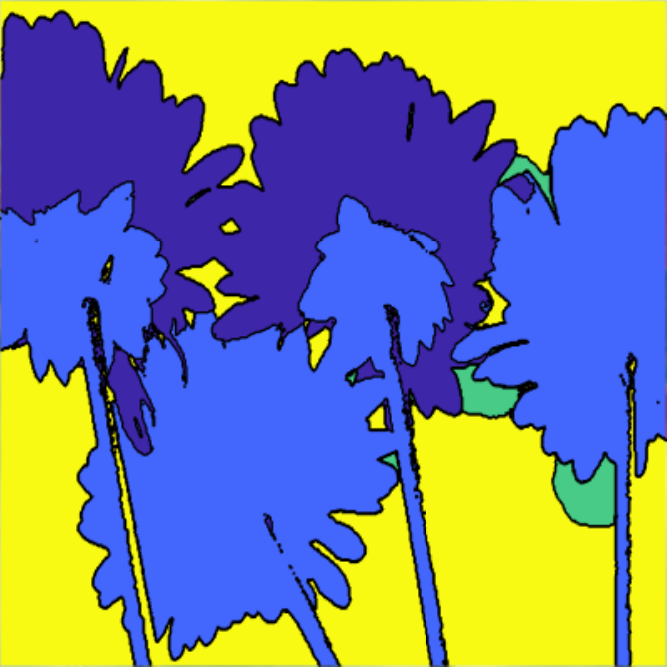}
	\includegraphics[width=2.8cm,height=2.2cm]{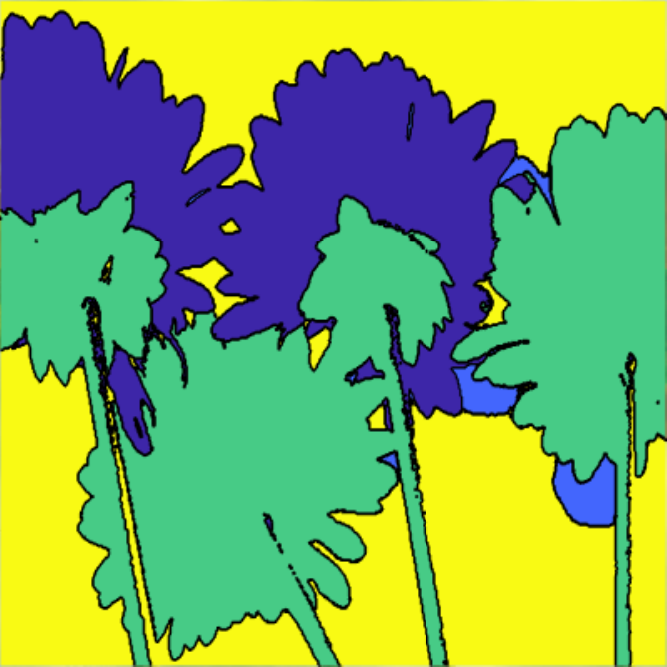}
	\includegraphics[width=2.8cm,height=2.2cm]{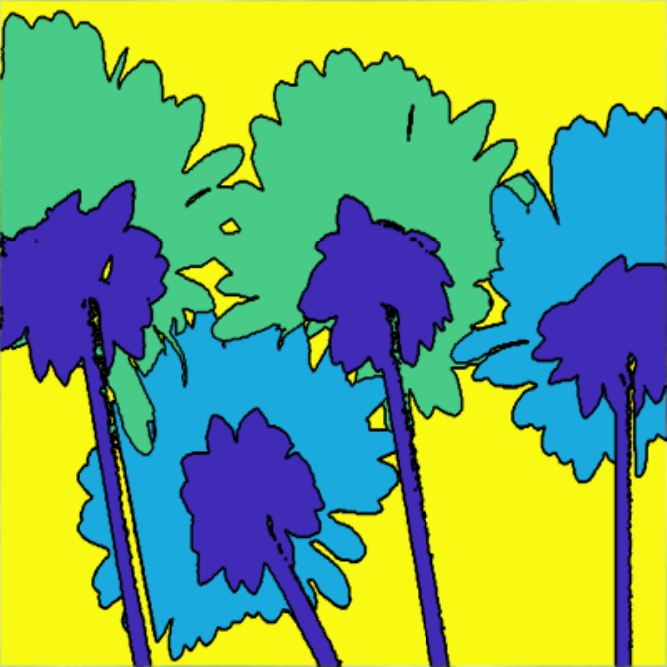}
	}
	 \caption{From top to bottom: Original images with initial contour, final contour, and final segments. All the final contours are highlighted by black lines.}
	 \label{fig_comparison_flower}
\end{figure}

\begin{table}[htbp]
\centering
	\begin{tabular}{lcc}
		\hline
		Initialization	&Iteration number &CPU time (s)\cr
		\hline
		Column 1 &19 &1.001778 \cr
		Column 2 &38  &1.994020  \cr
		Column 3 &182  &9.263411 \cr
     Multi-IGLIM &12  &0.717589\cr\hline
	\end{tabular}
	\label{table2}\caption{Comparison of the iteration number and CPU time for the segmentation in \cref{fig_comparison_flower}.}
\end{table}

\subsection{ICTM-LVF for images with noise}
In this part, we evaluate the performance of the ICTM-LVF solver for both the CV model and the LIF model on synthetic and medical images. Comparison is made between the original ICTM and the ICTM-LVF to demonstrate the effectiveness of the ICTM-LVF in the segmentation of the images with noise. All the initial contours are given by the Multi-IGLIM.

First, we solve the CV model for the segmentation of a clean synthetic color image and its noise version.  We display the segmentation results from the ICTM and the ICTM-LVF in \cref{fig_lvf_cv}. One can observe that the ICTM produces good segmentation for the clean image, but its performance on the noise image is not so appealing, while the ICTM-LVF performs quite well for both clean and noise images.

\begin{figure}[htbp]
	\centering
	\subfigure{
		\includegraphics[width=2.2cm]{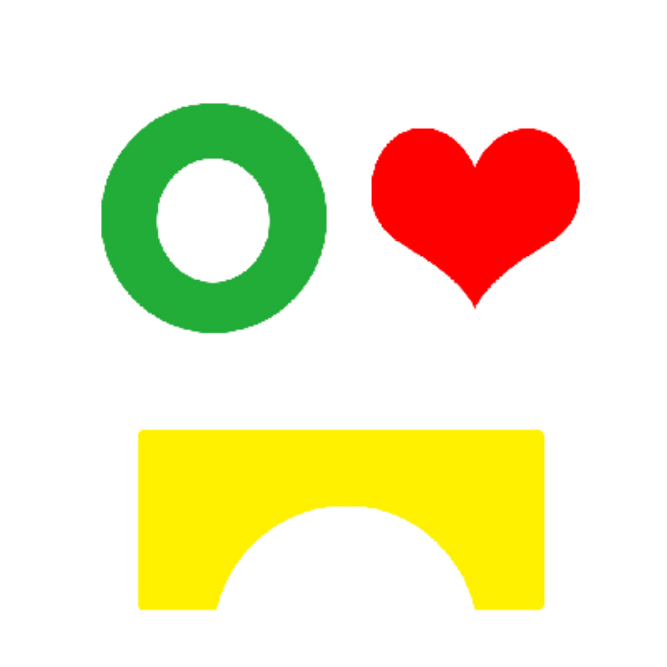}
		\includegraphics[width=2.2cm]{n3object0-eps-converted-to.pdf}
		\includegraphics[width=2.2cm]{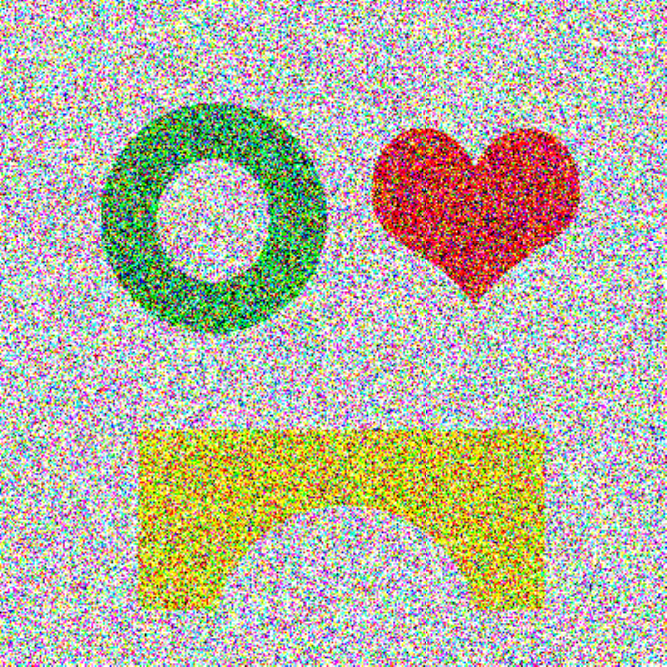}
		\includegraphics[width=2.2cm]{n3object1-eps-converted-to.pdf}
	}
	\subfigure{
		\includegraphics[width=2.2cm]{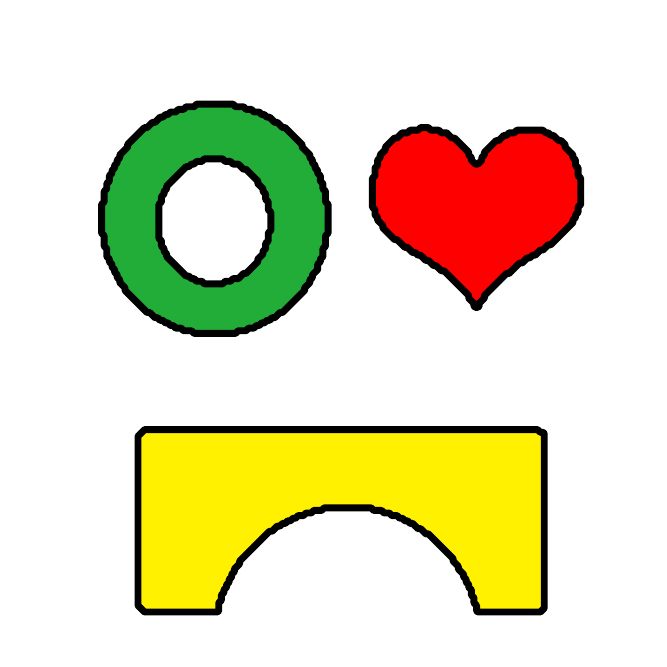}
		\includegraphics[width=2.2cm]{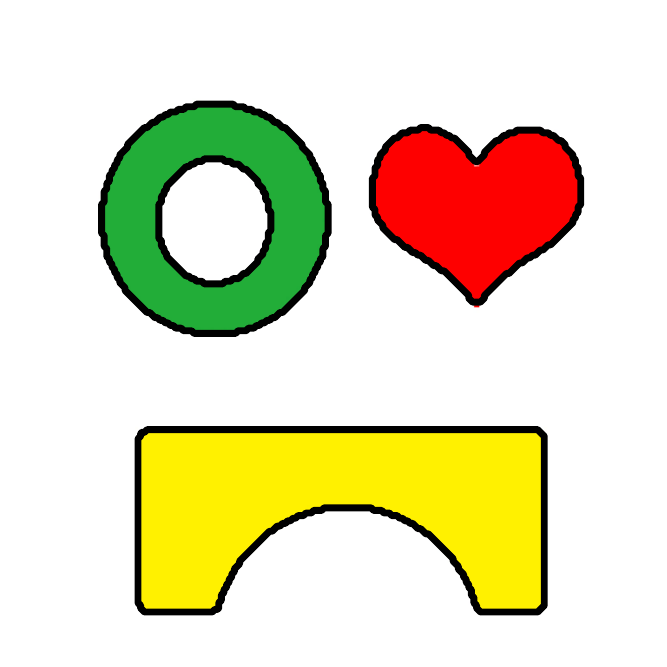}
		\includegraphics[width=2.2cm]{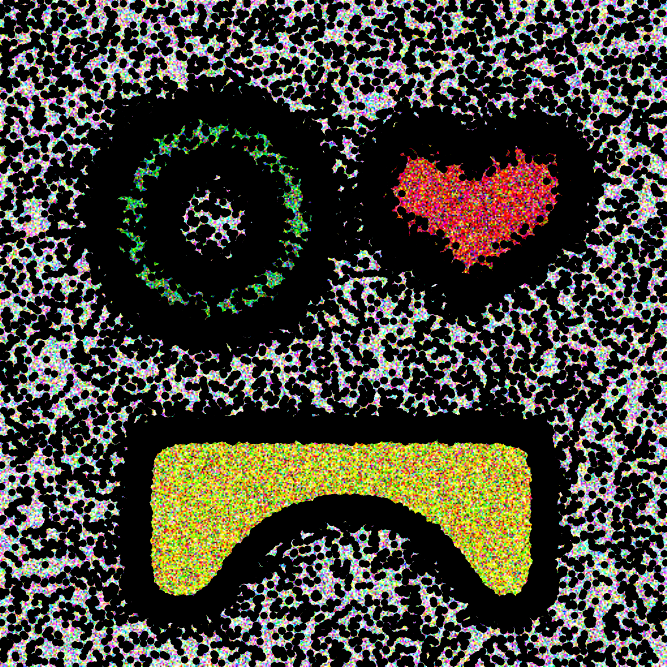}
		\includegraphics[width=2.2cm]{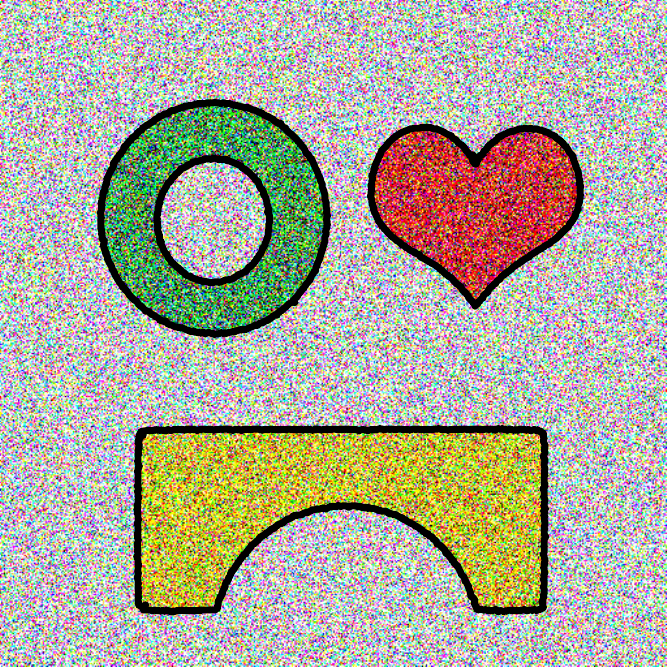}
	}
	\subfigure{
	\includegraphics[width=2.2cm]{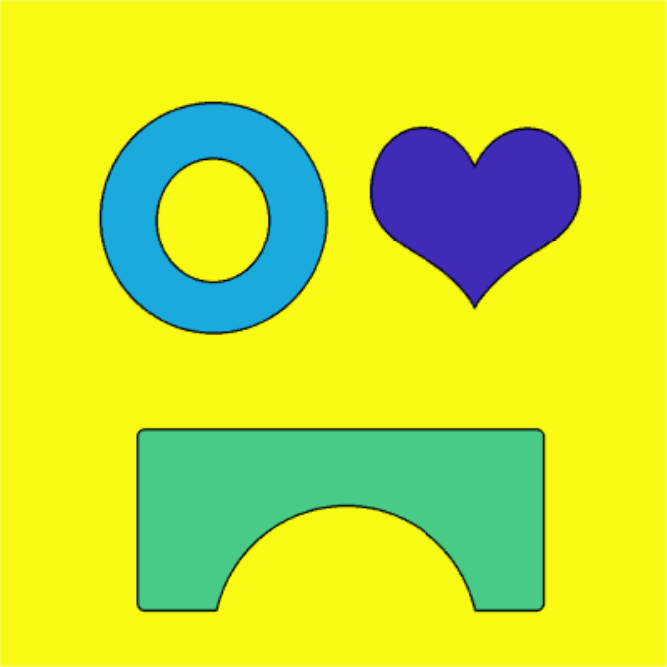}
	\includegraphics[width=2.2cm]{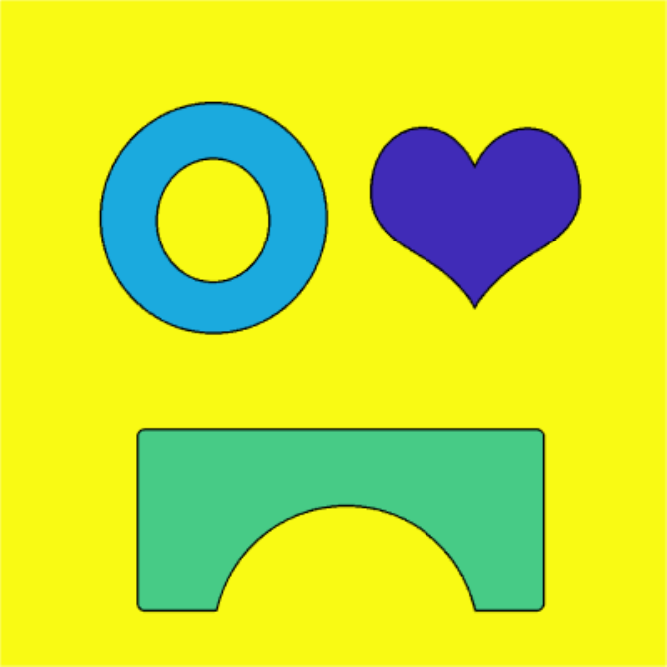}
	\includegraphics[width=2.2cm]{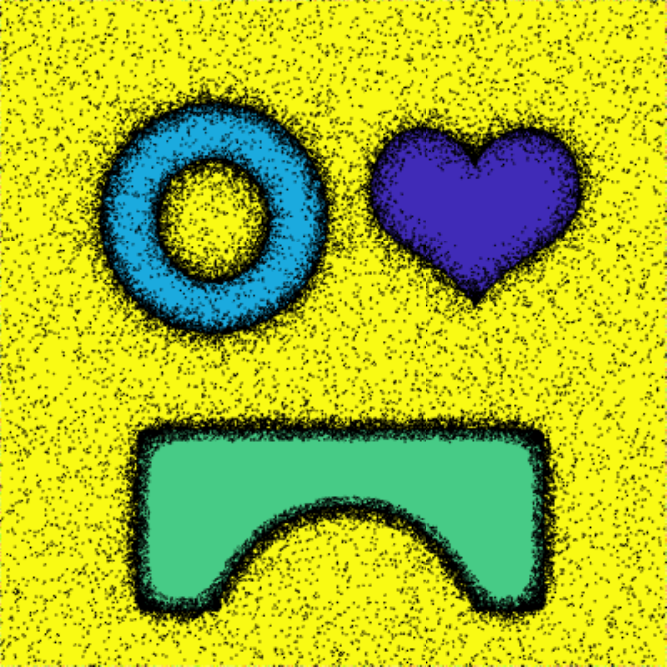}
	\includegraphics[width=2.2cm]{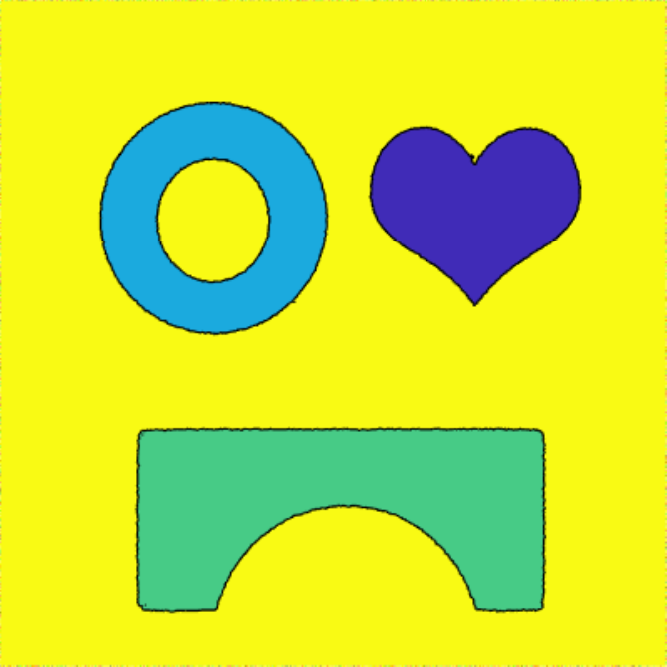}
	}
	 \caption{ Comparison between the ICTM and the ICTM-LVF for solving the CV model. The first row displays the clean and noise images, the second row and the third row are the results of final contours and final segments, respectively. The first and the third columns are the results of the ICTM, while the results of the ICTM-LVF appear in the second and the fourth columns. All the final contours are highlighted by black lines.} \label{fig_lvf_cv}
\end{figure}

Then, we solve the LIF model for the segmentation of images with noise at different levels with both the ICTM and the ICTM-LVF. The results are shown in \cref{fig_lvf_lif}. It can be seen that as the noise level increases, the results of the ICTM become worse, while the ICTM-LVF can still achieve relatively good segmentation.
\begin{figure}[htbp]
	\centering
	\subfigure{
		\includegraphics[width=2.2cm]{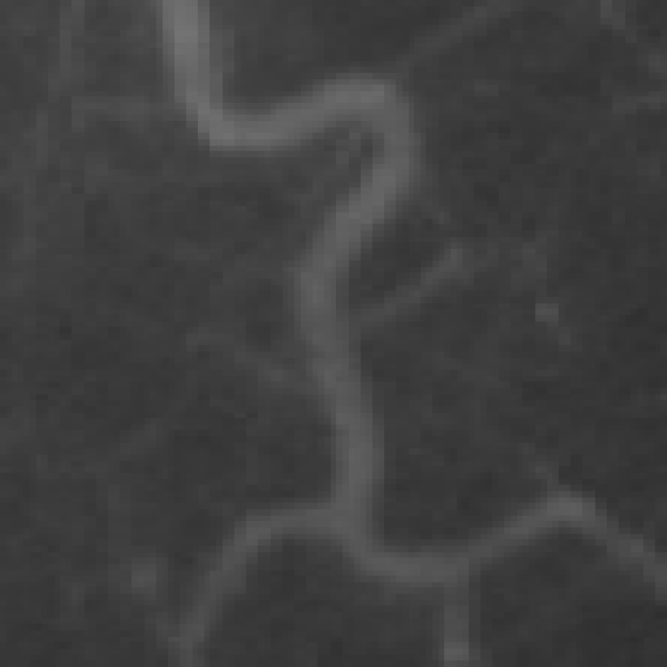}
		\includegraphics[width=2.2cm]{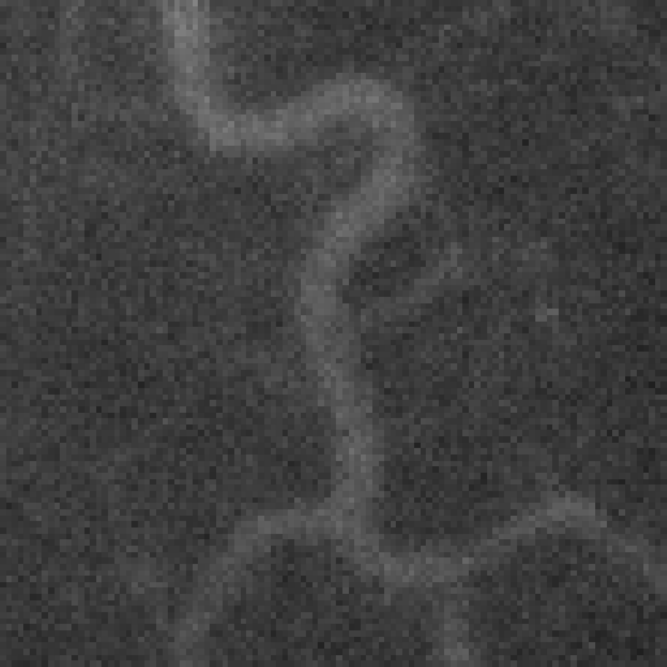}
		\includegraphics[width=2.2cm]{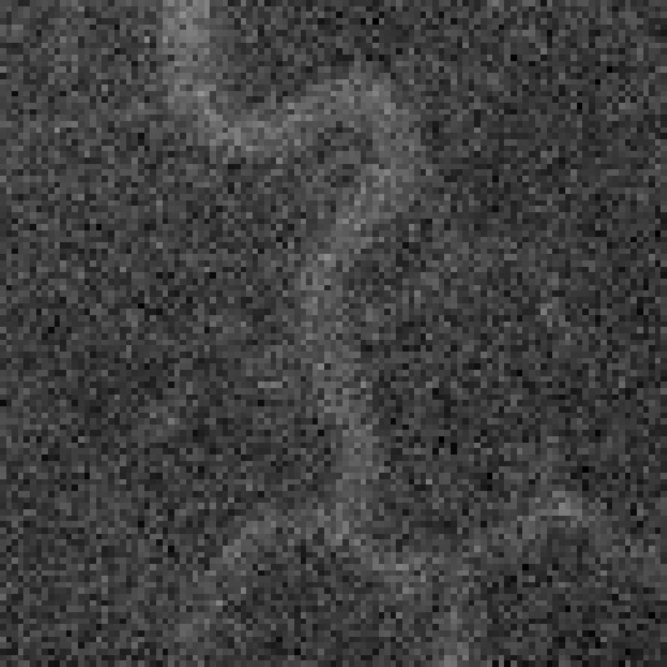}
		\includegraphics[width=2.2cm]{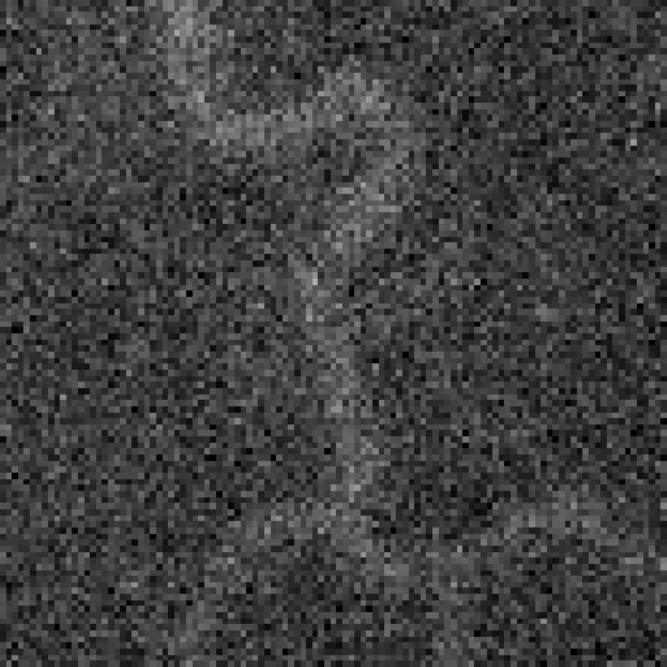}
	}
	\subfigure{
		\includegraphics[width=2.2cm]{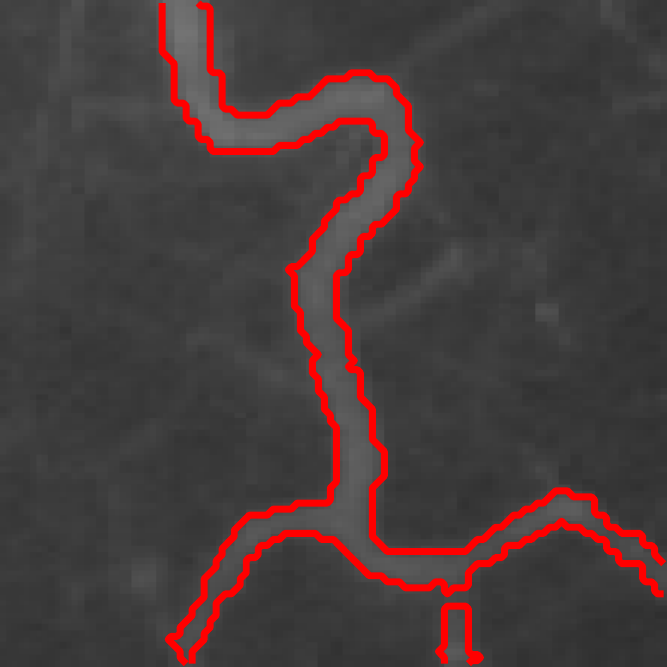}
		\includegraphics[width=2.2cm]{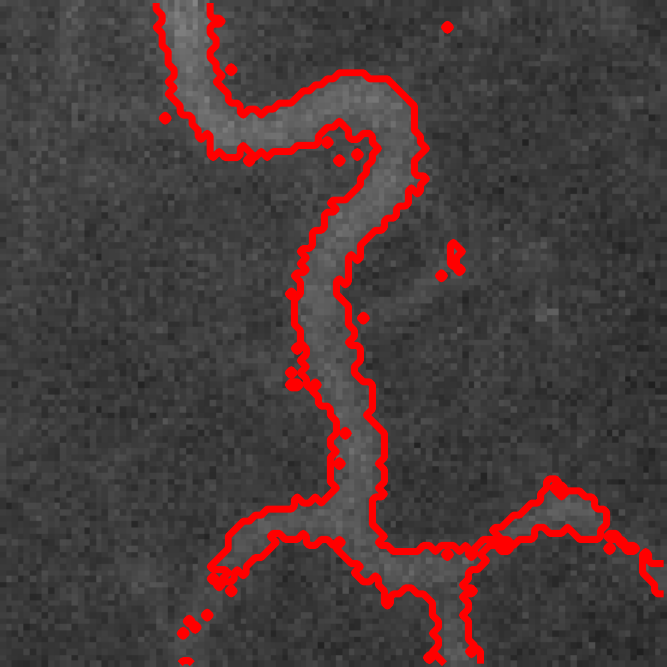}
		\includegraphics[width=2.2cm]{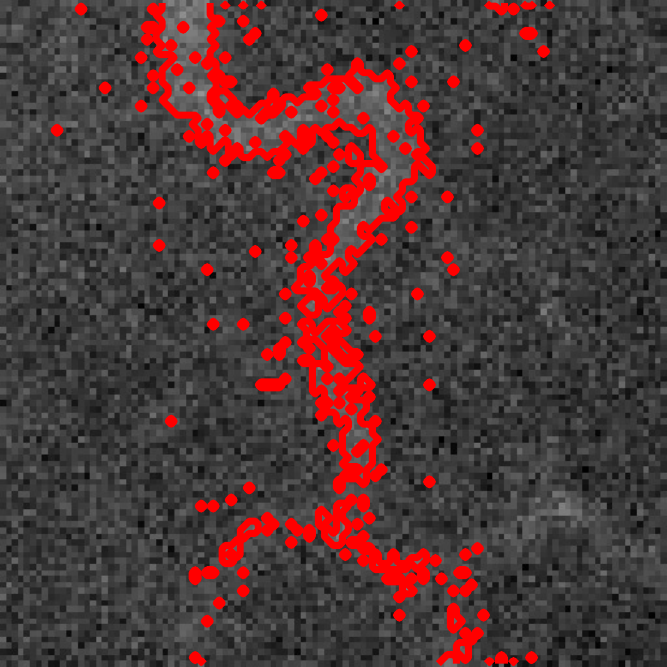}
		\includegraphics[width=2.2cm]{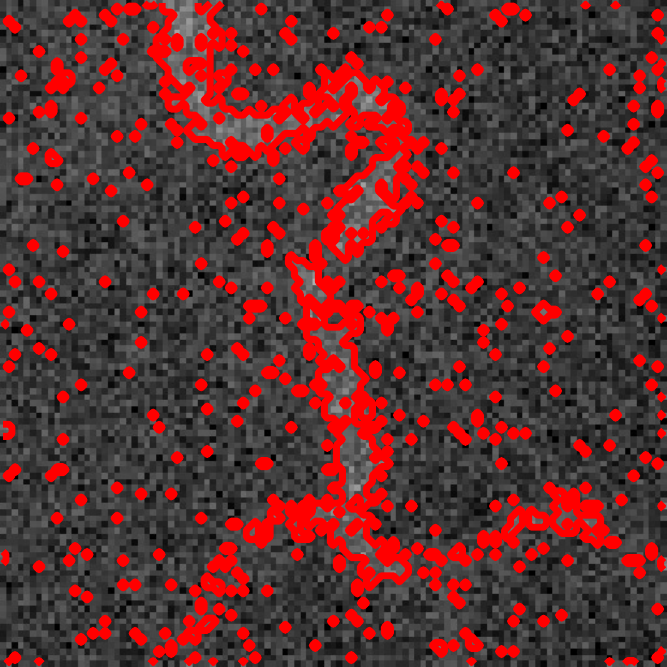}
	}
	\subfigure{
	\includegraphics[width=2.2cm]{ictm_lbf0-eps-converted-to.pdf}
	\includegraphics[width=2.2cm]{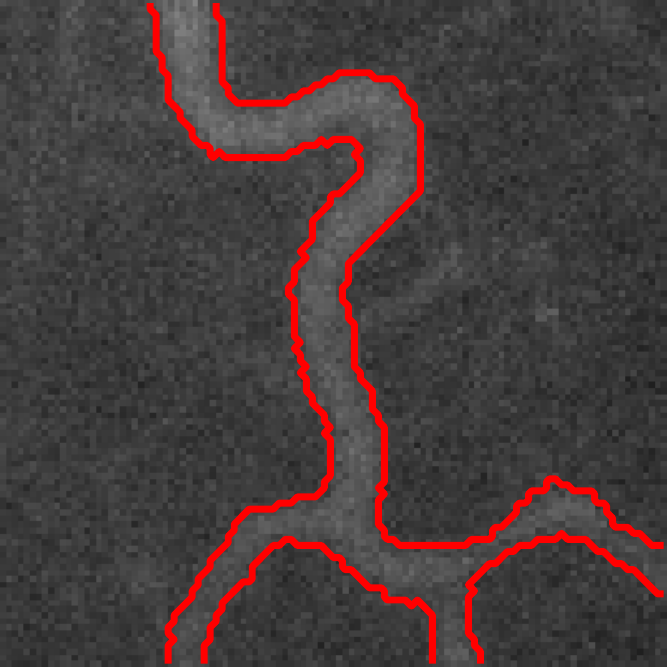}
	\includegraphics[width=2.2cm]{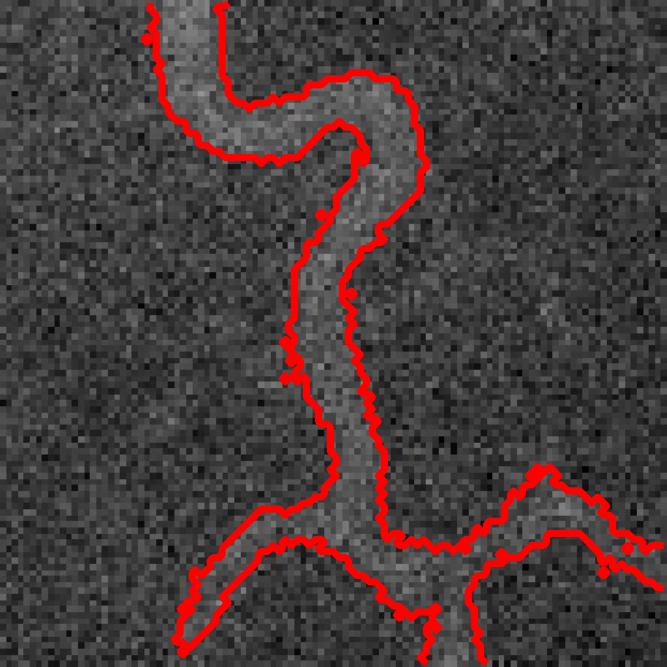}
	\includegraphics[width=2.2cm]{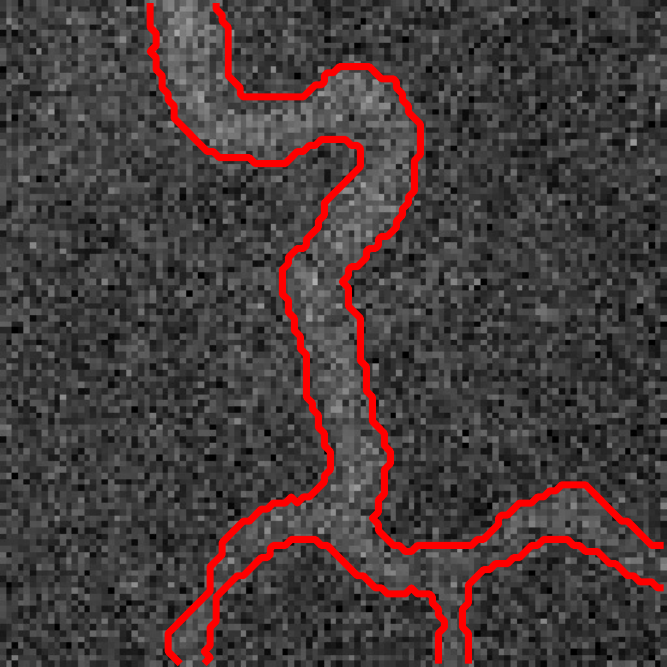}
	}
	 \caption{ Comparison between the ICTM and the ICTM-LVF for solving the LIF model on the segmentation of images with the noise of different levels (The variance from the first column to the last column: 0, 50, 300, 500). From top to bottom are original images with the noise of different levels, results of the ICTM, and results of the ICTM-LVF, respectively. All the final contours are highlighted by red lines.} \label{fig_lvf_lif}
\end{figure}

\subsection{ICTM-LVF for real images}
To further test the performance of the ICTM-LVF solver, we use it to solve the CV model for the segmentation of various real images, where the Multi-IGLIM is adopted for the initialization. The results are displayed in \cref{fig_lvf_realim}. It can be seen that our ICTM-LVF solver also performs very well on real images.

For these real images, we employ the dimension lifting technique to extract more information, which is first proposed in the SLaT model \cite{cai2017three} and drawn on by many other models \cite{cai2017three,wu2021color}. The dimension lifting in this paper is to add the color information of the CIELAB color space \cite{luong1993color,paschos2001perceptually} to the RGB images to reflect the color information in a more comprehensive way. Consequently, the 3D vector used to present RGB information will be extended to a 6D vector.

\begin{figure}[htbp]
	\centering
	\subfigure{
		\includegraphics[width=2cm]{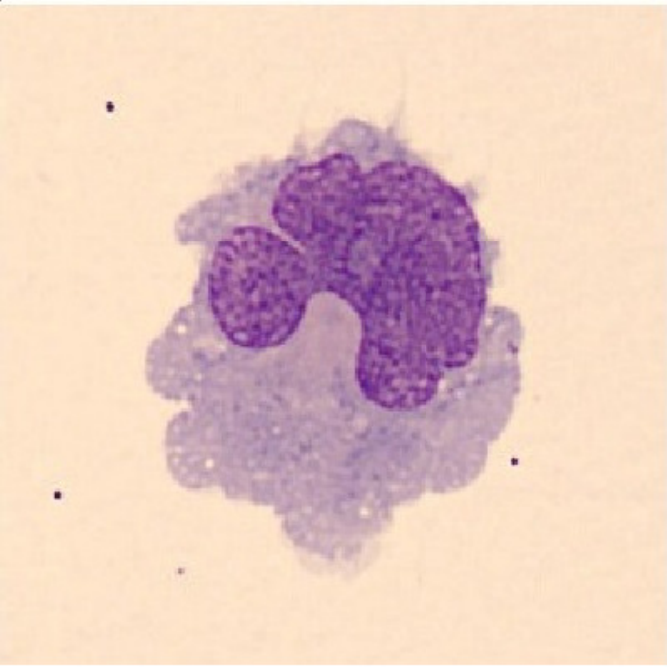}
		\includegraphics[width=2cm]{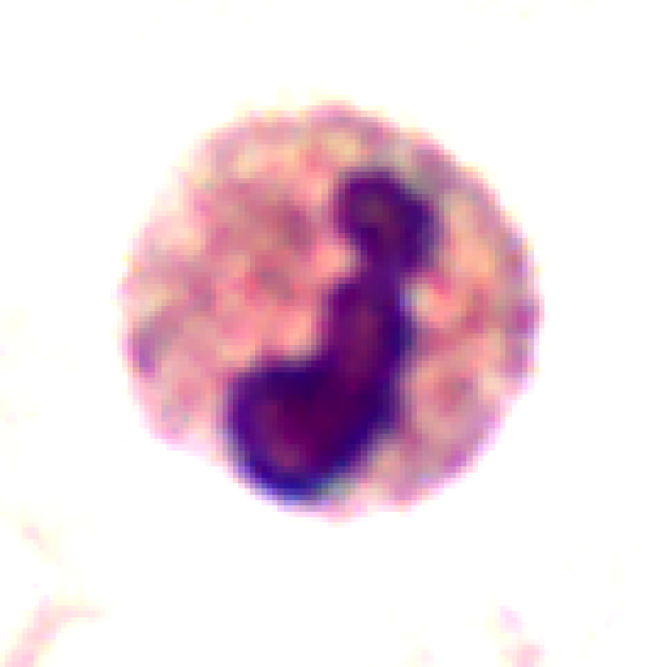}
		\includegraphics[width=2cm]{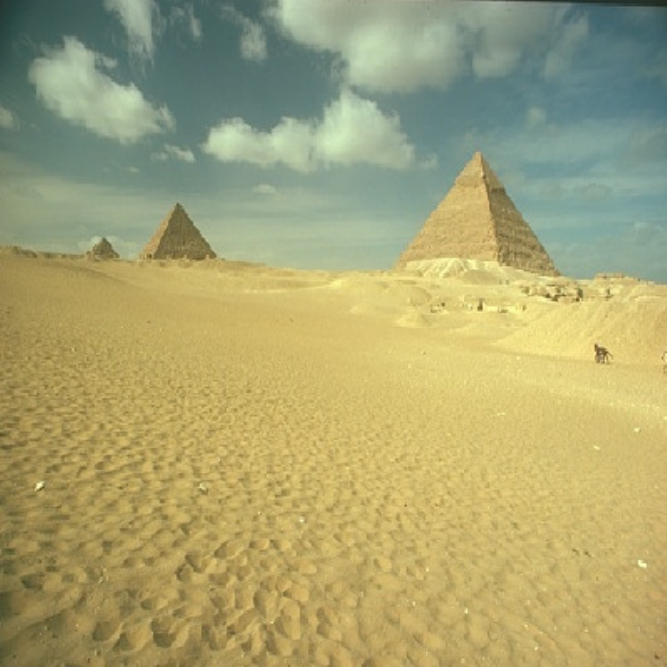}
		\includegraphics[width=2cm]{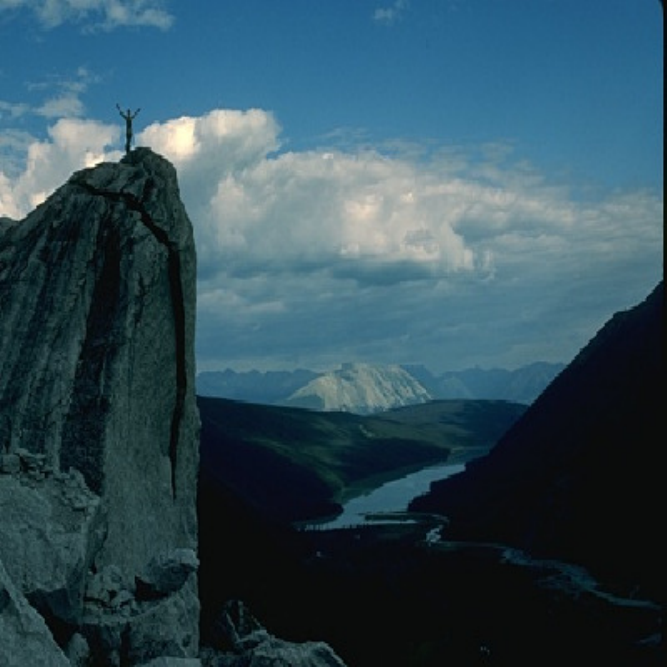}
		\includegraphics[width=2cm]{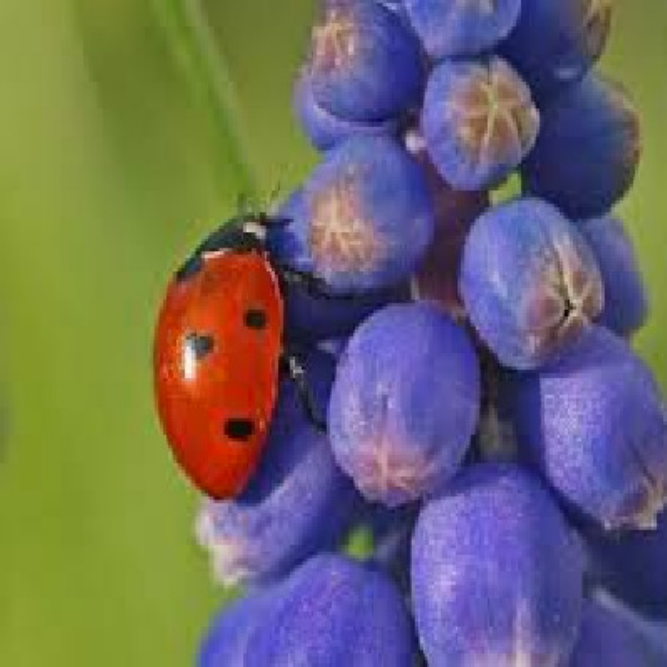}
		\includegraphics[width=2cm]{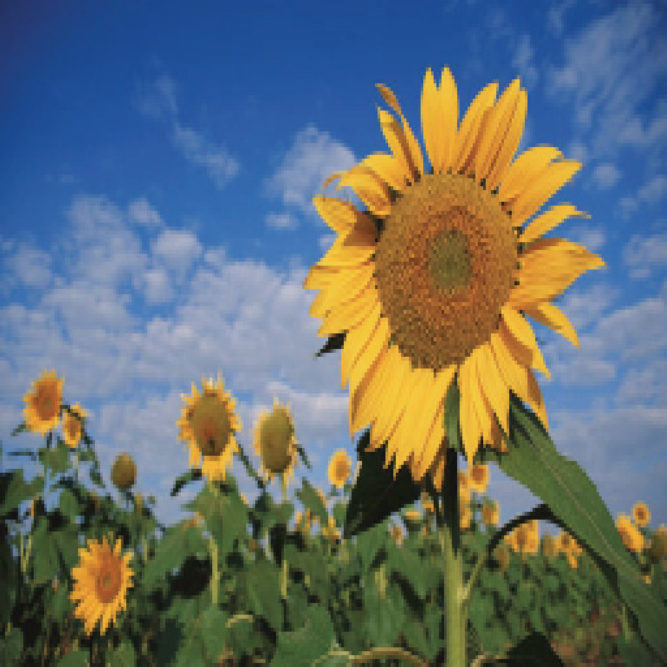}
		\includegraphics[width=2cm]{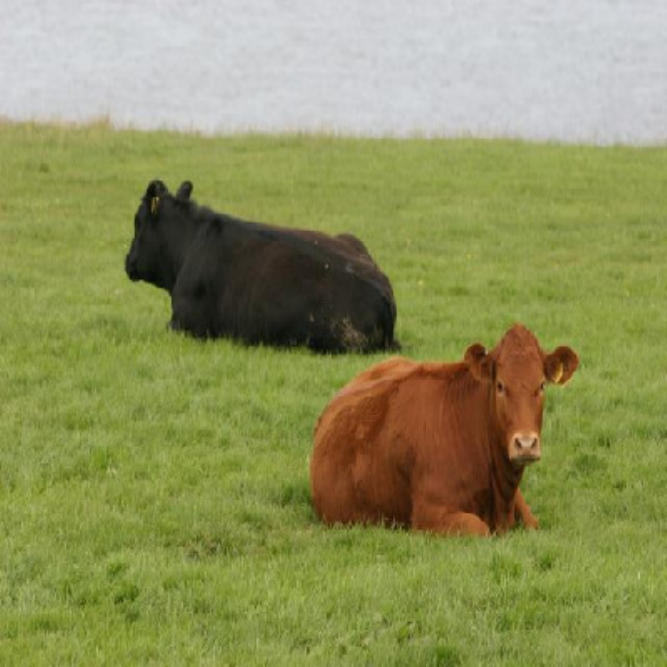}
	}
	\subfigure{
		\includegraphics[width=2cm]{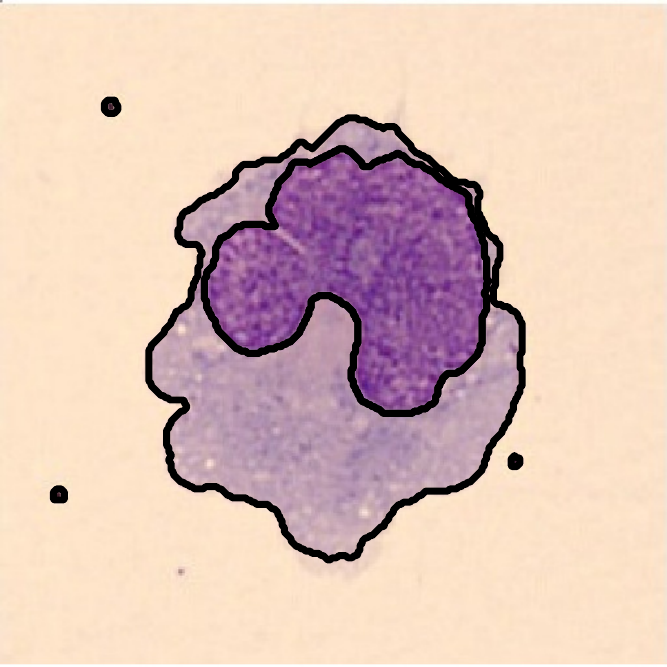}
		\includegraphics[width=2cm]{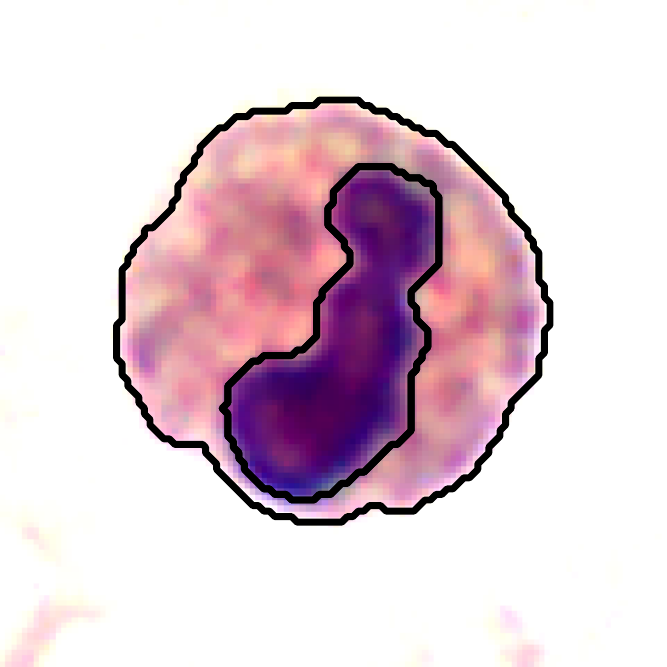}
		\includegraphics[width=2cm]{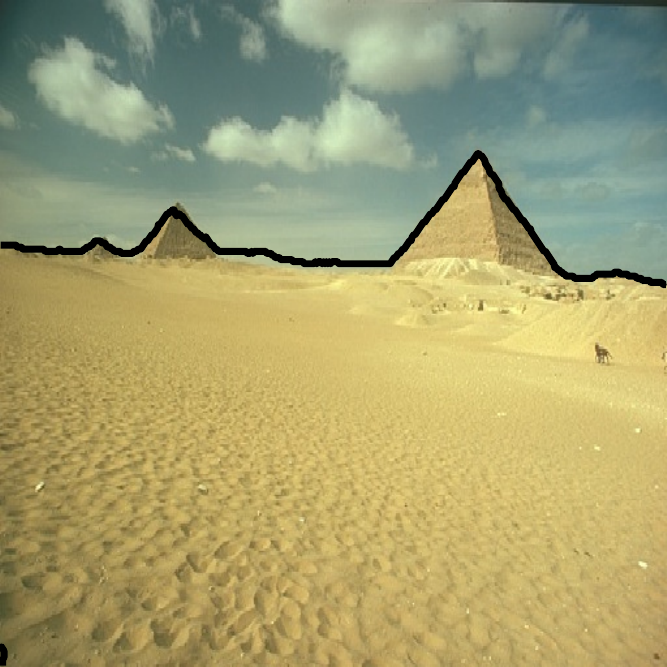}
		\includegraphics[width=2cm]{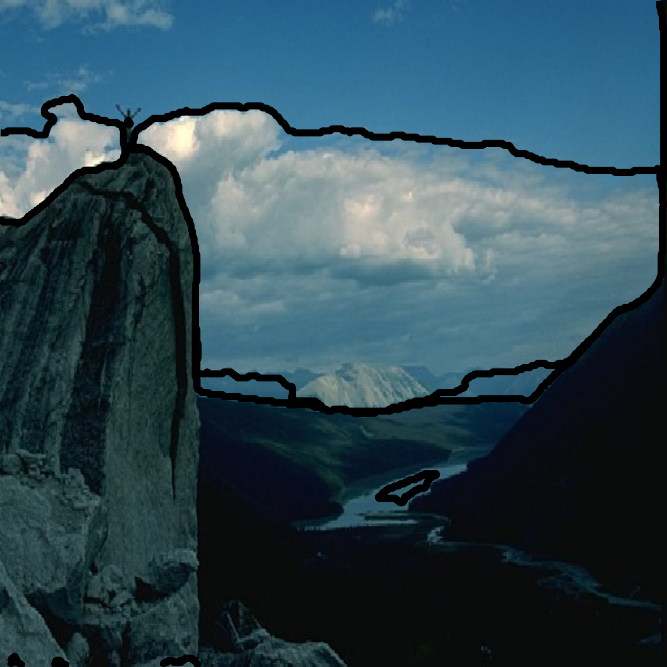}
		\includegraphics[width=2cm]{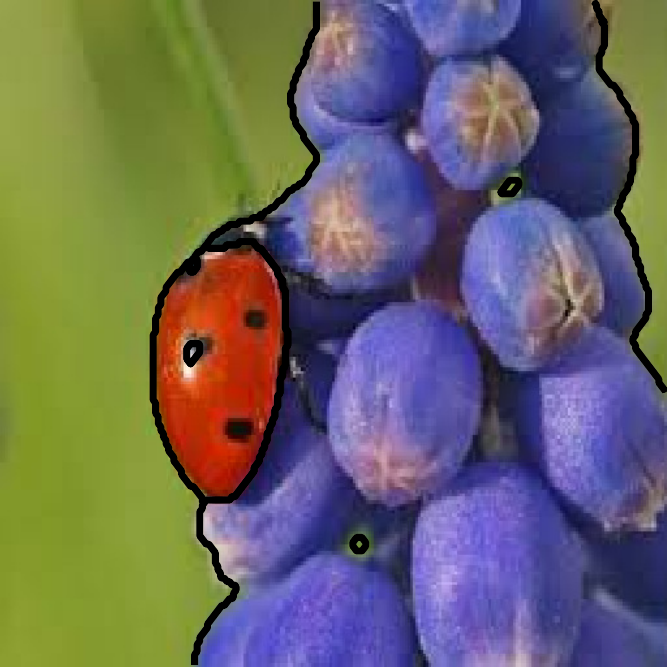}
		\includegraphics[width=2cm]{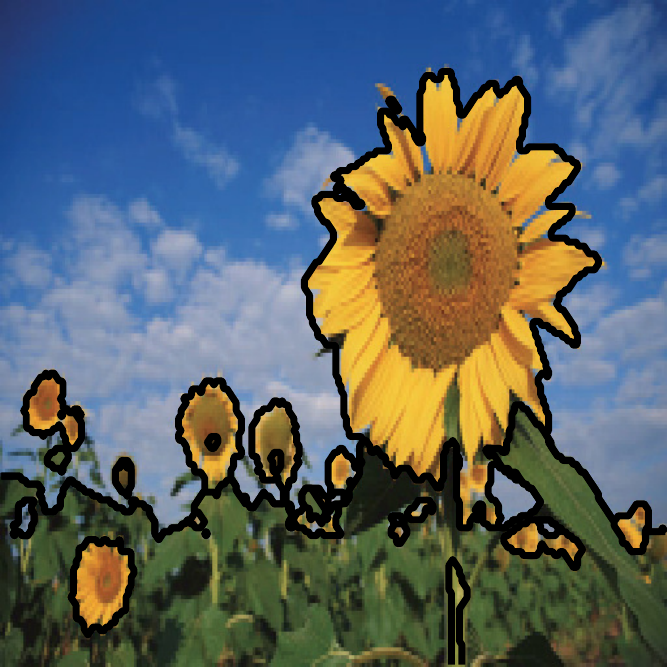}
		\includegraphics[width=2cm]{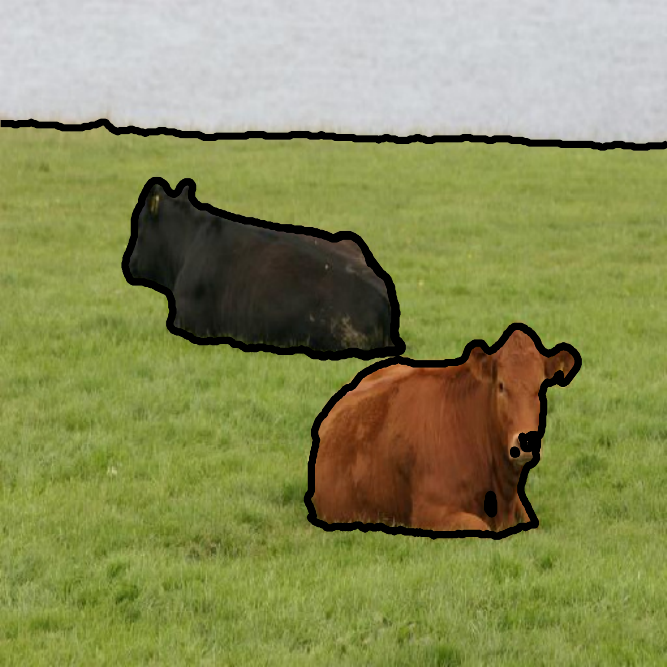}
	}
	\subfigure{
		\includegraphics[width=2cm]{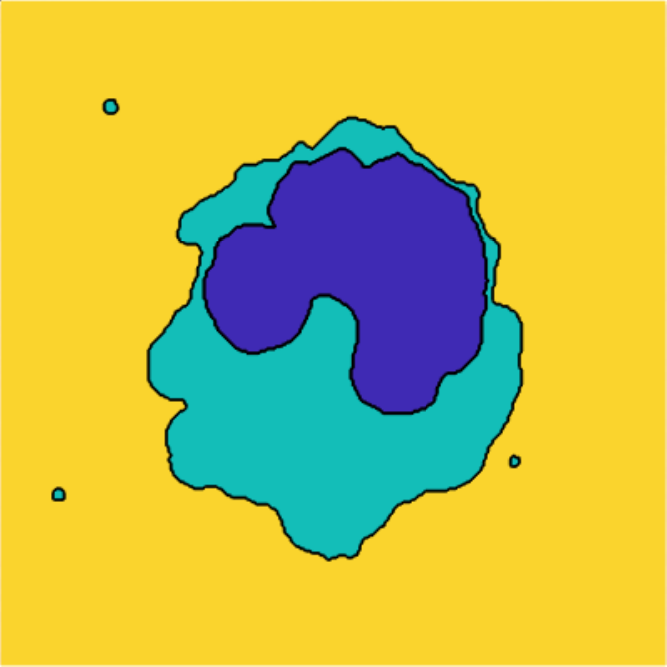}
		\includegraphics[width=2cm]{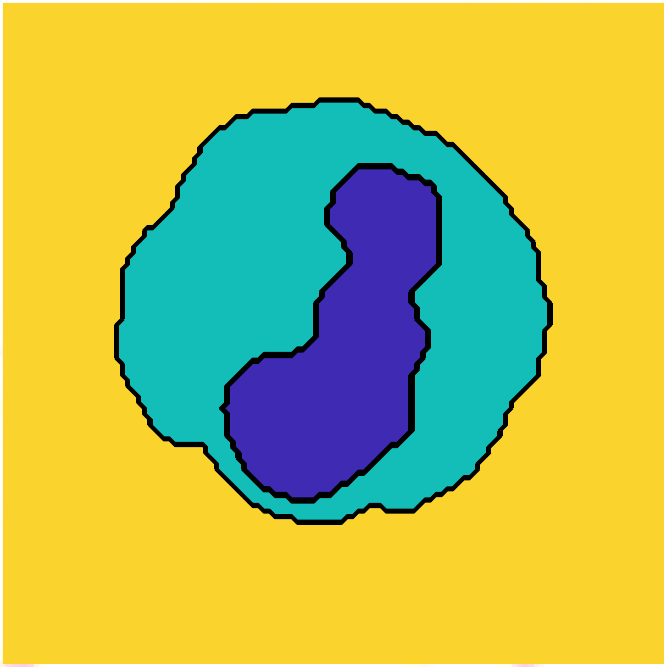}
		\includegraphics[width=2cm]{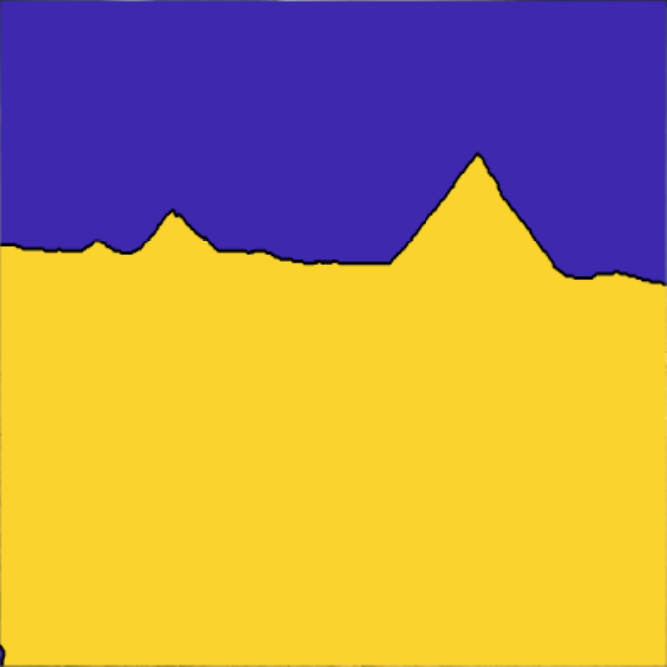}
		\includegraphics[width=2cm]{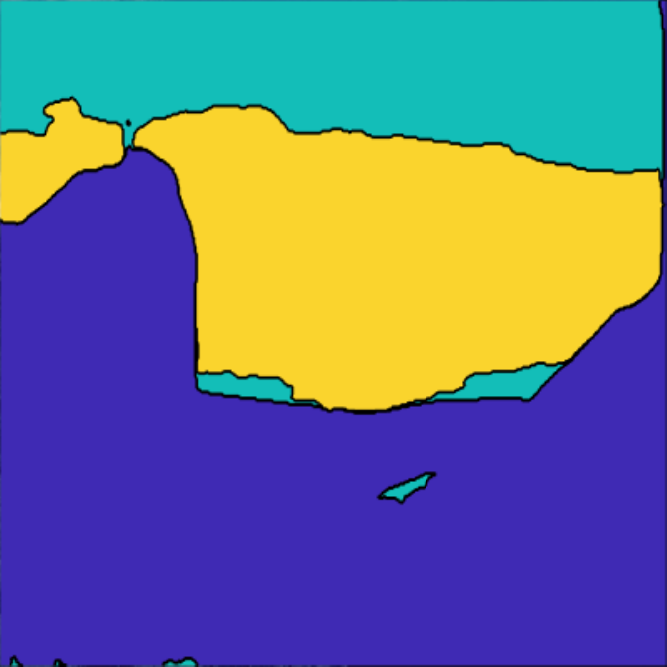}
		\includegraphics[width=2cm]{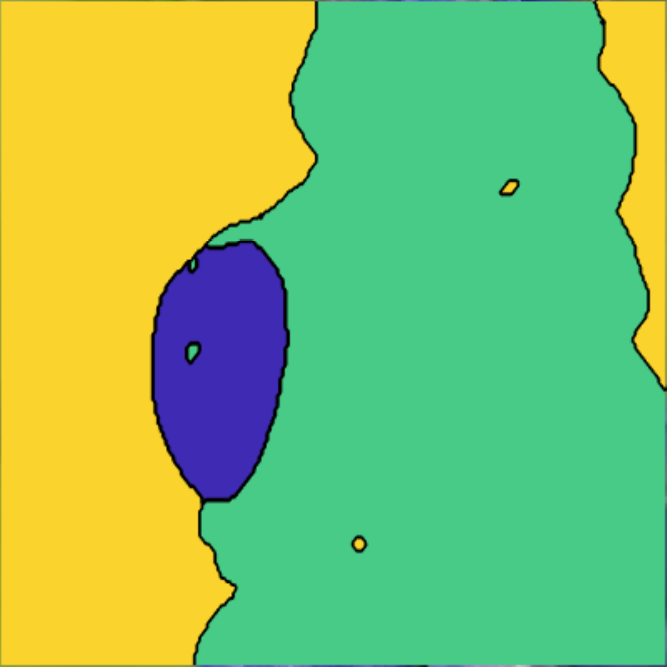}
		\includegraphics[width=2cm]{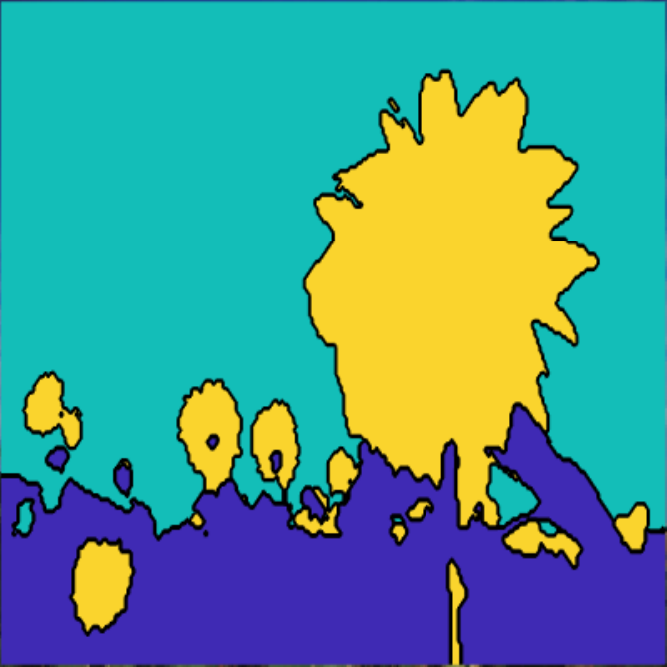}
		\includegraphics[width=2cm]{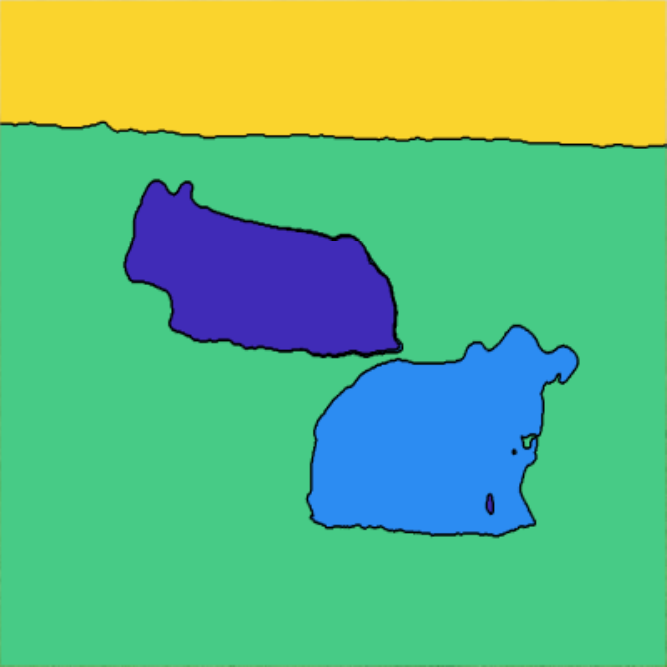}
	}
	 \caption{Experiments with the ICTM-LVF solver and the CV model on real images. From top to bottom: Original image, final contours, and final segments. All the final contours are highlighted by black lines.} \label{fig_lvf_realim}
\end{figure}

\subsection{Comparison between the ICTM-LVF and some state-of-art algorithms}
In this part, we compare our ICTM-LVF solver with some state-of-art image segmentation algorithms on noise corrupted images and medical images in terms of segmentation performance and efficiency. The segmentation results are shown in \cref{comparison1} and \cref{comparison2}. In all the experiments, the ICTM-LVF can give good results for given images while the other algorithms fail to segment all the images correctly, especially for the medical images in \cref{comparison2}. Particularly, in the fourth row of \cref{comparison1}, the ICTM-LVF solver with the CV model (ICTM-LVF-CV) can separate all the clouds from other objects while others can only separate the clouds with high brightness. Moreover, it can also be seen that the ICTM-LVF-CV achieves a more delicate segmentation result for the last image in \cref{comparison1}. The results displayed in \cref{comparison2} indicate that compared to other algorithms, the ICTM-LVF with the LIF model (ICTM-LVF-LIF) is more suitable for images with intensity inhomogeneity. The CPU times for all algorithms are recorded in \cref{table:comparison}, which exhibits the high efficiency of the ICTM-LVF solver.
\begin{figure}[htbp]
	\centering
	\subfigure{
	\begin{minipage}[b]{2cm}
		\includegraphics[width=2cm]{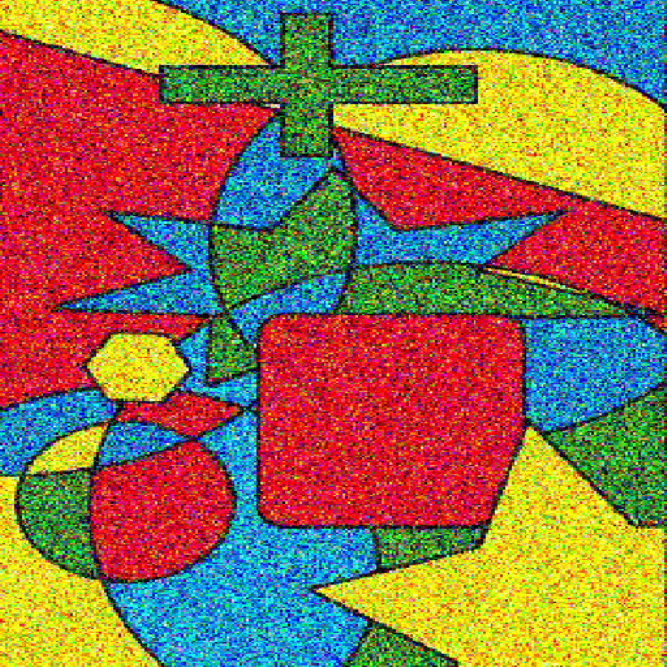}
		\includegraphics[width=2cm]{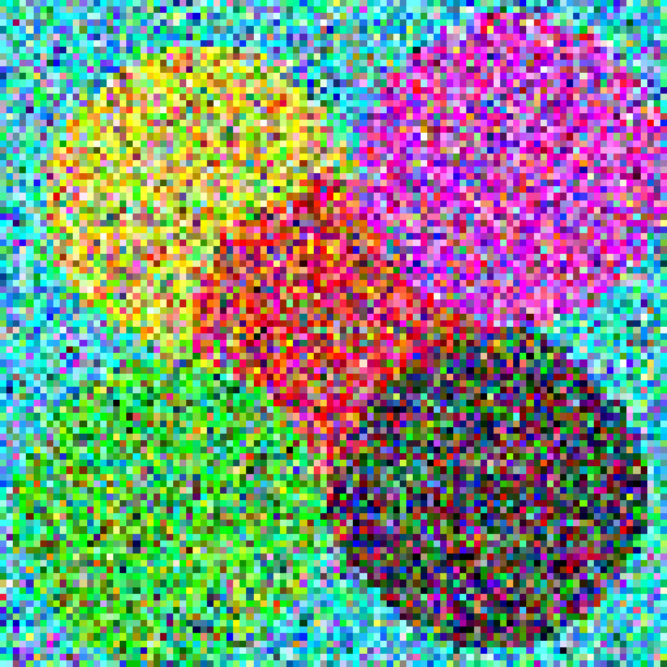}
		\includegraphics[width=2cm]{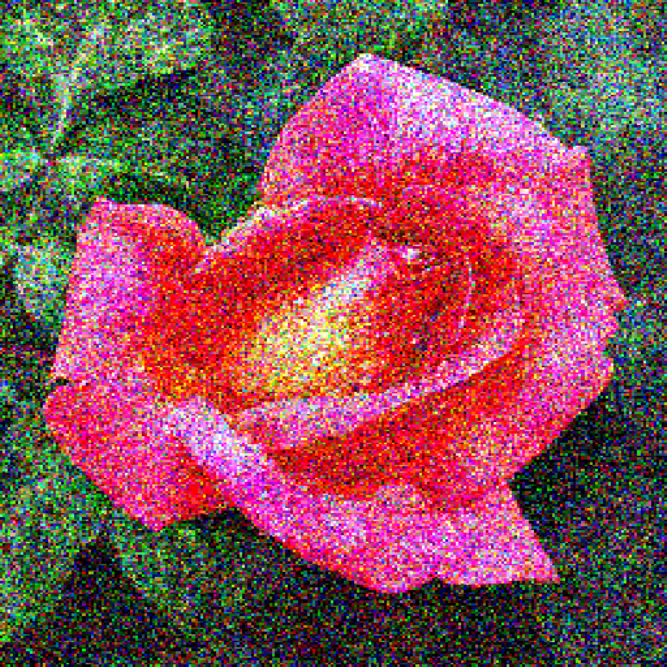}
		\includegraphics[width=2cm]{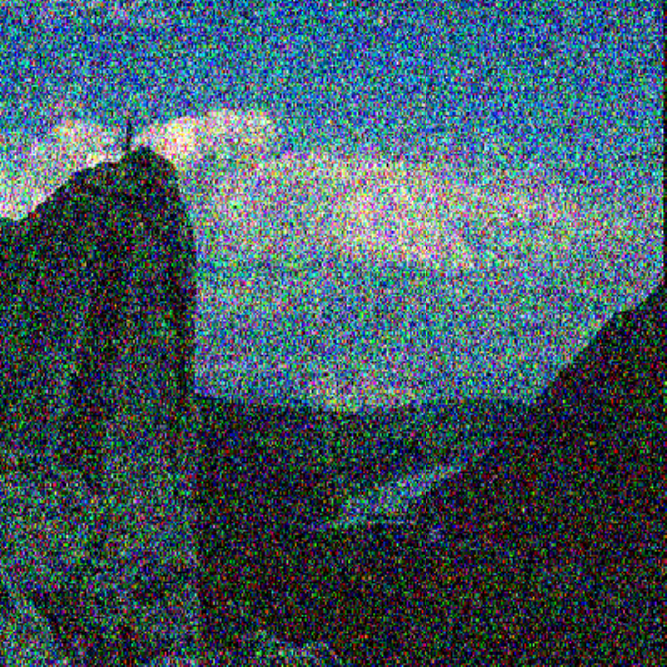}
		\includegraphics[width=2cm]{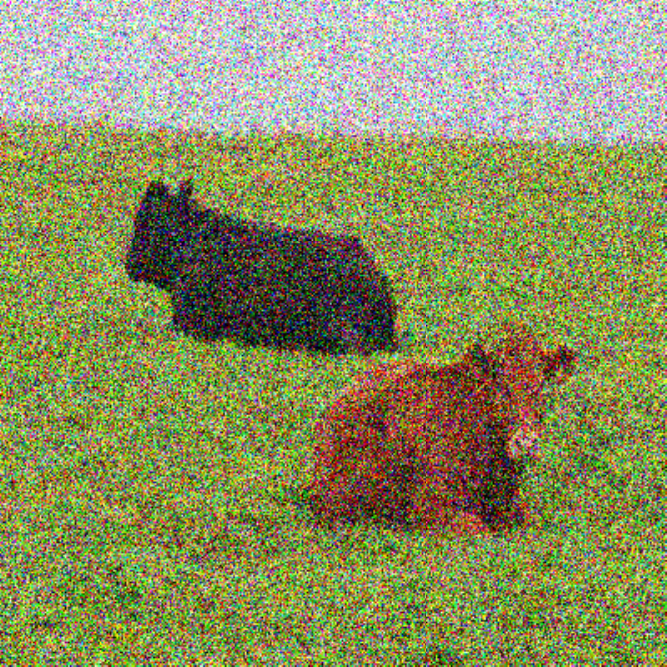}
		\centerline{\tiny (a)Original images}
	\end{minipage}
	}
	\subfigure{
	\begin{minipage}[b]{2cm}
		\includegraphics[width=2cm]{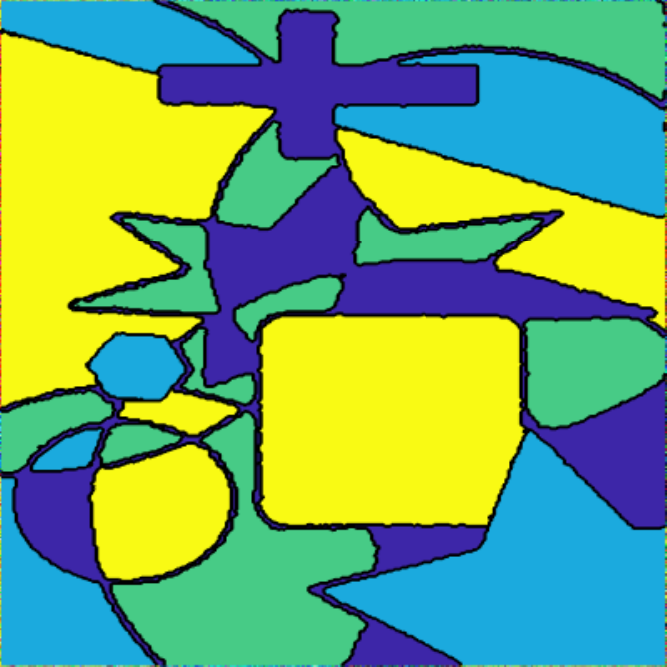}
		\includegraphics[width=2cm]{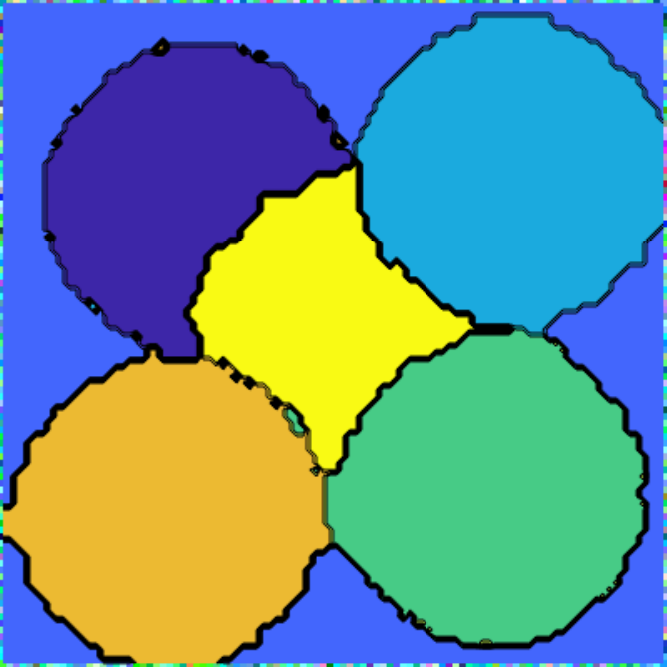}
		\includegraphics[width=2cm]{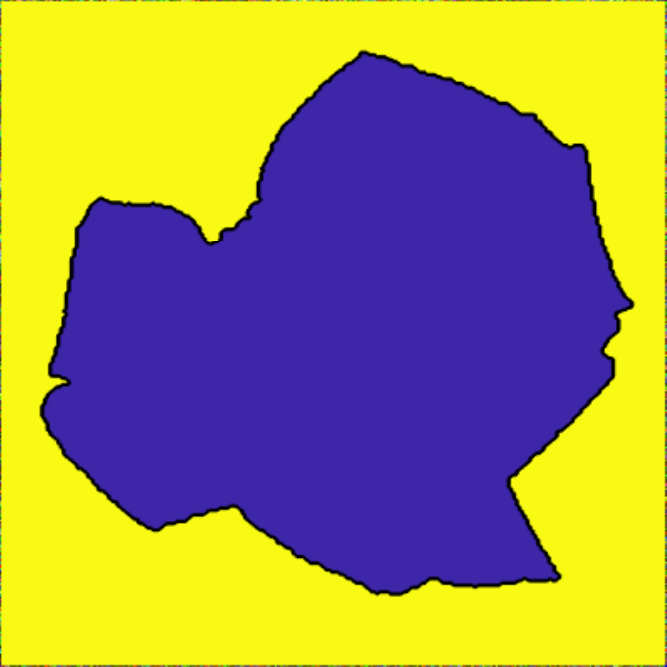}
		\includegraphics[width=2cm]{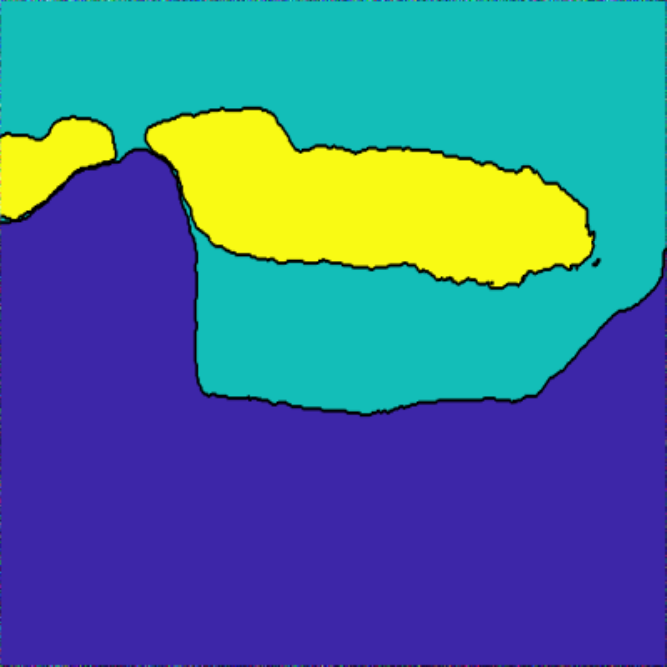}
		\includegraphics[width=2cm]{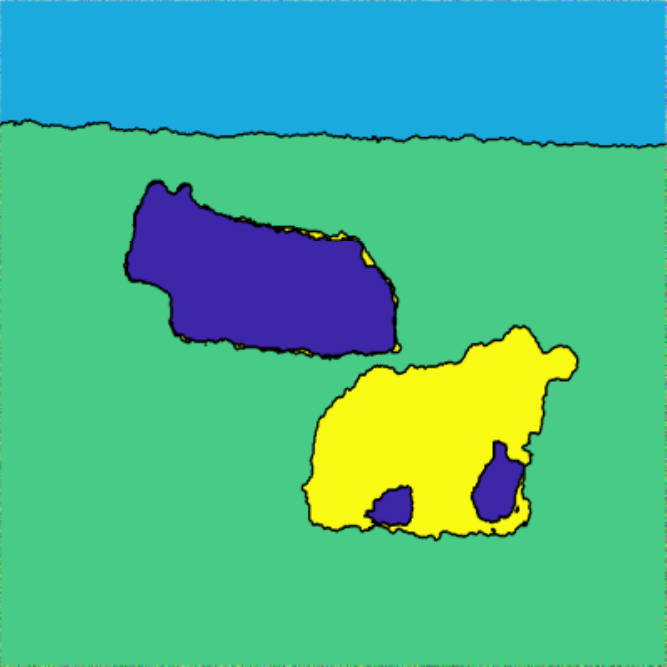}
	\centerline{\tiny(b) SLaT}
	\end{minipage}
	}
	\subfigure{
	\begin{minipage}[b]{2cm}
		\includegraphics[width=2cm]{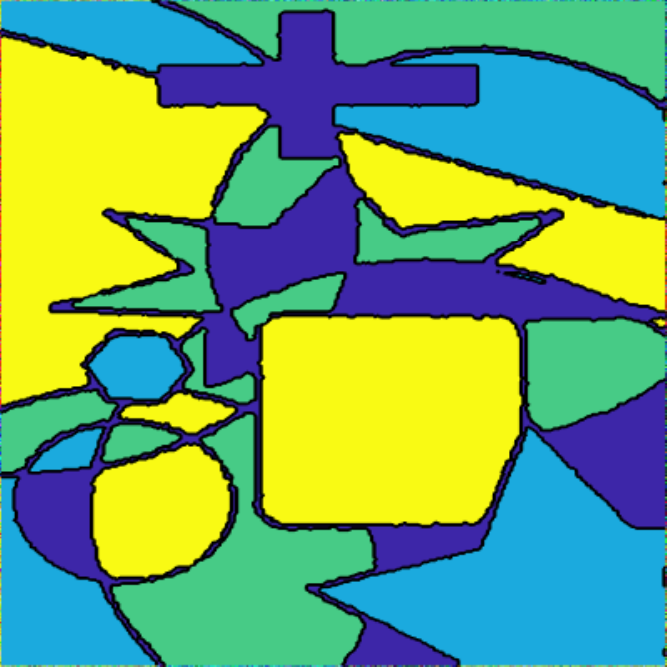}
		\includegraphics[width=2cm]{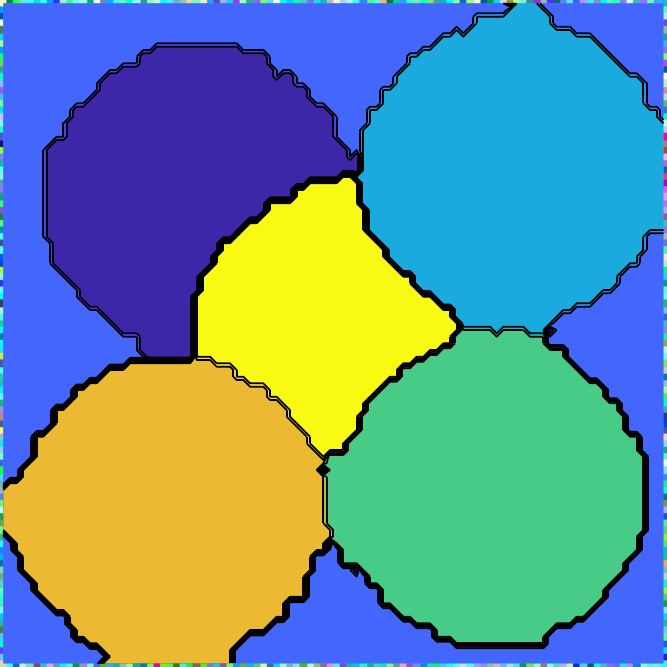}
		\includegraphics[width=2cm]{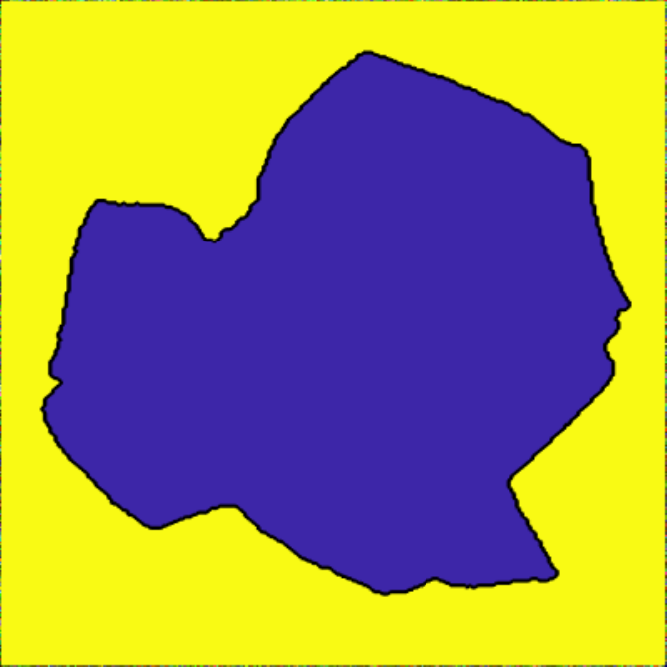}
		\includegraphics[width=2cm]{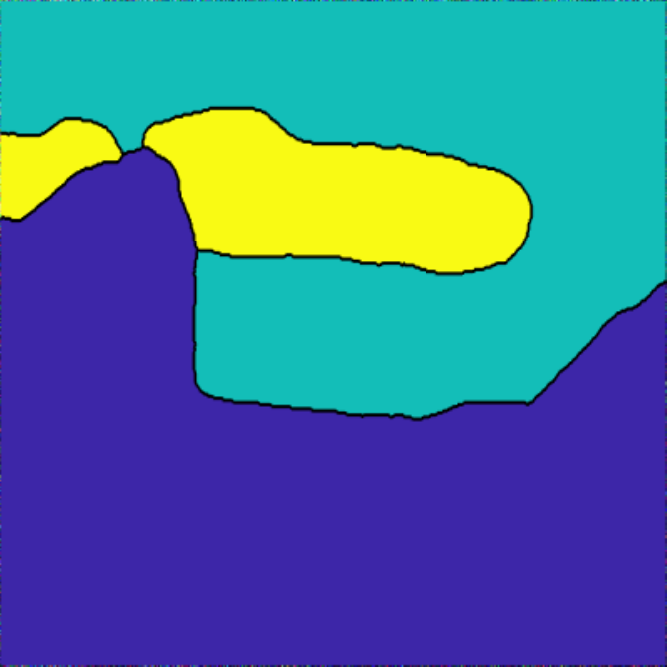}
		\includegraphics[width=2cm]{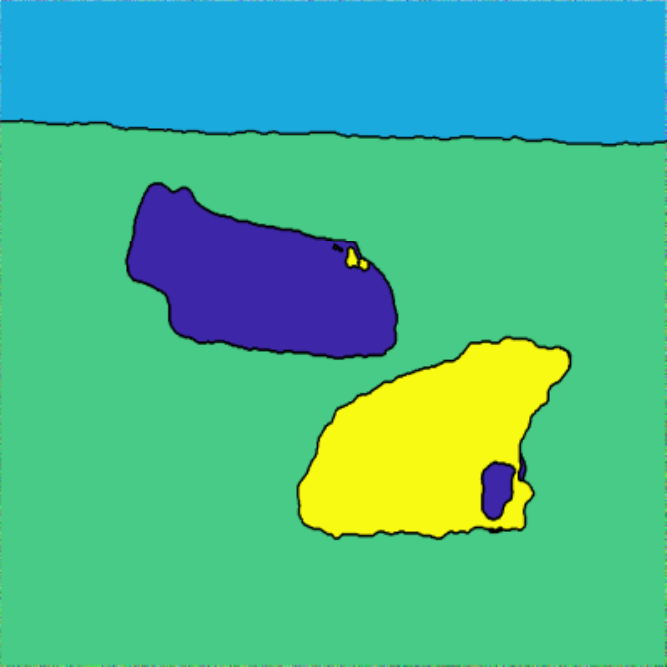}
	\centerline{\tiny(c) Model in \cite{wu2021color}}
	\end{minipage}
	}
	\subfigure{
	\begin{minipage}[b]{2cm}
		\includegraphics[width=2cm]{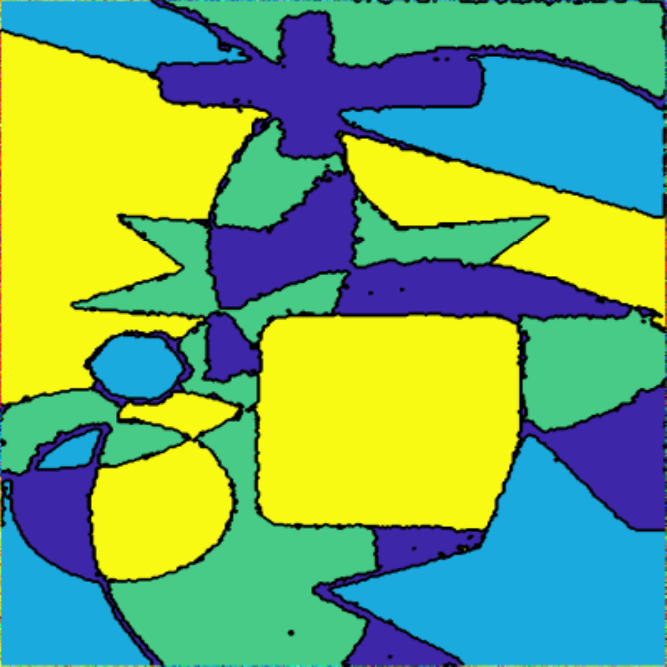}
		\includegraphics[width=2cm]{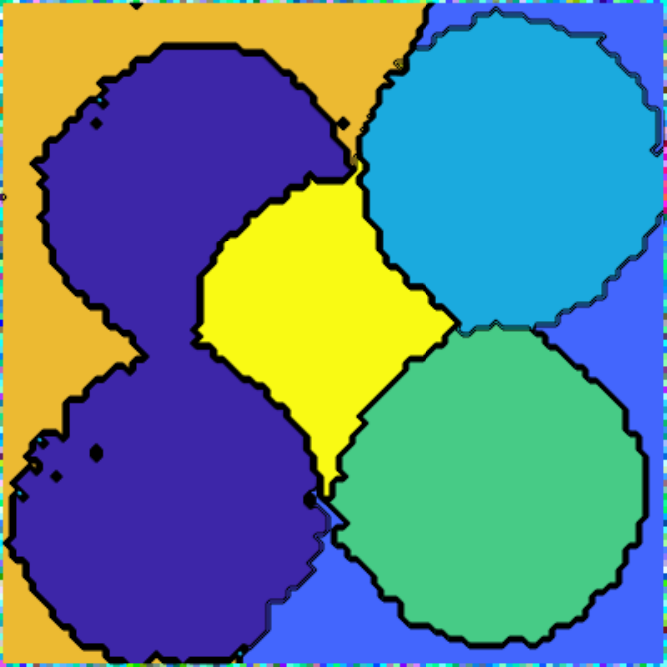}
		\includegraphics[width=2cm]{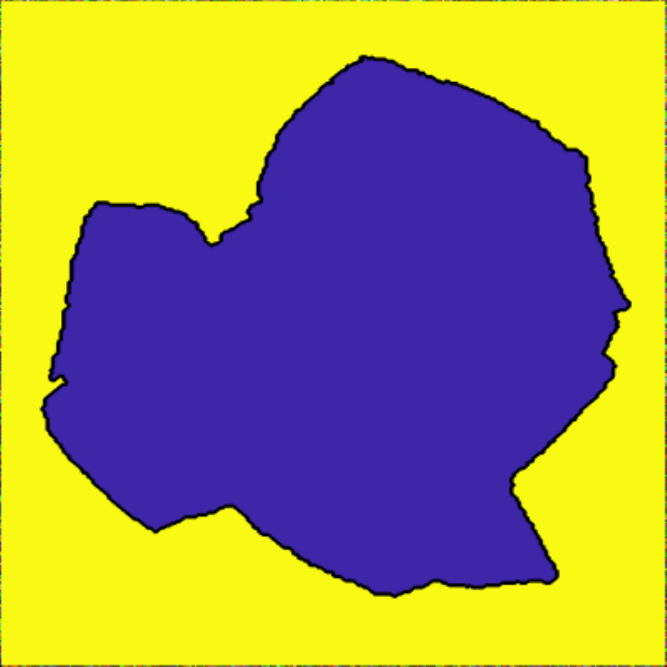}
		\includegraphics[width=2cm]{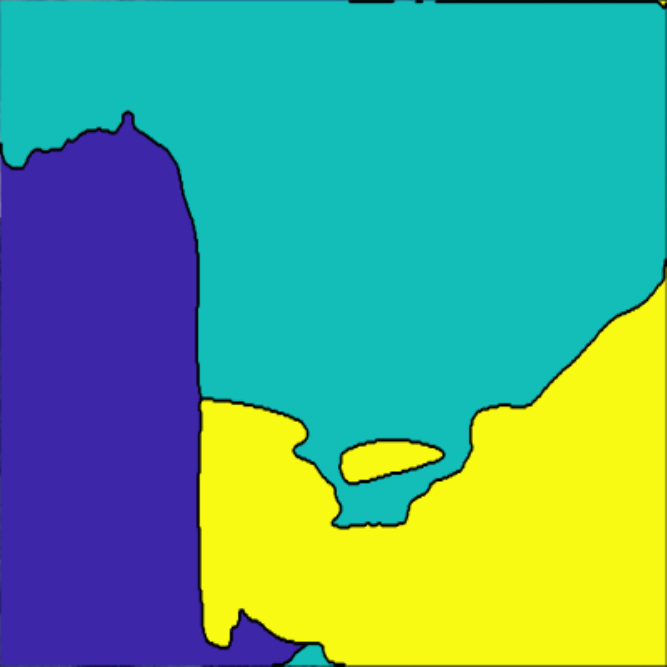}
		\includegraphics[width=2cm]{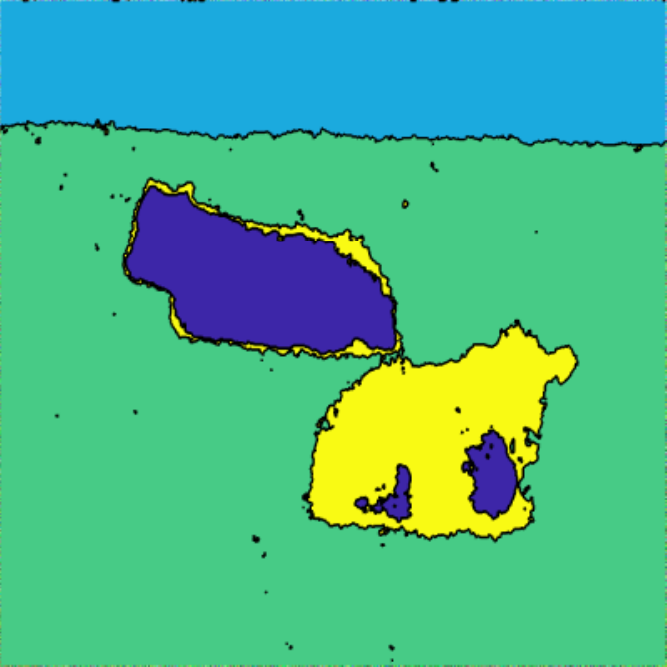}
	\centerline{\tiny(d) Model in \cite{Cai_Zeng}}
	\end{minipage}
	}
	\subfigure{
	\begin{minipage}[b]{2cm}
		\includegraphics[width=2cm]{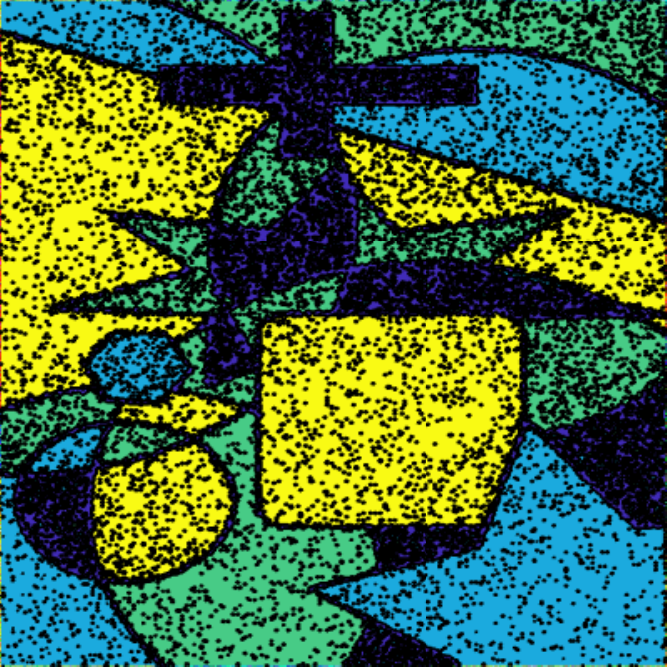}
		\includegraphics[width=2cm]{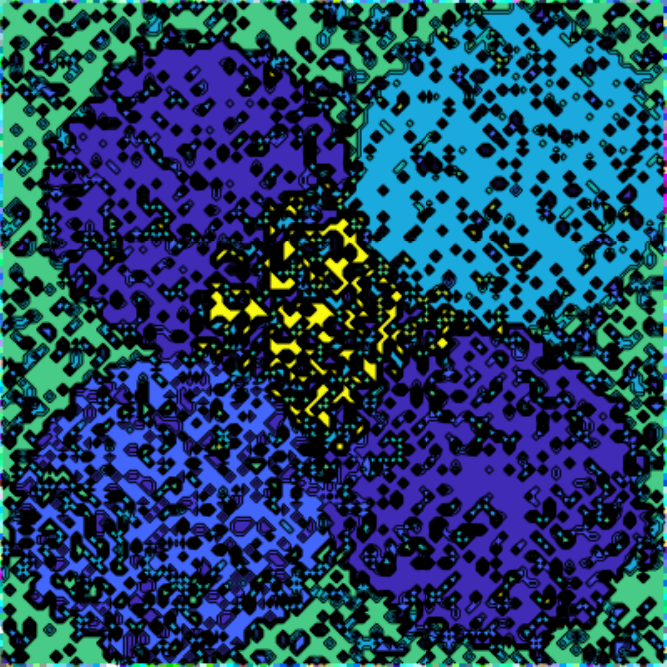}
		\includegraphics[width=2cm]{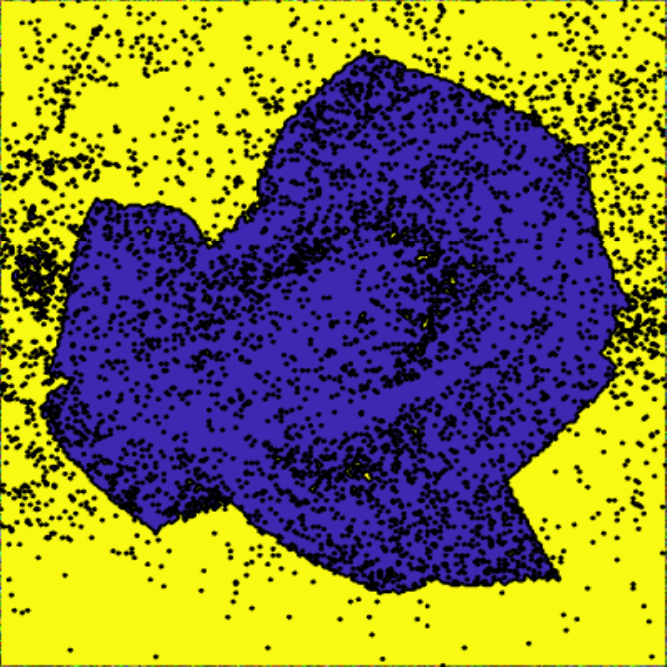}
		\includegraphics[width=2cm]{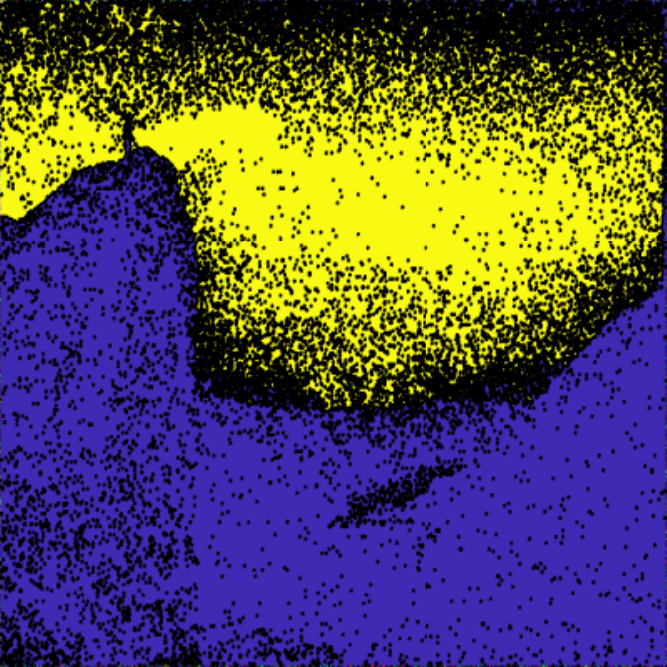}
		\includegraphics[width=2cm]{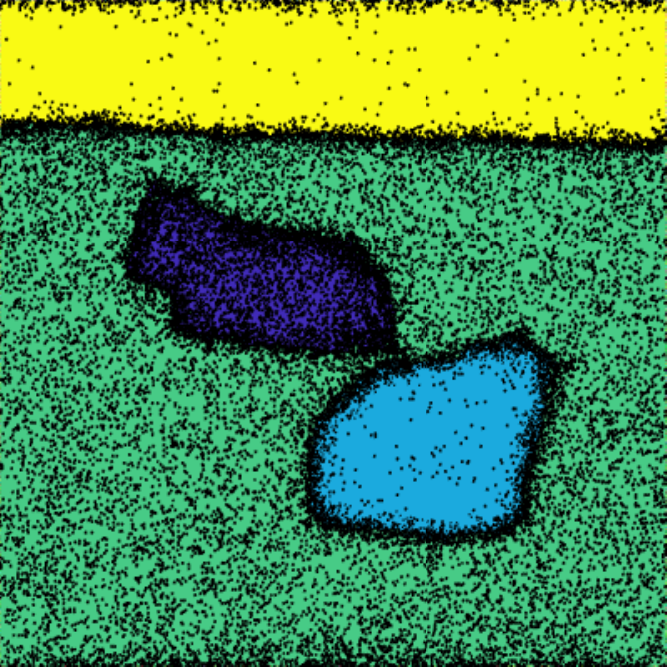}
	\centerline{\tiny(e) ICTM-CV}
	\end{minipage}
	}
	\subfigure{
	\begin{minipage}[b]{2cm}
		\includegraphics[width=2cm]{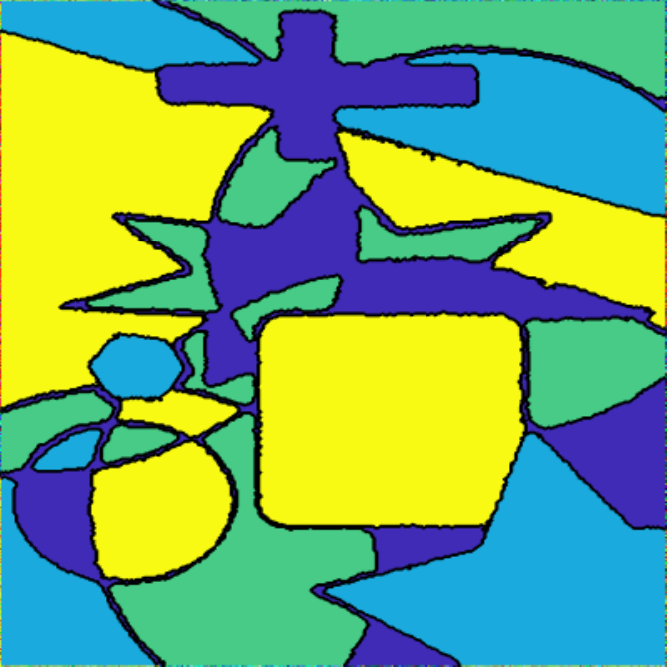}
		\includegraphics[width=2cm]{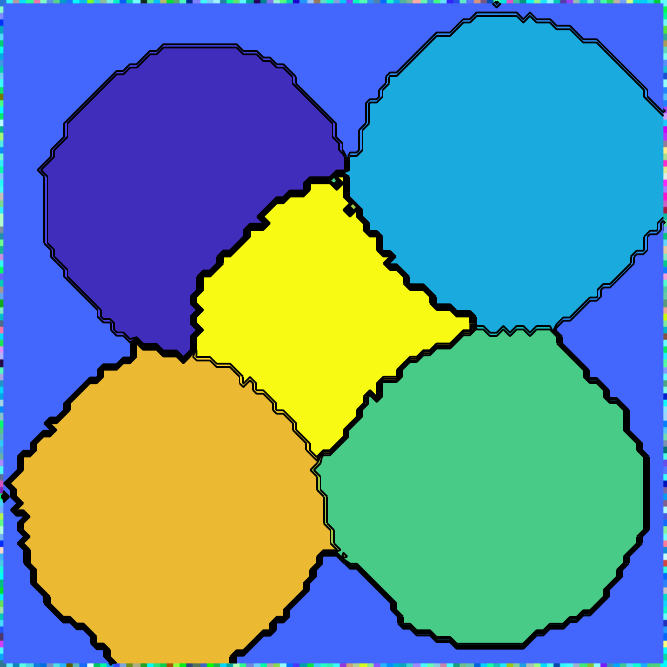}
		\includegraphics[width=2cm]{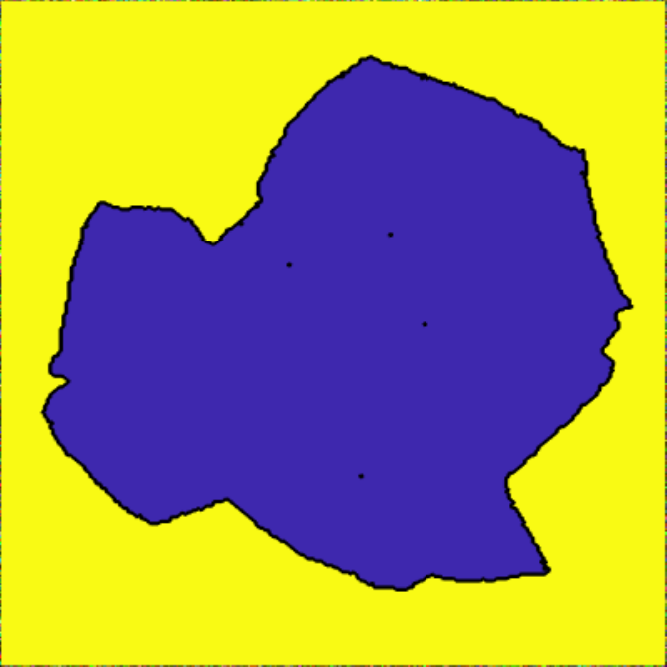}
		\includegraphics[width=2cm]{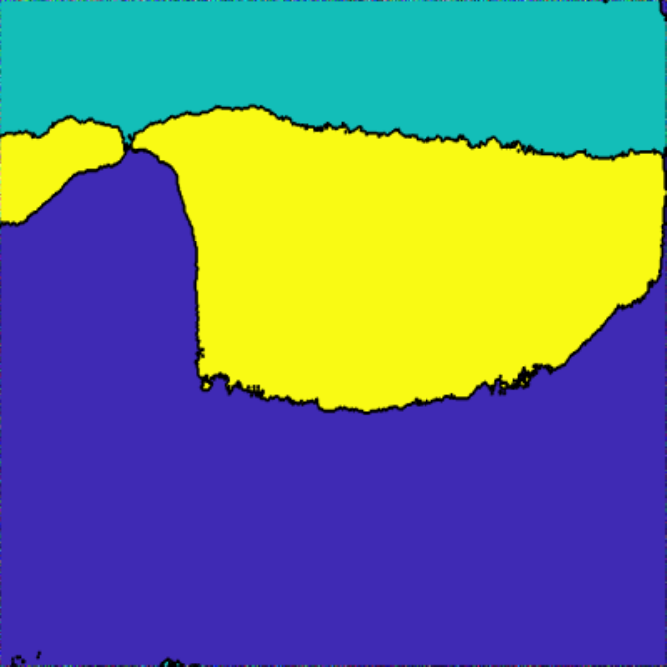}
		\includegraphics[width=2cm]{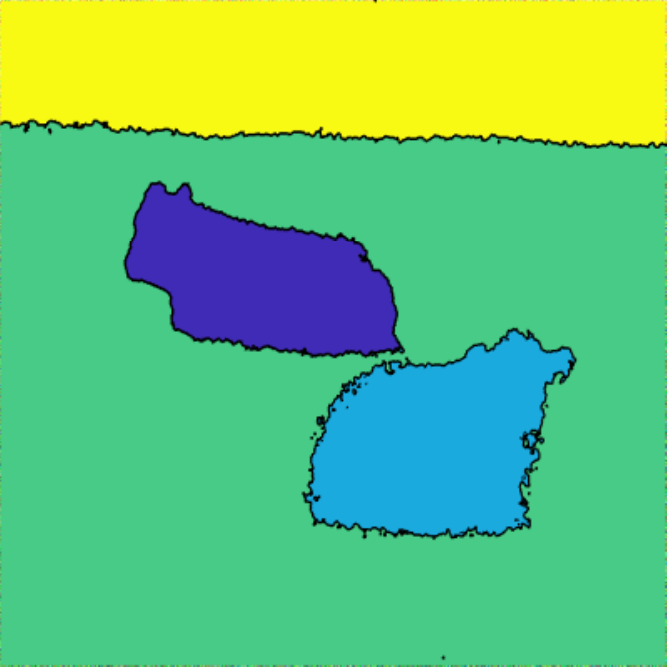}
		\centerline{\tiny(f) ICTM-LVF-CV}
  \end{minipage}
	}
 \label{comparison1}\caption{Comparison between the ICTM-LVF-CV model and other state-of-art algorithms.}
\end{figure}

\begin{figure}[htbp]
	\centering
	\subfigure{
	\begin{minipage}[b]{2cm}
		\includegraphics[width=2cm]{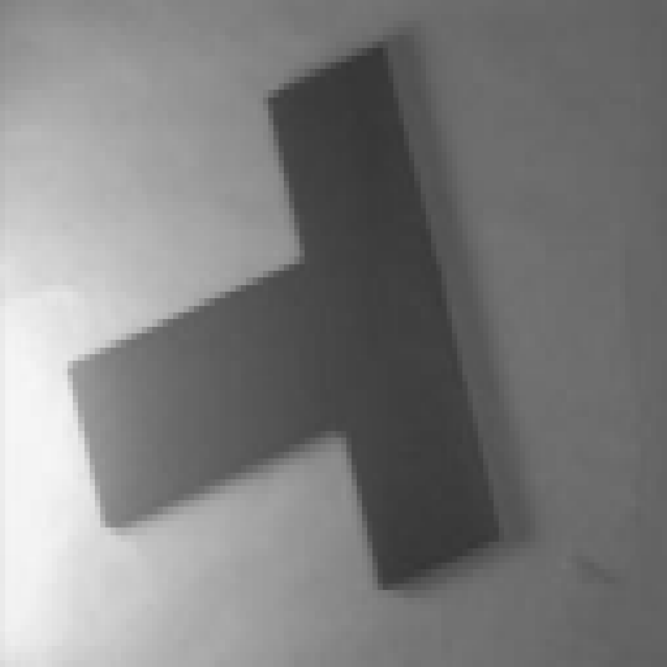}
		\includegraphics[width=2cm]{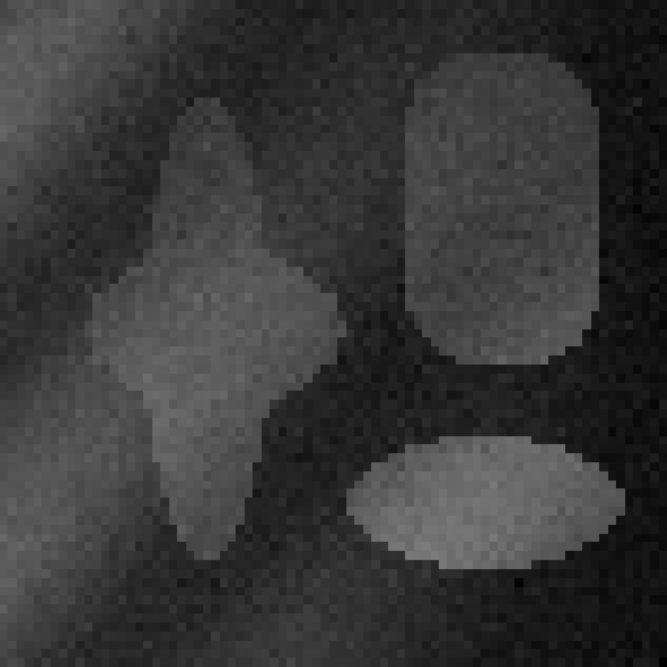}
		\includegraphics[width=2cm]{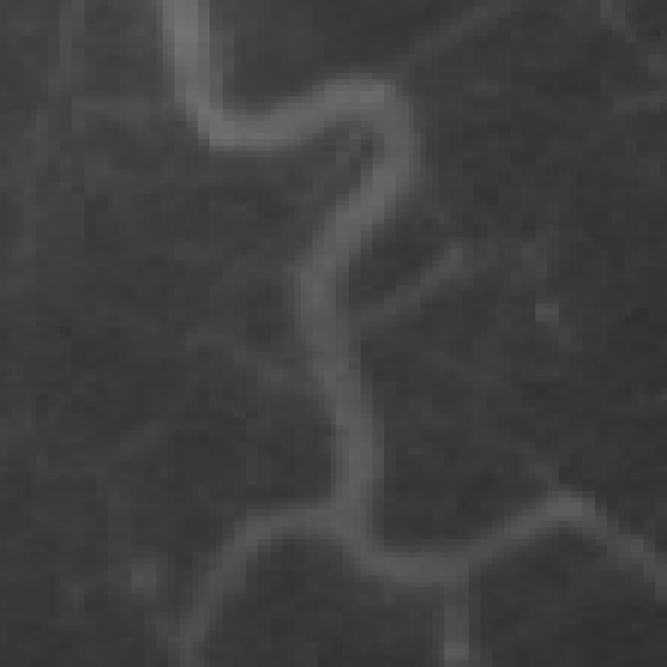}
		\includegraphics[width=2cm]{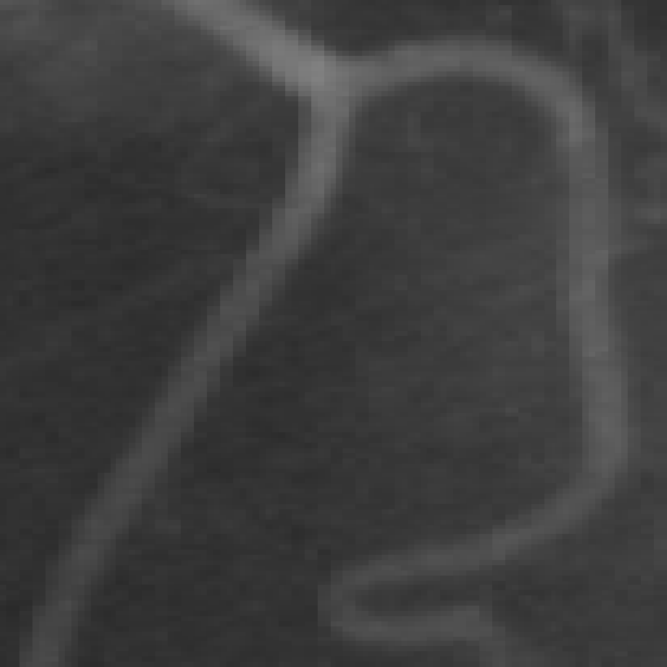}
		\centerline{\tiny (a)Original images}
	\end{minipage}
	}
	\subfigure{
	\begin{minipage}[b]{2cm}
		\includegraphics[width=2cm]{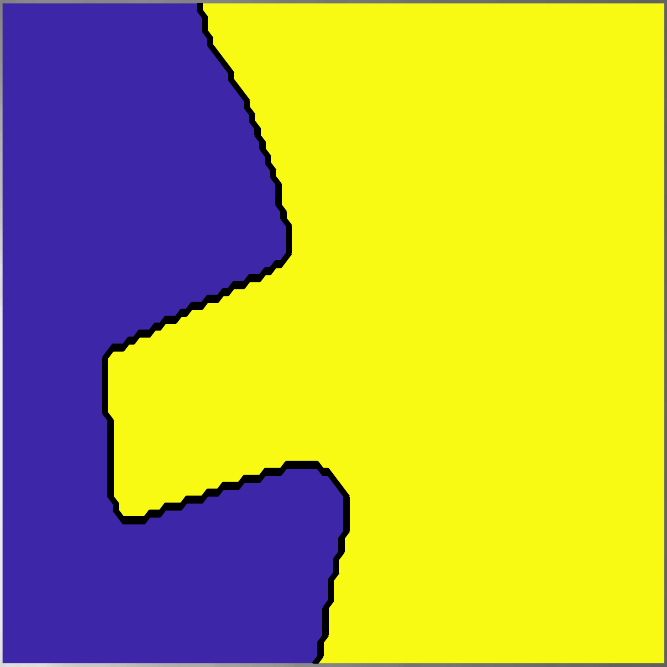}
		\includegraphics[width=2cm]{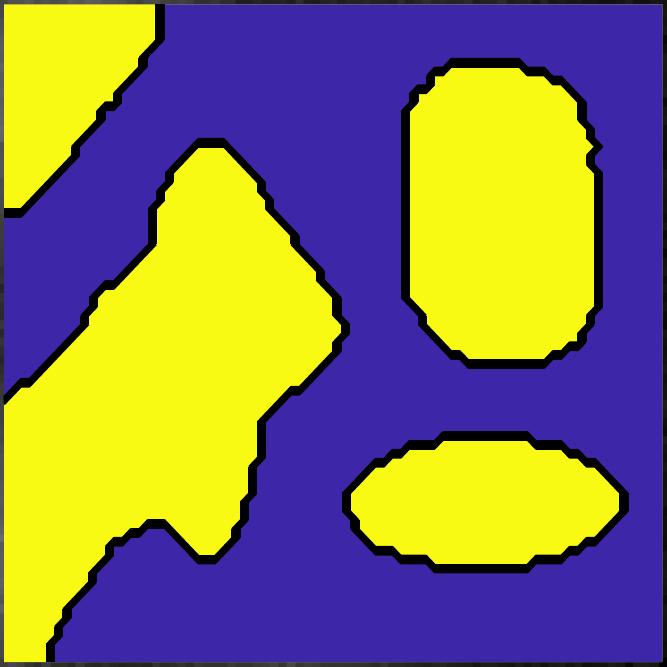}
		\includegraphics[width=2cm]{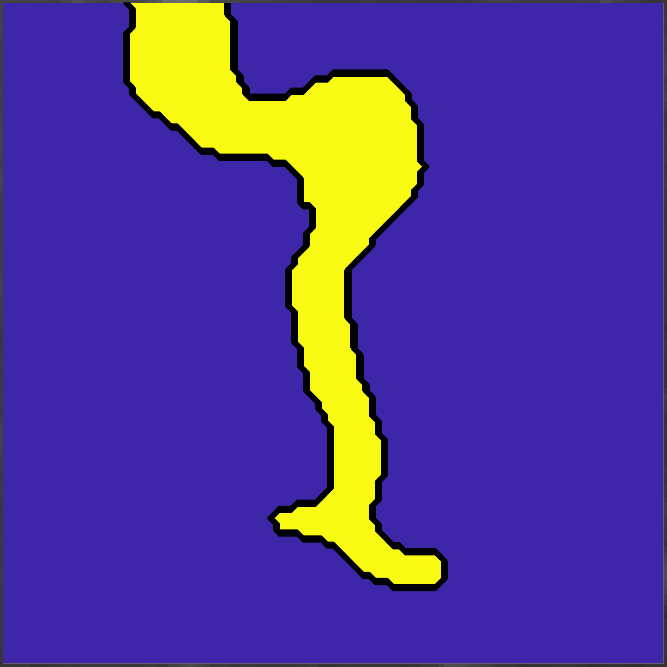}
		\includegraphics[width=2cm]{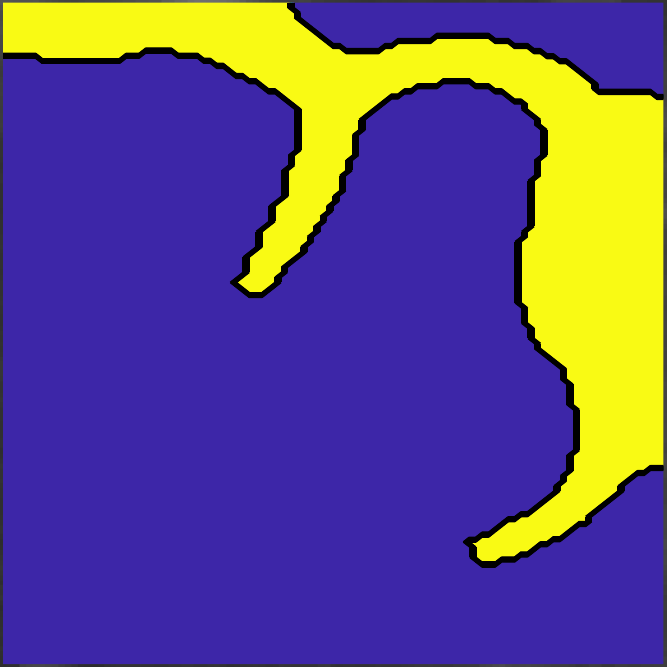}
	\centerline{\tiny (b) SLaT}
	\end{minipage}
	}
	\subfigure{
	\begin{minipage}[b]{2cm}
		\includegraphics[width=2cm]{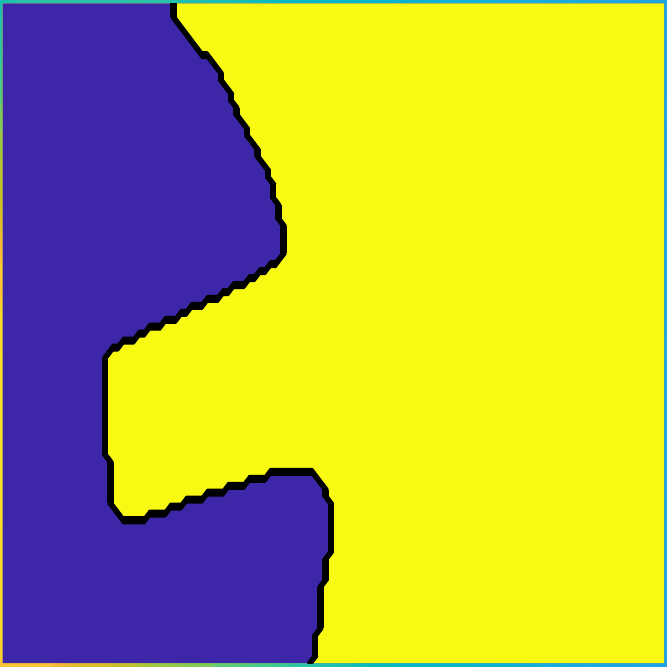}
		\includegraphics[width=2cm]{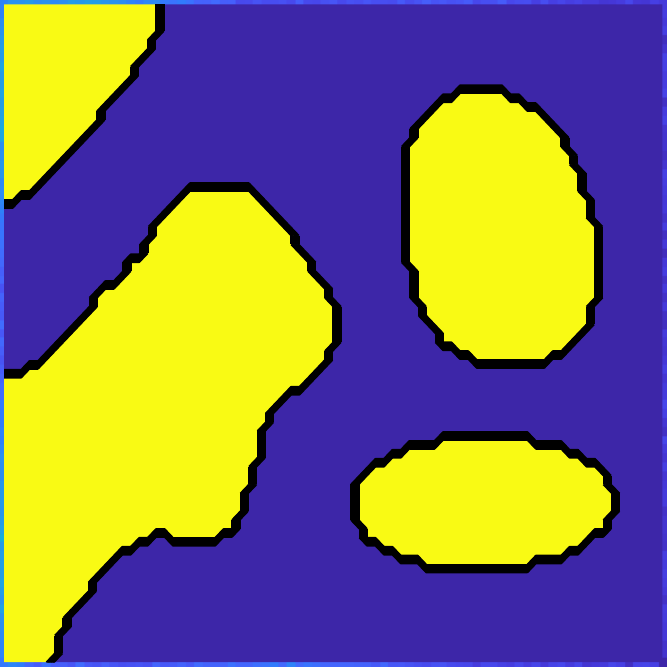}
		\includegraphics[width=2cm]{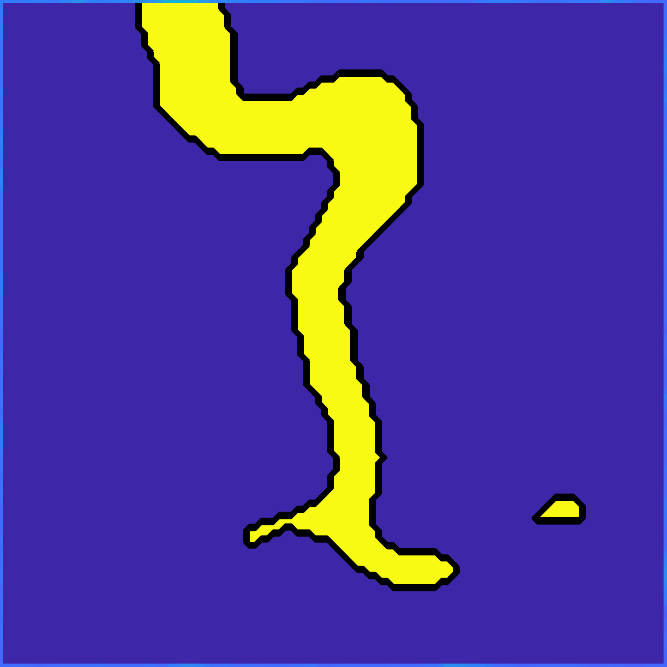}
		\includegraphics[width=2cm]{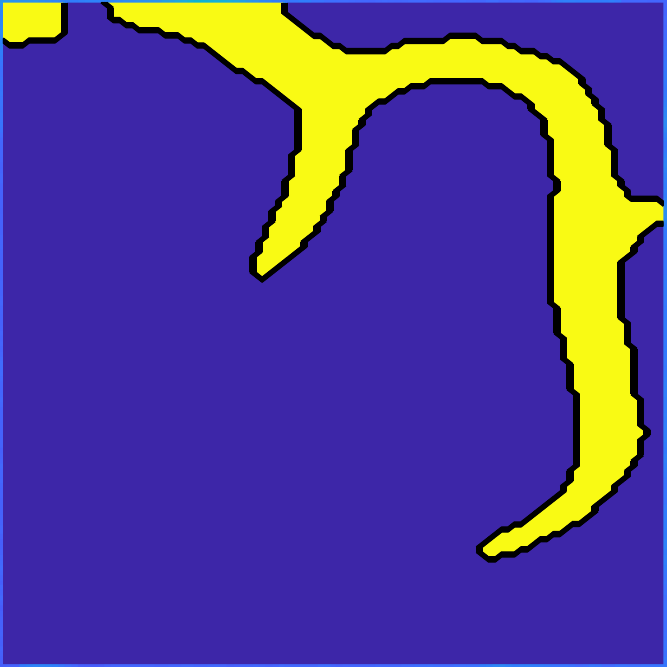}
	\centerline{\tiny(c) Model in \cite{wu2021color}}
	\end{minipage}
	}
	\subfigure{
	\begin{minipage}[b]{2cm}
		\includegraphics[width=2cm]{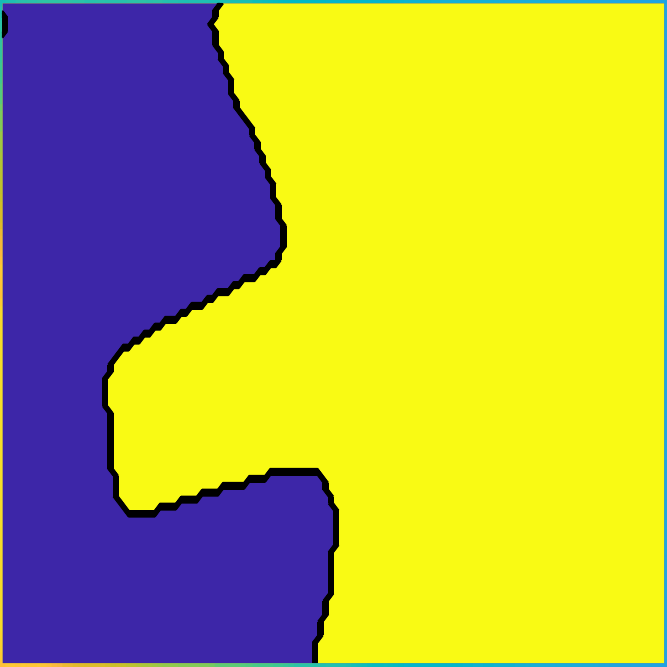}
		\includegraphics[width=2cm]{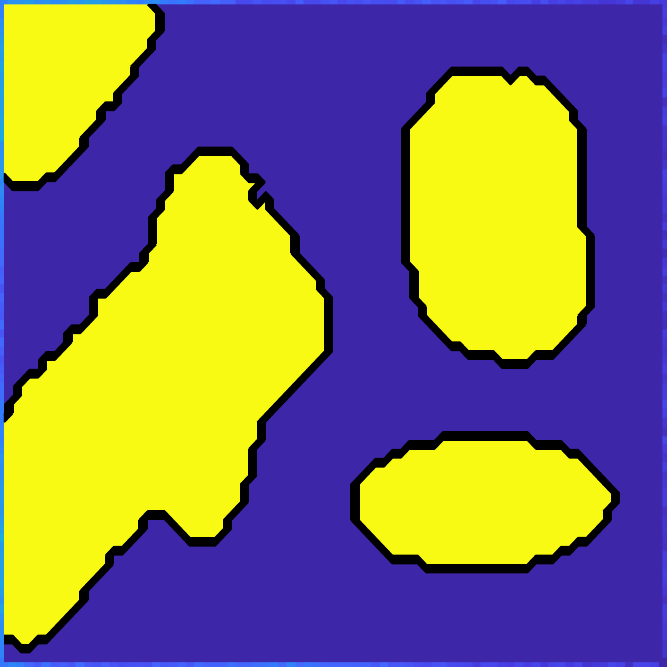}
		\includegraphics[width=2cm]{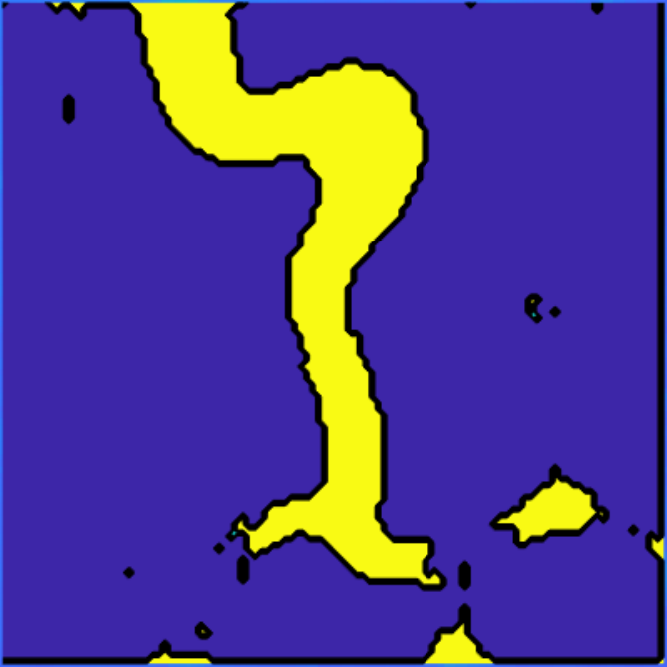}
		\includegraphics[width=2cm]{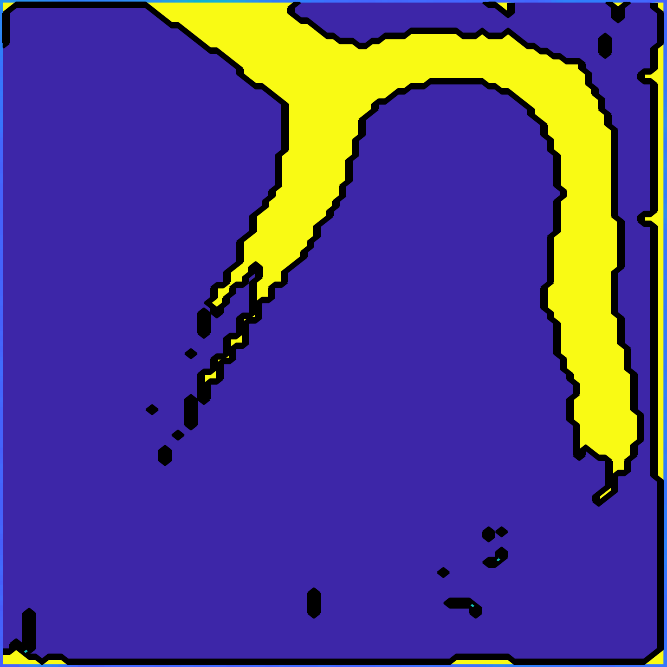}
	\centerline{\tiny(d) Model in \cite{Cai_Zeng}}
	\end{minipage}
	}
	\subfigure{
	\begin{minipage}[b]{2cm}
		\includegraphics[width=2cm]{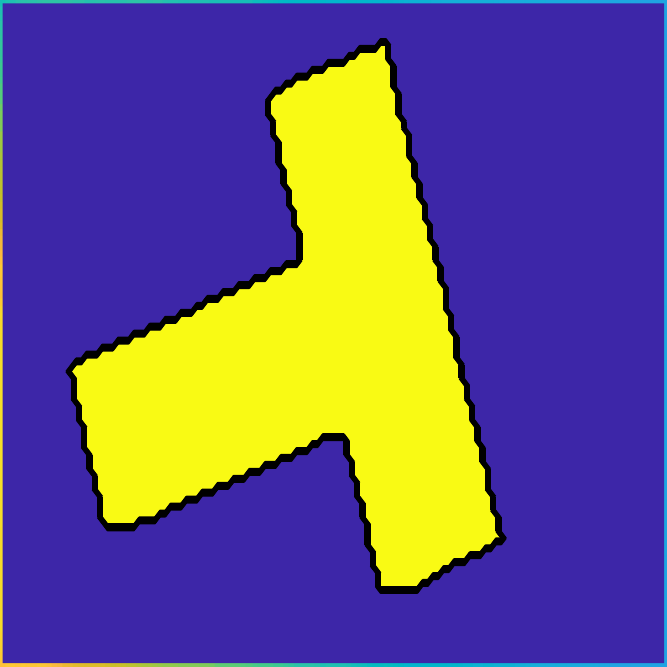}
		\includegraphics[width=2cm]{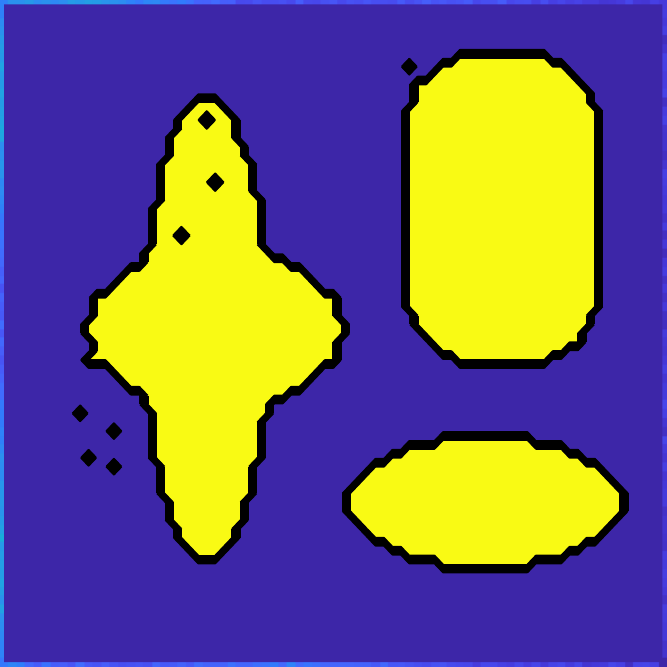}
		\includegraphics[width=2cm]{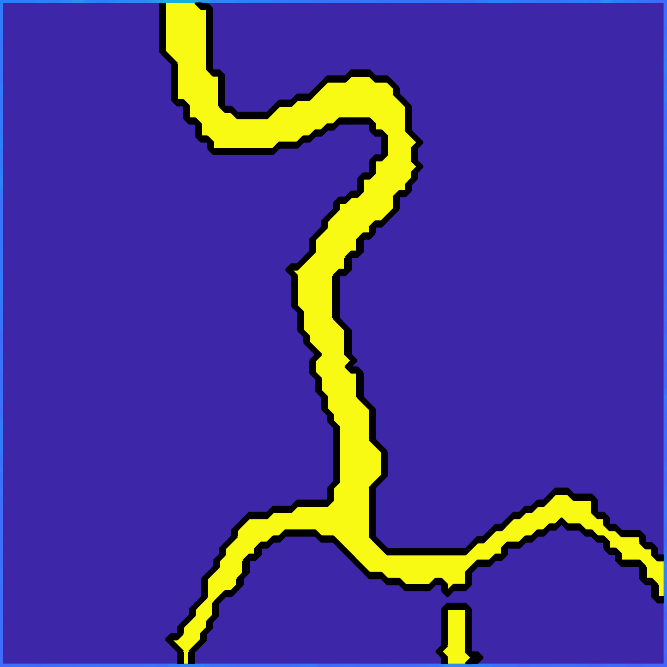}
		\includegraphics[width=2cm]{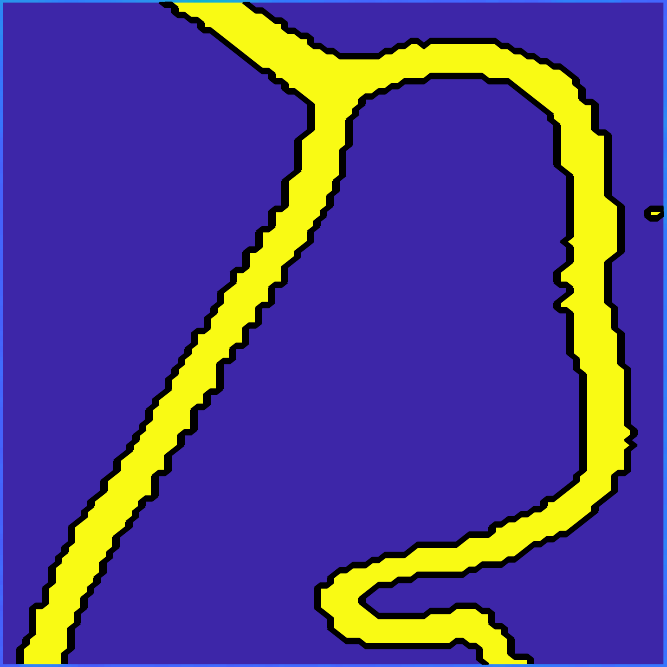}
	\centerline{\tiny(e) ICTM-LIF}
	\end{minipage}
	}
	\subfigure{
	\begin{minipage}[b]{2cm}
		\includegraphics[width=2cm]{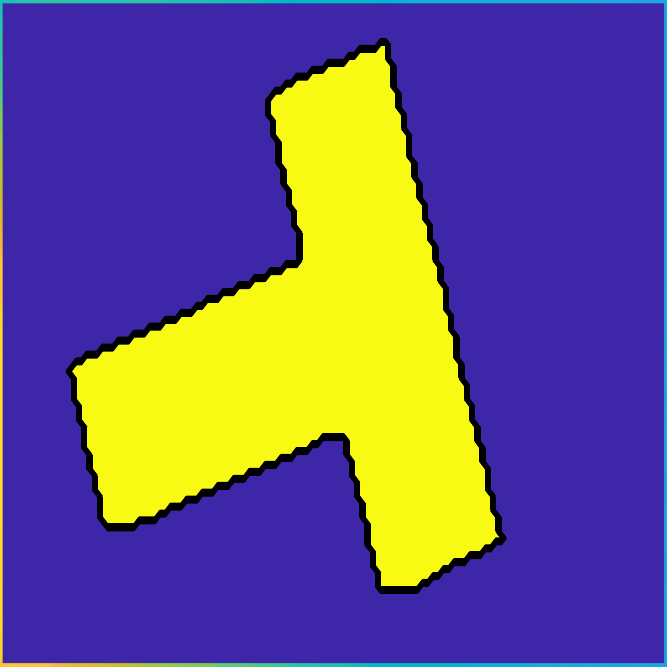}
		\includegraphics[width=2cm]{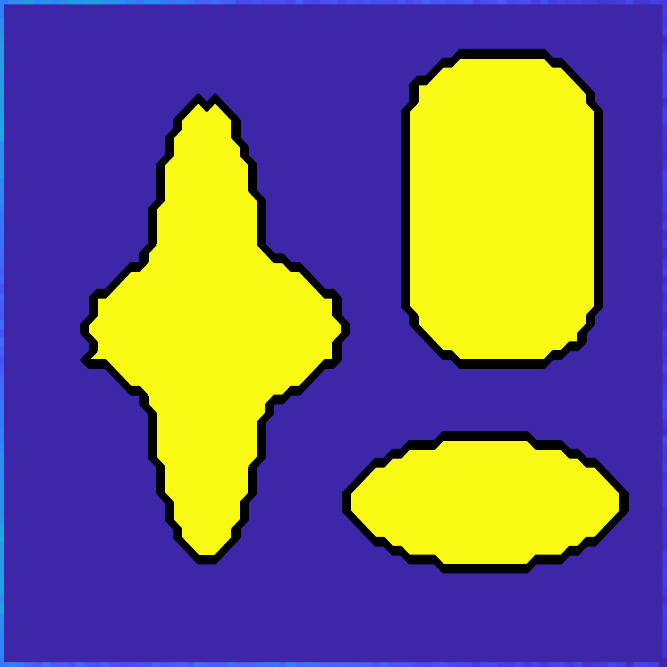}
		\includegraphics[width=2cm]{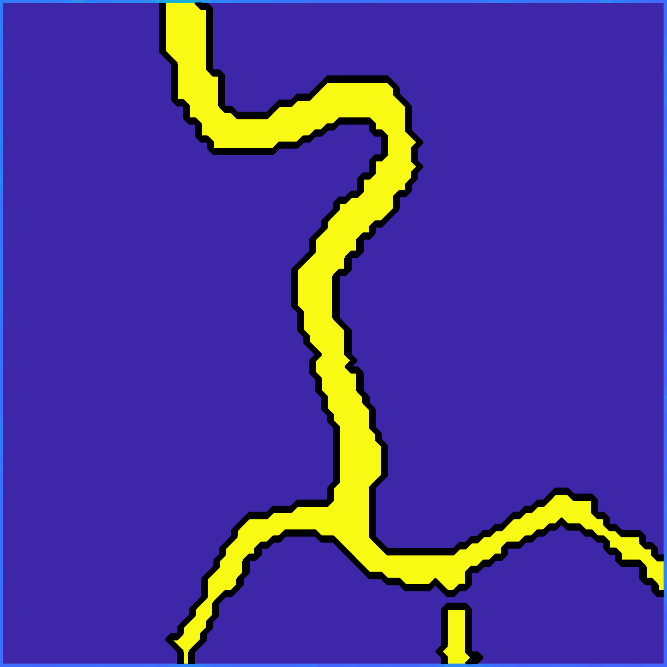}
		\includegraphics[width=2cm]{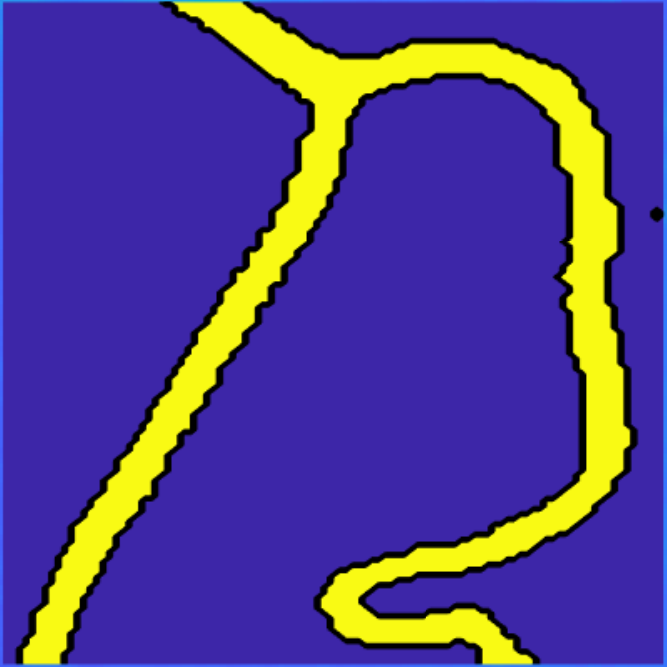}
		\centerline{\tiny(f) ICTM-LVF-LIF}
  \end{minipage}
	}
 \label{comparison2}\caption{Comparison between the ICTM-LVF-LIF model and other state-of-art algorithms.}
\end{figure}

\begin{table}[tb]
	\scriptsize
 \centering
 	\begin{tabular}{llccccc}
 		\hline
 		\multicolumn{2}{c}{Images} &
 		SLaT&	Model in \cite{wu2021color}& Model in \cite{Cai_Zeng} &ICTM & ICTM-LVF\cr
 		\hline
 		\multirow{4}{*}{\cref{comparison1}}
 		&Row1 &2.035450 &4.180187 &2.815690 &0.937057 &0.969474  \cr
 		&Row2 &0.485363 &0.500887 &0.462177 &0.298151 &0.316960  \cr
 		&Row3 &1.781461 &1.205533 &2.006805 &0.517074 &0.529115  \cr
		&Row4 &9.804071 &18.852447 &6.061458 &2.195570 &4.333940  \cr
		&Row5 &8.487740 &14.231849 &7.372541 &5.178055 &6.885624  \cr\hline
 		\multirow{4}{*}{\cref{comparison2}}
 		&Row1 &0.566980 &1.065398 &0.743467 &0.172654 &0.233498  \cr
 		&Row2 &0.372138 &1.007371 &0.384031 &0.096292 &0.166943  \cr
 		&Row3 &0.711676 &1.048547 &0.723667 &0.192009 &0.282930  \cr
 		&Row4 &1.066218 &1.171380 &0.964410 &0.261146 &0.465101  \cr\hline
 	\end{tabular}
 	\label{table:comparison}\caption{CPU time(s) for experiments in \cref{comparison1} and \cref{comparison2}.}
 \end{table}
\section{Conclusion}
\label{sec5}
We propose a new framework on active contour models for the multi-phase image segmentation. This framework consists of a novel energy with a local variance force term, an effective initialization method (Multi-IGLIM), and an efficient energy-decaying algorithm (ICTM-LVF) based on the modification of the ICTM. The proposed Multi-IGLIM and the ICTM-LVF solver can greatly improve the robustness of active contour models to the initialization and noise. Experiments show that compared to some state-of-the-art methods, the proposed framework has better performance in terms of the segmentation quality and the running time, even applied to images with severe intensity inhomogeneity and strong noise.

\section*{Acknowledgments}
We thank Prof. Tieyong Zeng at The Chinese University of Hong Kong for providing the MATLAB codes of \cite{Cai_Zeng} and thank Prof. Tingting Wu at Nanjing University of Posts and Telecommunications for providing the MATLAB codes of \cite{wu2021color}.

\bibliographystyle{siamplain}
\bibliography{ICTM-LVF.bib}

\end{document}